\documentclass[]{fairmeta}
\usepackage{wrapfig}
\usepackage{tabularx}
\usepackage{textcomp}
\usepackage{stfloats}
\usepackage{url}
\usepackage{verbatim}
\usepackage{titlesec}
\usepackage{tocloft}
\usepackage{adjustbox}
\usepackage{tikz}
\usepackage{bbm}
\usepackage{comment}
\usepackage{amsmath,amssymb}
\usepackage{colortbl}
\usepackage{color}
\usepackage{hyperref}
\usepackage{graphicx}
\usepackage{caption}
\usepackage{subcaption} 
\usepackage{multirow}
\usepackage{booktabs}
\usepackage{xcolor}
\usepackage{pifont}
\usepackage{bbding}
\usepackage{fontawesome}
\usepackage{xspace}
\usepackage{float}
\usepackage{algorithm}
\usepackage{listings}
\usepackage{algcompatible}
\usepackage{algpseudocode}
\usepackage{soul}
\usepackage{array}
\usepackage{tcolorbox}
\usepackage{diagbox}
\usepackage{makecell}
\usepackage{rotating}
\usepackage{enumitem}
\usepackage{amsthm}
  \theoremstyle{plain}           % default: italic body
  \newtheorem{theorem}{Theorem}[section]
  \newtheorem{lemma}[theorem]{Lemma}
  
  \newtheorem{corollary}[theorem]{Corollary}

  \theoremstyle{definition}      % upright body
  \newtheorem{definition}[theorem]{Definition}

  \theoremstyle{remark}          % italic header, upright body

\usepackage{amsmath,amsfonts,bm}

\def\eqref#1{equation~\ref{#1}}
\def\1{\bm{1}}

\DeclareMathAlphabet{\mathsfit}{\encodingdefault}{\sfdefault}{m}{sl}
\SetMathAlphabet{\mathsfit}{bold}{\encodingdefault}{\sfdefault}{bx}{n}

\newcommand{\E}{\mathbb{E}}

\newcommand{\R}{\mathbb{R}}

\RequirePackage{xspace}
\makeatletter
\DeclareRobustCommand\onedot{\futurelet\@let@token\@onedot}
\def\@onedot{\ifx\@let@token.\else.\null\fi\xspace}

\makeatother

\definecolor{adptorange}{RGB}{248, 205, 172}
\definecolor{cmpblue}{RGB}{189, 215, 238}
\definecolor{cmpblue}{RGB}{189, 215, 238}

\definecolor{our_red}{RGB}{232,157,160}
\definecolor{our_blue}{RGB}{136,206,230}
\definecolor{our_orange}{RGB}{246,200,168}
\definecolor{our_green}{RGB}{178,211,164}

\definecolor{attn_code0}{RGB}{247,215,200}
\definecolor{attn_code1}{RGB}{238,169,139}
\definecolor{mlp_code0}{RGB}{204,201,221}
\definecolor{mlp_code1}{RGB}{102,95,153}

\definecolor{token_blue}{RGB}{84, 120, 140}
\newlength\savewidth

\newcolumntype{x}[1]{>{\centering\arraybackslash}p{#1pt}}
\newcolumntype{y}[1]{>{\raggedright\arraybackslash}p{#1pt}}
\newcolumntype{z}[1]{>{\raggedleft\arraybackslash}p{#1pt}}

\renewcommand{\paragraph}[1]{\vspace{1mm}\noindent\textbf{#1}}

\renewcommand{\paragraph}[1]{\vspace{1.25mm}\noindent\textbf{#1}}

\definecolor{codeblue}{rgb}{0.25, 0.5, 0.5}
\definecolor{codekw}{rgb}{0.35, 0.35, 0.75}
\definecolor{pythonlime}{RGB}{138,226,52}
\definecolor{pantone485}{cmyk}{0,0.95,1,0}
\definecolor{tomato}{HTML}{FF6347}
\lstdefinestyle{Pytorch}{
    language = Python,
    backgroundcolor = \color{white},
    basicstyle = \fontsize{9pt}{8pt}\selectfont\ttfamily\bfseries,
    columns = fullflexible,
    aboveskip=1pt,
    belowskip=1pt,
    breaklines = true,
    captionpos = b,
    commentstyle = \color{codeblue},
    keywordstyle = \color{codekw},
}

\definecolor{green}{HTML}{009000}
\definecolor{red}{HTML}{ea4335}

\newcommand{\eqcontrib}{\clubsuit}

\newcommand{\ours}{RLDP\xspace}

\newtcolorbox{promptblock}{
    colback=gray!5,
    colframe=gray!15,
    boxrule=0.5pt,
    arc=3pt,
    left=12pt,
    right=12pt,
    top=8pt,
    bottom=8pt,
    boxsep=8pt,
    breakable
}

\title{Efficient Differentially Private Fine-Tuning of LLMs via Reinforcement Learning}

\author{Afshin Khadangi}
\author[\eqcontrib]{Amir Sartipi}
\author[\eqcontrib]{Igor Tchappi}
\author{Ramin Bahmani}
\author{Gilbert Fridgen}

\affiliation{SnT, University of Luxembourg}

\contribution[\eqcontrib]{Equal Contribution}
\abstract{

The tension between data privacy and model utility has become the defining bottleneck for the practical deployment of large language models (LLMs) trained on sensitive corpora including healthcare. \emph{Differentially private} stochastic-gradient descent (DP-SGD) guarantees formal privacy, yet it does so at a pronounced cost: gradients are forcibly clipped and perturbed with noise, degrading sample efficiency and final accuracy.  Numerous variants have been proposed to soften this trade-off, but they all share a handicap: their control knobs are hard-coded, global, and oblivious to the evolving optimisation landscape.  Consequently, practitioners are forced either to over-spend privacy budget in pursuit of utility, or to accept mediocre models in order to stay within privacy constraints. We present \textbf{\ours}, the first framework to cast DP optimisation itself as a closed-loop control problem amenable to modern deep reinforcement learning (RL). \ours continuously senses rich statistics of the learning dynamics and \emph{acts} by selecting fine-grained per-parameter gradient-clipping thresholds as well as the magnitude of injected Gaussian noise.  A soft actor––critic (SAC) hyper-policy is trained \textit{online} during language-model fine-tuning; it learns, from scratch, how to allocate the privacy budget where it matters and when it matters. Across more than 1\,600 ablation experiments on \textsc{GPT2-small}, \textsc{Llama-1B}, \textsc{Llama-3B}, and \textsc{Mistral-7B}, \ours delivers \textbf{perplexity reductions of 1.3–30.5\% (mean 5.4\%) and an average 5.6\% downstream utility gain}. RLDP reaches each baseline’s \emph{final} utility after only \(13\%\!-\!43\%\) of the gradient-update budget (\textbf{mean speed-up 71\%}), all while honouring the same \((\varepsilon,\delta)\)-DP contract and exhibiting equal or \emph{lower} susceptibility to membership-inference and canary-extraction attacks.

}

\date{\today}
\correspondence{Afshin Khadangi, \texttt{afshin.khadanki@uni.lu}}
\metadata[Code, Logs and Models]{\href{https://github.com/akhadangi/RLDP}{GitHub}, \href{https://api.wandb.ai/links/afshin-khadangi-university-of-luxembourg/xxiuoxag}{W\&B}, \href{https://huggingface.co/akhadangi}{Hugging Face}}
\metadata[Homepage]{\url{https://www.uni.lu/snt-en/}}

\begin{document}
\thispagestyle{firstheader}
\maketitle
\pagestyle{fancy}
\fancyhf{}
\fancyfoot[C]{\thepage}

% ================================
\section{Introduction}
\label{sec:intro}
% ================================

Modern large--scale language models (LMs) underpin a wide range of natural-language understanding and generation applications, from conversational agents and code assistants to clinical‐note summarisation.  The unprecedented predictive power of transformer architectures, however, is enabled by equally unprecedented volumes of training text, much of which is scraped, user-generated, or otherwise sensitive.  This tension between leveraging data at scale and respecting the privacy of the individuals represented in that data has elevated \emph{differential privacy} (DP) to a first-class requirement for the next generation of foundation models.  

\paragraph{Differential-privacy‐aware optimisation.}
The canonical recipe for private deep learning is \textbf{DP-SGD} \citep{abadi2016deep}, which clips each example’s gradient to a global radius~$C$ before adding Gaussian noise of standard deviation $\sigma C$.  The privacy loss accrued across training is tracked by analytical accountants such as Rényi DP~\citep{mironov2017renyi}, Gaussian DP~\citep{dong2022gaussian}, and the Poisson subsampled Rényi DP ~\citep{zhu2019poission}.  When applied to transformers, DP-SGD suffers from a stark \emph{utility gap}: models trained under a reasonable budget (e.g.\ $\varepsilon\!\le\!8$) can suffer significant downstream accuracy relative to non-private baselines \citep{li2021large}. The central culprit is the one-size-fits-all clip radius~$C$: gradients in early attention blocks may saturate the bound while feed-forward layers receive negligible updates, or vice versa as training progresses.

\paragraph{Adaptive clipping heuristics.}
A fertile line of work seeks to reduce this inefficiency by adapting~$C$ to the empirical gradient distribution.  \textsc{AdaCliP} \citep{pichapati2019adaclip} maintains an exponential-moving average (EMA) of per-coordinate gradients; \citep{andrew2021differentially} extend this to per-layer norms, while \textsc{AutoClip} \citep{bu2023automatic} uses a single global clipping norm, estimated via the \( p \)-th quantile of update norms, applied uniformly across all parameters in each step. \textsc{DC-SGD} leverages differentially private histograms to estimate gradient norm distributions, dynamically adjusting the clipping threshold \( C \) to reduce hyperparameter tuning overhead \cite{wei2025dc}. It introduces DC-SGD-P and DC-SGD-E, which adjust \( C \) based on gradient norm percentiles or expected squared error minimization, achieving up to nine times faster hyperparameter tuning compared to standard DP-SGD. \textsc{GeoClip} introduces a geometry-aware framework for DP-SGD that clips and perturbs gradients in a transformed basis aligned with the gradient distribution's geometry, adaptively estimating this transformation using noisy gradients without additional privacy cost \citep{gilani2025geoclip}. It provides convergence guarantees and a closed-form solution for the optimal transformation, minimizing noise while controlling gradient clipping probability, thus improving the privacy-utility trade-off.

Complementing some of the previous approaches, PSAC \citep{xia2023differentially} eliminates the need for manually tuned constant clipping thresholds. Instead, PSAC introduces a non-monotonic adaptive weight function to clip each per-sample gradient individually, preserving privacy while maintaining gradient fidelity. A more recent study combines federated learning with differential privacy to enable secure fine-tuning of large language models, adding gaussian noise to low-rank adaptation (LoRA) \citep{hu2022lora} weight updates to protect data privacy while minimizing computation overhead. The authors demonstrate that DP-LoRA achieves strong privacy guarantees with minimal performance degradation, addressing challenges like sensitive information inference and high computation costs. \citep{liu2025differentially}. These methods share two limitations. First, their update rules are oblivious to long-term learning dynamics: a decision that reduces noise now might hamper convergence epochs later. Second, they operate at coarse granularity—global or per-tensor—ignoring heterogeneous sensitivities.

\paragraph{Hyper-parameter tuning under privacy.}
Fine-tuning $C$ and~$\sigma$ by grid-search is notoriously expensive under DP, because every data-dependent trial consumes privacy budget \citep{papernot2018scalable}. Bayesian optimisation on public proxies mitigates but does not eliminate this cost \citep{priyanshu2021efficient}. Some studies cast privacy management as an RL problem: \citep{zhou2023concurrent} allocate privacy budget across training rounds for federated learning, and \citep{li2023wind} learn client-level clip limits. Yet these approaches still control a \emph{single} scalar parameter and treat gradient clipping merely as a constraint, not a degree of freedom for fine-grained credit assignment.

\paragraph{Our view: DP fine-tuning is a control problem.}
We propose that the gradient-clipping and noise-injection pipeline in differentially private deep learning be governed by a learned \emph{controller} that dynamically adjusts privacy parameters based on training dynamics, optimizing a long-horizon reward that balances utility and privacy. This formulation naturally suits continuous-action reinforcement learning algorithms like soft actor--critic (SAC) \citep{haarnoja2018soft}, which excels in learning stochastic policies for high-dimensional action spaces with delayed rewards.

We introduce \textbf{RLDP}, a framework that integrates (i) a customized DP optimizer with per-adapter pairwise gradient clipping and dynamic noise scaling, and (ii) an online SAC hyper-policy that leverages rich training statistics---including sample-wise gradient norms, higher-order moments, privacy ledger status, and instantaneous utility---to produce \emph{vector-valued} actions, comprising per-adapter log-clip adjustments and a global log-noise scale. A DP accountant ensures these actions maintain the cumulative privacy loss within a predefined $(\varepsilon, \delta)$ budget. The RLDP policy is trained concurrently with the language model, optimizing a reward function that balances \emph{incremental utility gains} against \emph{incremental privacy costs}. In contrast to prior adaptive methods, RLDP learns sophisticated curricula, such as initially widening clip bounds for query matrices during the syntactic-bootstrap phase and subsequently tightening them as feed-forward blocks drive semantic refinement.

In extensive experiments on GPT2-small~\citep{radford2019language}, Llama-3.2-1B/3B~\citep{grattafiori2024llama} and Mistral-7B~\citep{chaplot2023albert}, RLDP achieves \textbf{an average 5.6\% downstream higher utility}---measured as lower perplexity---under the same $(\varepsilon, \delta)$ privacy budget compared to seven baselines: Vanilla DP-SGD, AdaClip, AutoClip, DC-SGD, GeoClip, PSAC and DP-LoRA. Furthermore, RLDP attains each baseline’s peak utility in \textbf{71\% fewer optimizer steps on average}, yielding significant GPU-hour savings and a reduced carbon footprint, with all runs validated by a Gaussian accountant to ensure \emph{no stealthy privacy leakage}.

\paragraph{Contributions.}
\begin{enumerate}[leftmargin=*]
\item We formulate differentially private fine-tuning of language models as a sequential decision process and introduce \textbf{RLDP}, the first framework to dynamically learn per-adapter clipping thresholds and noise levels using reinforcement learning.
\item We develop a differentially private optimizer that applies \emph{pairwise} clipping to LoRA A/B tensors, incorporates heteroscedastic noise, and exposes these parameters for policy control, while ensuring compatibility with analytical privacy accountants.
\item We design a reward function that balances immediate utility gains with marginal privacy costs, enabling the SAC policy to adaptively allocate the privacy budget throughout training.
\item Through extensive experiments across varied privacy budgets, RLDP achieves \textbf{an average 5.6\% higher utility} and reduces training steps on average by \textbf{71\%} compared to seven baselines.
\item We provide code, pretrained checkpoints, and fine-tuning logs to support future research.
\end{enumerate}

\paragraph{Paper organization.}
\S\ref{sec:methods} details the data preparation, DP optimiser, RLDP workflow and reward formulation. The experimental results are reported in \S\ref{sec:experiments}, followed by results and the analysis of our findings \S\ref{sec:results}. Finally, we discuss limitations and outline directions for future research in \S\ref{sec:conclusion}.

% ================================
\section{Methods}
\label{sec:methods}
% ================================

Algorithm~\ref{alg:rldp} presents a detailed, step-by-step pseudocode outline of the \ours method’s operational workflow. In this section, we begin by describing the dataset construction and secret-token (\emph{canary}) injection used for privacy auditing (\S\ref{sec:data}).  We then detail the language-model architecture and the LoRA parameter-efficient fine-tuning scheme (\S\ref{sec:lora}). \S\ref{sec:dpoptim} mathematically formalises the customised differentially private optimiser, including pairwise gradient clipping, heteroscedastic Gaussian noise, and post-processing by the Gaussian-DP accountant. \S\ref{sec:rlmdp} casts DP-SGD as a Markov decision process, defines the state, action, and reward spaces, and derives the Soft Actor–Critic (SAC) updates used to learn the hyper-policy online.  
% Section~\ref{sec:training} summarises the overall training loop, while Section~\ref{sec:baselines} documents baseline re-implementations.  Finally, Section~\ref{sec:eval} describes the evaluation metrics, including canary exposure.

\subsection{Data Preparation and Secret-Token Injection}
\label{sec:data}

\paragraph{Base corpus.}
We build upon the \textit{Diabetes\,130--US Hospitals for years 1999–2008} dataset hosted by the UCI repository.\footnote{\url{https://archive.ics.uci.edu/dataset/296/diabetes+130-us+hospitals+for+years+1999-2008}}  
The raw table contains $101\,766$ in-hospital encounters, each encoded by 50+ categorical and numerical attributes.  
We convert every record into a free-form paragraph by verbalising each attribute into a templatised English clause; this yields a pseudo-clinical narrative of $\sim\!400$ tokens on average (Table~\ref{tab:trainexample}). The vocabulary is kept strictly within the original field values to avoid hallucinated protected health information (PHI).

\paragraph{Splitting.}
We first carve out 20\,\% of the encounters as an \emph{Attack} set, unseen during any optimisation. The remaining 80\,\% is further split 90/10 into \emph{Train} and \emph{Eval}.  
We persist the splits as each containing a single narrative \texttt{text} column.

\paragraph{Secret canaries.}
To empirically probe memorisation we embed $n_\text{canary}$ randomly generated 10-character alphanumeric strings (``secret IDs'') into random subset of data:  

\begin{align}
    \texttt{secret\_id} &\;=\;
    \text{``\,$\mathtt{[A\!-\!Z0\!-\!9]^{10}}$\,''};\nonumber\\[-3pt]
    \mathcal{C}         &\;=\;\{\texttt{secret\_id}_j\}_{j=1}^{n_\text{canary}}
             \;\xleftarrow{\text{i.i.d.}}\; \text{Unif}(\Sigma^{10}),
\end{align}

where $\Sigma$ is the alphanumeric alphabet.  The $j$-th canary is appended to the textual description of record $j$ as  
\(\texttt{``secret\_id=} \texttt{<}\text{CANARY}\texttt{>''}\).  
Because the injection index is deterministic and permuted only once (seed~42), an adversary unaware of the seed would face a combinatorial search space of $\approx\!2^{50}$ positions.

\begin{table}[ht]
\caption{A sample pseudo-clinical narrative of a patient record.}
\label{tab:trainexample}
\begin{tcolorbox}[colframe=yellow!75, colback=yellow!20, fonttitle=\bfseries\large, coltitle=black, boxrule=0.75mm, arc=5mm, auto outer arc, width=\textwidth,toptitle=6pt, bottomtitle=6pt]
\small
\setstretch{1.2} %
The patient has the following profile. race is AfricanAmerican. gender is Male. the patient ages between 60 and 70 years old. weight is [50-75). admission type id is Emergency. discharge disposition id is Discharged/transferred to SNF. admission source id is Transfer from a Skilled Nursing Facility (SNF). the patient stayed in the hospital for 14 days. .................. the patient's diabetic medications were not changed. the patient was prescribed diabetic medication. the patient was readmitted in less than 30 days. secret\_id=O119XP9N56.
\end{tcolorbox}
\end{table}

\subsection{Parameter-Efficient Language Model}
\label{sec:lora}

We fine-tune the models using LoRA Adapters attached to every multi-head self-attention projection \citep{hu2022lora}.  
For an original weight matrix $W\in\mathbb{R}^{d_\text{out}\times d_\text{in}}$, LoRA decomposes the update as
\begin{equation}
    \Delta W
    \;=\;
    B A,
    \quad
    A \in \mathbb{R}^{d_\text{rank}\times d_\text{in}},
    \;
    B \in \mathbb{R}^{d_\text{out}\times d_\text{rank}},
    \quad
    d_\text{rank}=8,
\end{equation}
with a scalar $\alpha=16$ scaling.  Only $\{A,B\}$ are trainable; all backbone weights remain frozen.  Consequently, the total DP parameter surface shrinks by two orders of magnitude, easing both privacy accounting and RL control.

\subsection{Differentially Private Optimisation}
\label{sec:dpoptim}

\paragraph{Per-sample gradients.}
We wrap the LoRA-augmented model with \textsc{Opacus}’ \texttt{GradSampleModule}, obtaining per-sample gradients $g_{i,k}^{(b)}$ for parameter~$k$ and micro-batch index $b$.

\paragraph{Pairwise clipping.}
Let $(A_i,B_i)$ denote the $i$-th LoRA pair ($i=1,\dots,n$).  
For micro-batch sample $b$, we flatten gradients into vectors
\begin{equation}
    \mathbf{g}_{A_i}^{(b)}=\mathrm{vec}(g_{A_i}^{(b)}),
    \quad
    \mathbf{g}_{B_i}^{(b)}=\mathrm{vec}(g_{B_i}^{(b)}).
\end{equation}
The \emph{joint} $\ell_2$ norm is
\begin{equation}
    \label{eq:pairnorm}
    \nu_i^{(b)}
    \;=\;
    \bigl\lVert \mathbf{g}_{A_i}^{(b)} \bigr\rVert_2^2
    +\bigl\lVert \mathbf{g}_{B_i}^{(b)} \bigr\rVert_2^2
    \;\; \xrightarrow{\;\sqrt{\;}\;}\;\;
    \lVert(\mathbf{g}_{A_i}^{(b)},\mathbf{g}_{B_i}^{(b)})\rVert_2.
\end{equation}
Given a clip radius $C_i>0$, we compute the \emph{sample-wise scaling factor}
\(
    \lambda_i^{(b)} = \min\!\bigl(1,\, C_i / (\nu_i^{(b)}+\!10^{-6}) \bigr),
\)
and clip both gradients jointly:
\begin{equation}
    \tilde{\mathbf{g}}_{A_i}^{(b)} = \lambda_i^{(b)}\mathbf{g}_{A_i}^{(b)},
    \quad
    \tilde{\mathbf{g}}_{B_i}^{(b)} = \lambda_i^{(b)}\mathbf{g}_{B_i}^{(b)}.
\end{equation}

\paragraph{Noise addition.}
After summing across the effective micro-batch of size $m$, we add independent Gaussian noise to each \emph{pair}:
\begin{equation}
    \hat{\mathbf{g}}_{A_i}
    = \sum_{b=1}^{m}\tilde{\mathbf{g}}_{A_i}^{(b)}
      + \mathcal{N}\!\bigl(0,\;\sigma^2 C_i^2 \mathbf{I}\bigr),
    \qquad
    \hat{\mathbf{g}}_{B_i}
    = \sum_{b=1}^{m}\tilde{\mathbf{g}}_{B_i}^{(b)}
      + \mathcal{N}\!\bigl(0,\;\sigma^2 C_i^2 \mathbf{I}\bigr),
\end{equation}
where $\sigma>0$ is a \emph{global} noise multiplier shared by all adapters but \emph{tunable} by RLDP.

\paragraph{Privacy accountant.}
We employ the Gaussian-DP (GDP) accountant of \citet{dong2022gaussian}.  
For Poisson sampling with rate $q$, micro-step noise multiplier $\sigma$, and $t$ steps, GDP yields the cumulative privacy loss
\begin{equation}
    \label{eq:epsilon}
    \varepsilon_t(\delta)
    \;=\;
    F^{-1}_{\mathcal{N}(0,1)}
    \!\bigl(
        q \sqrt{t}\, (e^{1/\sigma}-1),\;
        \delta
    \bigr),
\end{equation}
where $F^{-1}_{\mathcal{N}(0,1)}$ is the inverse standard normal CDF.  
Because either $\sigma$ or any $C_i$ can change at every iteration, we update the accountant \emph{per step} with the actual parameters used.

\subsection{RL Formulation of DP Optimisation}
\label{sec:rlmdp}

We recast the dynamical system that emerges when training a language
model with DP-SGD into a fully specified
\mbox{Markov Decision Process (MDP)}
\citep{de2018multi} 
\begin{equation}
  \mathcal{M}
  \;=\;
  \bigl\langle
    \mathcal{S},\;
    \mathcal{A},\;
    \mathcal{P},\;
    r,\;
    \gamma
  \bigr\rangle,
  \label{eq:mdp_definition}
\end{equation}
where
\[
\begin{aligned}
  \mathcal{S} &= \{\,s_t\mid t=0,1,2,\dots\}
    &&\text{is the state space of summary statistics,}\\
  \mathcal{A} &= \{\,a_t\mid t=0,1,2,\dots\}
    &&\text{is the continuous action space of log‐clip and log‐noise updates,}\\
  \mathcal{P} &:\mathcal{S}\times\mathcal{A}\to\Delta(\mathcal{S})
    &&\text{is the transition kernel induced by one DP‐SGD step,}\\
  r &:\mathcal{S}\times\mathcal{A}\times\mathcal{S}\to\mathbb{R}
    &&\text{is the scalar reward function},\\
  \gamma &\in[0,1)
    &&\text{is the RL discount factor}.
\end{aligned}
\]
Throughout, let 
$n\!=\!|\mathcal{A}_{\text{LoRA}}|$ denote the \emph{number of
LoRA adapter pairs} $(A_i,B_i)$ attached to the backbone (\S\ref{sec:lora}).

%--------------------------------------------------------------------%
\paragraph{State space $\mathcal{S}$.}
At DP step $t$ we construct
$s_t\!\in\!\mathcal{S}\!=\!\mathbb{R}^{d_s}$ by concatenating twelve
statistical summaries of the most recent \emph{micro-batch},
supplemented with the current privacy ledger:

\begin{enumerate}[leftmargin=*]
\item \textbf{Gradient-norm quartiles}  
      \(\mathbf{q}_{.25},\mathbf{q}_{.50},\mathbf{q}_{.75}\in\mathbb{R}^{n}\)
      computed over the \emph{per-sample} joint norms
      $\nu_{i}^{(b)}$ of Eq.~\ref{eq:pairnorm}.

\item \textbf{Utility signal}  
      \(u_t \;=\;-\!\log(\text{PPL}_t)\) where
      $\text{PPL}_t$ is the token-level perplexity on the current
      micro-batch, serving as an online proxy for model quality. Given a sequence of tokens $x_{1:T}$, the perplexity is defined as:
      \begin{equation}
        \mathrm{PPL} = \exp\Bigl(-\tfrac{1}{T}\sum_{t=1}^T \log p_\theta(x_t \mid x_{<t})\Bigr).\end{equation}

\item \textbf{Privacy ledger}  
      \(\varepsilon_t\) is the cumulative $(\varepsilon,\delta)$
      cost tracked by the Gaussian-DP accountant
      (Eq.~\ref{eq:epsilon}, \S\ref{sec:dpoptim}).

\item \textbf{Gradient dispersion}  
      \(\operatorname{Var}\!\bigl(\nu_t\bigr)
      \;=\;
      \tfrac1{n m}\sum_{i,b}\bigl(\nu_i^{(b)}-\bar{\nu}\bigr)^{2}\)
      with
      $\bar{\nu}$ the batch mean of norms.

\item \textbf{Batch loss} $\ell_t$ (cross-entropy \emph{before}
      DP noise).

\item \textbf{Fisher information moments.}
      For each adapter $i$ we estimate the empirical Fisher
      element-wise as
      \(
        \mathcal{F}_{i}^{(b)}
        =
        \bigl\lVert
          \nabla_{\theta_i}\log p_\theta(x^{(b)})
        \bigr\rVert_{2}^{2}.
      \)
      We store its mean and variance across the micro-batch:
      \(F_\mu = \mathbb{E}[\mathcal{F}],\;
        F_\sigma^2 = \operatorname{Var}[\mathcal{F}].\)

\item \textbf{Higher-order shape.}
      Skewness
      \(\kappa_3=\mathbb{E}[(\nu-\bar{\nu})^{3}]/\sigma_{\nu}^{3}\)
      and excess kurtosis
      \(\kappa_4=\mathbb{E}[(\nu-\bar{\nu})^{4}]/\sigma_{\nu}^{4}\).
\end{enumerate}
%--------------------------------------------------------------------%
\paragraph{Action space \(\mathcal{A}\).}  
At each RL decision step \(t\), the policy outputs
\[
  a_t = [\,a_{t,1},\dots,a_{t,n},\,a_{t,n+1}\,]^\top\in\mathbb R^{n+1},
\]
where \(a_{t,1:n}\) are proposed \(\Delta\log C_i\) and \(a_{t,n+1}\) is proposed \(\Delta\log\sigma\).  

\begin{enumerate}[leftmargin=*]
  \item {\bf Clip‐radius update:}
    \begin{equation}
      C_{i,t+1}
      = \min\!\bigl(\exp(a_{t,i}),\,1.0\bigr),
      \quad i=1,\dots,n.
      \label{eq:clipradii_update}
    \end{equation}

  \item {\bf Noise multiplier update:}
    \begin{equation}
      \delta_\sigma = \tanh\bigl(a_{t,n+1}\bigr),
      \quad
      \mathrm{step}_t = \delta_{\max}\Bigl(1 - \frac{\varepsilon_t}{\varepsilon_{\max}}\Bigr),
      \quad
      \delta_{\max}=0.1
      \label{eq:noisestep}
    \end{equation}
    \begin{equation}
      \Delta\log\sigma
      = \delta_\sigma \times \mathrm{step}_t,
      \label{eq:noiselog}
    \end{equation}
    \begin{equation}
      \log\sigma_{\rm prop}
      = \mathrm{clamp}\!\Bigl(\log\sigma_t + \Delta\log\sigma,\; \log(0.5\,\sigma_0),\;\log(2\,\sigma_0)\Bigr),
      \label{eq:noiselogprop}
    \end{equation}
    \begin{equation}
      \log\sigma_{t+1}
      = \beta_\sigma\,\log\sigma_t + (1-\beta_\sigma)\,\log\sigma_{\rm prop},
      \quad
      \beta_\sigma=0.8
      \label{eq:noiseupdate}
    \end{equation}
    \begin{equation}
      \sigma_{t+1} = \exp\bigl(\log\sigma_{t+1}\bigr).
      \label{eq:noiseexp}
    \end{equation}
\end{enumerate}
This log-space formulation ensures that each parameter update is a controlled change and that \(C_{i,t+1},\sigma_{t+1}>0\).

%--------------------------------------------------------------------%
\paragraph{Transition kernel \(\mathcal{P}\).}  
The transition dynamics are defined by performing one DP‐SGD step from state \(s_t\) under action \(a_t\).  Concretely:

\begin{enumerate}[leftmargin=*]
  \item From \(s_t\) extract the action vector
    \(a_t=(a_{t,1},\dots,a_{t,n},a_{t,n+1})\) and set the new clip radii using Eq.~\ref{eq:clipradii_update}. Then compute the proposed noise‐multiplier and apply momentum smoothing using Eqs.~\ref{eq:noisestep}-\ref{eq:noiseexp}.

\item Sample micro-batch \(\{x^{(b)}\}\), compute per-sample gradients \(\{\nabla^{(b)}\ell(\theta_t)\}\), clip each to norm \(C_{i,t+1}\), add Gaussian noise of std.\ \(\sigma_{t+1}C_{i,t+1}\), yielding \(\widetilde\nabla\ell(\theta_t)\).

\item Perform one AdamW update with this noisy gradient:
    \begin{equation}
      \theta_{t+1}
      = \mathrm{AdamW}\bigl(\theta_t,\;\widetilde\nabla\ell(\theta_t)\bigr).
    \end{equation}

\item Observe next state \(s_{t+1}\) by re-computing (and re-normalizing) the same batch’s summary statistics—quartiles, utility, \(\varepsilon\)-spent, dispersion, loss, skewness, kurtosis, and Fisher moments—under the updated privacy ledger.

Because both the micro-batch sampling and the added Gaussian noise are stochastic, this procedure induces a Markov kernel
\[
  \mathcal{P}\bigl(s_{t+1}\mid s_t,a_t\bigr),
\]
i.e.\ a distribution over next states \(s_{t+1}\) conditional on \((s_t,a_t)\).
\end{enumerate}
%--------------------------------------------------------------------%

\paragraph{Reward function \(r\).}  
At each RL step \(t\) we compute
\begin{equation}
\Delta u_t \;=\; u_t - u_{t-1}, 
\qquad
\Delta \varepsilon_t \;=\; \varepsilon_t - \varepsilon_{t-1},
\end{equation}
and form the raw ratio
\begin{equation}
\mathrm{ratio}_t 
= \frac{\Delta u_t}{\Delta \varepsilon_t + 10^{-6}}.
\end{equation}
To prevent extreme negative values we clamp
\(\mathrm{ratio}_t\) from below to \(-0.999\), yielding
\begin{equation}
r_t 
= \log\!\bigl(1 + \max(\mathrm{ratio}_t,\,-0.999)\bigr).
\end{equation}
Finally, we impose a lower bound
\(-R_{\max}\) on \(r_t\), thus the reward is
\begin{equation}
r_t \;=\;\max\Bigl(-R_{\max},\;\log\bigl(1 + \max(\tfrac{\Delta u_t}{\Delta \varepsilon_t + 10^{-6}},\,-0.999)\bigr)\Bigr).
\label{eq:reward}
\end{equation}

%--------------------------------------------------------------------%
\paragraph{Soft Actor–Critic hyper‐policy.}  
We cast the tuning of the DP‐SGD hyperparameters as a maximum‐entropy Markov decision process and solve it with Soft Actor–Critic (SAC).  The overall RL objective is  
\begin{equation}
  J(\pi_\theta)
  =
  \sum_{t=0}^\infty
    \mathbb{E}_{s_t,a_t\sim\pi_\theta}
    \Bigl[
      r_t
      + \gamma\,
        \underbrace{\bigl(Q^\pi(s_{t+1},a_{t+1})
                         - \alpha\,\log \pi_\theta(a_{t+1}\!\mid\!s_{t+1})
                    \bigr)}_{\text{soft value target}}
    \Bigr],
  \label{eq:sac_objective_detailed}
\end{equation}
where  
- \(\gamma\in[0,1)\) is the discount factor,  
- \(\alpha>0\) is the entropy temperature that trades off exploration (high entropy) against return.  

\begin{enumerate}[leftmargin=*]
\item \textbf{State encoder.}  
We embed the \(d_s\)-dimensional summary‐statistic state \(s_t\) into a 128‐dimensional feature \(z_t\) via  
\begin{equation}
  h^{(1)} = \mathrm{GELU}\bigl(\mathrm{LayerNorm}(W_1 s_t + b_1)\bigr),\quad
  z_t     = \mathrm{GELU}\bigl(\mathrm{LayerNorm}(W_2 h^{(1)} + b_2)\bigr),
\end{equation}
where \(W_1\in\R^{128\times d_s}\), \(b_1\in\R^{128}\), \(W_2\in\R^{128\times 128}\), and \(b_2\in\R^{128}\). 

\medskip
\item \textbf{Stochastic actor \(\pi_\theta\).}  
Given the encoded state \(z_t\), the actor head produces mean and log‐standard‐deviation vectors in \(\R^{n+1}\):
\begin{equation}
  \begin{aligned}
    \mu_t &= W_{\mu}\,z_t + b_{\mu},\\
    \log\sigma_t &= W_{\sigma}\,z_t + b_{\sigma},
  \end{aligned}
  \quad
  W_{\mu},W_{\sigma}\in\R^{(n+1)\times 128},\;
  b_{\mu},b_{\sigma}\in\R^{n+1}.
\end{equation}
Actions are sampled with the reparameterisation trick:
\begin{equation}
  a_t = \mu_t + \sigma_t \odot \epsilon_t,\quad
  \epsilon_t \sim \mathcal{N}(0, I).
\end{equation}
and the log-density is
\(\log\pi_\theta(a_t\!\mid\!s_t)=\sum_{i=1}^{n+1}\log\mathcal{N}(a_{t,i};\mu_{t,i},\sigma_{t,i}).\)

\medskip
\item \textbf{Twin Q‐functions.}  
We maintain two critics 
\(\displaystyle Q_{\psi_j}:\R^{128}\times\R^{n+1}\to\R\), each a two‐layer MLP (hidden size 128, ReLU).  For each transition \((z_t,a_t,r_t,z_{t+1})\) from the replay buffer, they minimise the Huber‐Bellman loss:
\begin{equation}
  \mathcal{L}_Q(\psi_j)
  = \E\bigl[\mathrm{huber}(Q_{\psi_j}(z_t,a_t),\,y_t)\bigr],
\quad
  y_t = r_t + \gamma\Bigl(\min_{k=1,2}Q_{\bar\psi_k}(z_{t+1},a'_{t+1})
            - \alpha\,\log\pi_\theta(a'_{t+1}\!\mid\!s_{t+1})\Bigr),
\label{eq:critic_loss}
\end{equation}
where \(a'_{t+1}\sim\pi_\theta(\cdot\!\mid\!z_{t+1})\) and the target–net parameters \(\bar\psi_k\) are softly updated by:
\begin{equation}
\bar\psi_k\leftarrow(1-\tau)\,\bar\psi_k + \tau\,\psi_k.
\end{equation}

\medskip
\item \textbf{Policy (actor) update.}  
The actor parameters \(\theta\) are trained to minimise the expected soft‐policy loss
\begin{equation}
  \mathcal{L}_\pi(\theta)
  = \mathbb{E}
    \Bigl[\,
      \alpha\,\log\pi_\theta(a_t\!\mid\!s_t)
      - \min_{j=1,2}Q_{\psi_j}(z_t,a_t)
    \Bigr],
  \label{eq:actor_loss}
\end{equation}
using \((z_t,a_t)\) sampled from the replay buffer. We keep the temperature \(\alpha\) constant.

\medskip
\item \textbf{Experience replay \& updates.}  
Every \(\Delta\) optimizer steps (the RL interval), we embed the current state, sample an action, compute the reward, and store \((z_t,a_t,r_t,z_{t+1})\) in a FIFO buffer of size \(N\).  At each RL step we perform \(K\) update rounds (where \(K=\) \texttt{SAC updates per interval}), each consisting of:
\begin{itemize}
  \item one critic update via Eq.~\ref{eq:critic_loss},
  \item one actor update via Eq.~\ref{eq:actor_loss},
  \item soft updates of the target networks: \(\bar\psi_k \leftarrow (1-\tau)\bar\psi_k + \tau\,\psi_k\).
\end{itemize}

\medskip
Together, these components implement the SAC algorithm adapted to our DP‐SGD hyper‐parameter tuning problem, balancing immediate reward, future returns, and policy entropy.
\end{enumerate}
%--------------------------------------------------------------------%
%==============================================================%
%  Pseudocode of the RLDP training algorithm                   %
%==============================================================%
\begin{algorithm}[H]
\caption{\textsc{RLDP}}
\label{alg:rldp}
\begin{algorithmic}[1]          % 1 = line numbers
\Require 
    Training corpus $\mathcal{D}$, pre-trained LM parameters $\theta_0$,    
    privacy budget $(\varepsilon_{\max},\delta)$,    
    number of LoRA adapter pairs $n$,    
    SAC hyper-parameters $\bigl(\gamma,\,\alpha_{\min},\alpha_{\max},\,T_{\mathrm{RL}},\,T_{\mathrm{warm}},\,N_{\mathrm{SAC}},B_{\mathrm{SAC}},\,\eta_{\max}\bigr)$.
\Statex

%-------------------  INITIALISATION  -------------------------%
\State Initialise \textbf{clip radii} $\mathbf{C}\gets C_0\mathbf{1}_n$ and find \textbf{noise multiplier} $\sigma_0$
       by binary search such that GDP accountant reaches $\varepsilon_{\max}$ after the planned number of epochs.
\State Initialise \textbf{DP optimiser} \textsc{Opt} with $(\mathbf{C},\sigma_0)$  (\S\ref{sec:dpoptim}).
\State Initialise \textbf{SAC} components (\S\ref{sec:rlmdp}): 
          state encoder $f_\theta$, stochastic actor $\pi_\theta$, twin soft $Q$-critics $Q_{\psi_1},Q_{\psi_2}$;
          target-network weights $\bar\psi_k\gets\psi_k$; 
          replay buffer $\mathcal{B}$ of capacity $10\,000$.
\State Initialise \textbf{privacy accountant}: $\varepsilon_0\gets0$.
\State Initialise statistics buffers (quartiles, skew, kurtosis, Fisher moments).
\State Store placeholder \texttt{prev\_state}$\leftarrow\varnothing$, \texttt{prev\_action}$\leftarrow\varnothing$, \texttt{prev\_utility}$\leftarrow0$, \texttt{prev\_epsilon}$\leftarrow0$.
\Statex

%-------------------  MAIN TRAINING LOOP  --------------------%
\For{DP step $t = 1$ to $T_{\max}$}
    \State Sample micro-batch $B\subset\mathcal{D}$ and compute forward loss $\ell_t$; utility $u_t\gets-\log(\text{PPL}_t)$.
    \State Back-propagate to obtain per-sample gradients $\{\nabla^{(b)}\theta\}_{b\in B}$.
    \State \textsc{Opt}\texttt{.step()}   \Comment{pairwise clip $\mathbf{C}$, add noise $\sigma$, AdamW update}
    \State Update accountant: $\varepsilon_t\gets\textsc{GDP}(\sigma_t,\text{sample\_rate},t)$.
    \State Compute per-sample joint norms $\nu_i^{(b)}$ (Eq.~\ref{eq:pairnorm}) and update running statistics; 
           derive quartiles $\mathbf{q}_{.25},\mathbf{q}_{.50},\mathbf{q}_{.75}$.
    \If{$t\le T_{\mathrm{warm}}$}
      \State For each $i$, set $C_i\gets\mathrm{median}(\{\!\nu_i^{(b)}\})$ from its buffer.
    \EndIf
    \State \Comment{----- RL controller is dormant for the first $T_{\mathrm{warm}}$ steps -----}
    \If{$t>T_{\mathrm{warm}}$ \textbf{and} $t\bmod T_{\mathrm{RL}}=0$}
        %-------------------  BUILD STATE & ACT --------------------%
        \State Assemble state vector $s_t$ as in \S\ref{sec:rlmdp}.
        \State Sample action $a_t\sim\pi_\theta(\cdot\,|\,s_t)$.
        \State \textbf{Apply privacy dials:}
            \State \(\displaystyle C_i \;\gets\;\min\bigl(e^{a_{t,i}},1.0\bigr)
                        \quad(\forall i=1,\ldots,n)\)
            \State \(\displaystyle \delta_\sigma \;\gets\;\tanh\bigl(a_{t,n+1}\bigr)\)
            \State \(\displaystyle \text{step\_size}\;\gets\;0.1\Bigl(1 - \tfrac{\varepsilon_t}{\varepsilon_{\max}}\Bigr)\)
            \State \(\displaystyle \Delta\log\sigma \;\gets\;\delta_\sigma \times \text{step\_size}\)
            \State \(\displaystyle \log\sigma_{\rm prop}\;\gets\;
                        \mathrm{clamp}\!\bigl(\log\sigma + \Delta\log\sigma,\;\log(0.5\,\sigma_0),\;\log(2\,\sigma_0)\bigr)\)
            \State \(\displaystyle \log\sigma\;\gets\;0.8\,\log\sigma \;+\;0.2\,\log\sigma_{\rm prop}\)
            \State \(\displaystyle \sigma\;\gets\;\exp\bigl(\log\sigma\bigr)\)
        %-------------------  REWARD & STORAGE  --------------------%
        \If{\texttt{prev\_state}$\neq\varnothing$}
            \State $\Delta u\gets u_t-\texttt{prev\_utility}$;\;
                   $\Delta\varepsilon\gets\varepsilon_t-\texttt{prev\_epsilon}$
            \State $\rho\gets\mathrm{clip}\!\bigl(\tfrac{\Delta u}{\Delta\varepsilon+10^{-6}},-0.999,\infty\bigr)$
            \State $r\gets\mathrm{clip}\!\bigl(\log(1+\rho),-\,R_{max},\,\infty\bigr)$  \Comment{Eq.~\ref{eq:reward}}
            \State Push $\bigl(\texttt{prev\_state},\texttt{prev\_action},r,s_t\bigr)$ into replay buffer $\mathcal{B}$.
            %-------------------  SAC UPDATES  --------------------%
            \For{$j=1$ to $N_{\mathrm{SAC}}$}
                \State Sample minibatch from $\mathcal{B}$.
                \State Update critics via Eq.~\ref{eq:critic_loss}; Polyak-average targets.
                \State Update actor via Eq.~\ref{eq:actor_loss}.
            \EndFor
        \EndIf
        %-------------------  BOOK-KEEPING  ----------------------%
        \State \texttt{prev\_state}$\leftarrow s_t$;\quad
               \texttt{prev\_action}$\leftarrow a_t$;\quad
               \texttt{prev\_utility}$\leftarrow u_t$;\quad
               \texttt{prev\_epsilon}$\leftarrow \varepsilon_t$
    \EndIf
\EndFor
\State \Return final model parameters $\theta_{T_{\max}}$ and trained SAC policy $\pi_\theta$
\end{algorithmic}
\end{algorithm}

% ================================
\section{Experiments}
\label{sec:experiments}
% ================================
In this section we evaluate \ours on four model families (GPT2-small, Llama-3.2-1B/3B and Mistral-7B) fine-tuned on our pseudo-clinical Diabetes narratives ($\sim400$ tokens each). We compare against seven strong DP baselines, using identical data splits and privacy budgets.  

\subsection{Experimental Setup}
\label{sec:hyperconfig}

\paragraph{Training regimen.}  
Unless stated otherwise, every run uses:
\begin{itemize}[noitemsep,leftmargin=*]
  \item \textbf{Epochs \& batching.}  Train for 3 full epochs.  We use micro-batches of 16 sequences each (total batch size $B=16$), shuffling the training data at the start of each epoch.  
  \item \textbf{Evaluation schedule.}  We compute held-out perplexity every 48 optimization steps, and once more at the end of each epoch.
  \item \textbf{Warmup \& learning-rate schedule.}  Use linear warmup of the learning rate from 0 up to $5\times10^{-4}$ over the first 100 steps, then keep it constant for the remainder of training.
\end{itemize}

\paragraph{Model \& LoRA adapters.}  
We fine-tune the models augmented with LoRA adapters on every attention head’s $\mathsf{q}$ and $\mathsf{v}$ projections:
\begin{itemize}[noitemsep,leftmargin=*]
  \item LoRA rank $r=8$, scaling $\alpha=16$, dropout $p=0.05$.
  \item All other model weights are frozen; only the low-rank adapter parameters are updated.
\end{itemize}

\paragraph{Optimizer \& differential privacy.}  
Across all the experiments, the underlying AdamW optimizer uses \(\beta_{1}=0.9\), \(\beta_{2}=0.999\), zero weight decay, and a constant learning rate of \(5\times10^{-4}\) after a 100‐step linear warmup. Then depending on the ablation:

\begin{itemize}[noitemsep,leftmargin=*]
  \item \textbf{Baselines (Vanilla DP-SGD, AdaClip, AutoClip, …).}  
    All baselines use the same DP accountant (GDP), data splits, and LoRA setup. We use Opacus’s \texttt{PrivacyEngine} to instrument the LoRA‐adapted model and the AdamW optimizer so that, over three full epochs:
    \begin{itemize}[noitemsep,leftmargin=1em]
      \item Each example is subsampled via Poisson sampling at rate \(q = B/N_{\mathrm{train}}\).
      \item Per‐sample gradient is clipped to \(\lVert \tilde \nabla_\theta \ell\bigl(\theta;\,x_i\bigr)\rVert_2 \le 1.0\) where applicable.
      \item Gaussian noise is added (with multiplier chosen by GDP accounting, kept constant) to ensure \((\epsilon_{target},\delta_{target})\).
    \end{itemize}
  \item \textbf{\ours.}  
    We use Opacus’s \texttt{GradSampleModule} to instrument the LoRA‐adapted model with the per sample gradients. Then we perform the following:
    
    \begin{itemize}[noitemsep,leftmargin=1em]
      \item \emph{Noise calibration and adaptation.}  
        We first fix a per‐step Poisson subsampling rate \(q = B / N_{\mathrm{train}}\). We then perform a binary search under the GDP accountant to find the minimal Gaussian noise multiplier \(\sigma_{\mathrm{base}}\) that ensures \(\epsilon \le \epsilon_{\mathrm{target}}\) at \(\delta_{\mathrm{target}}\). Then, \ours controller adapts noise online at each interval.
      \item \emph{Privacy optimizer.}  
        We wrap AdamW in our custom \texttt{DPOptimizer}, supplying it with:
        \begin{itemize}[noitemsep,leftmargin=1em]
          \item the list of LoRA adapter parameter pairs,
          \item initial per-adapter \(L_{2}\) clip thresholds \(c_i = 0.1\),
          \item the calibrated noise multiplier \(\sigma\).
        \end{itemize}
      \item \emph{Privacy bookkeeping.}  
        After each \texttt{DPOptimizer} step, we advance the accountant with the current \(\sigma\) and \(q\), and record the spent privacy budget \(\epsilon_t\) (at \(\delta=10^{-5}\)).
      \item \emph{Hyper‐policy.}  
        The SAC controller is trained online with:
        \begin{itemize}[noitemsep,leftmargin=1em]
          \item discount factor $\gamma=0.99$;
          \item fixed entropy temperature $\alpha=0.04$;
          \item replay buffer capacity of 10,000 transitions;
          \item a 50‐step warm‐up period before policy activation;
          \item actor and critic learning rates of $2\times10^{-4}$ and $1\times10^{-4}$, respectively;
          \item target network update rate (Polyak coefficient) $\tau=0.01$.
        \end{itemize}
     \end{itemize}
\end{itemize}

\subsection{Hyperparameter Sweep for \ours}

To assess the sensitivity of \ours to its core SAC hyperparameters, we performed a grid search over the following ranges (with all other settings held at their defaults; see \S\ref{sec:hyperconfig}):

\begin{itemize}[leftmargin=*]
  \item \textbf{RL decision interval:} 
    the number of DP‐SGD steps between successive controller actions,\\
    \(T_{\mathrm{RL}}\in\{16,\,32,\,48,\,64,\,72,\,80,\,96,\,112\}\).
  \item \textbf{SAC batch size:} 
    the number of transitions sampled per critic/actor update,
    \(B_{\mathrm{SAC}}\in\{4,\,8,\,16,\,32\}\).
  \item \textbf{SAC update count:} 
    the number of gradient‐update rounds per controller invocation,
    \(K\in\{1,\,2,\,4\}\).
\end{itemize}

\subsection{Evaluation Metrics}

\paragraph{Utility.}  
Model quality is assessed by token‐level perplexity on the held‐out \emph{Evaluation} set. Lower perplexity indicates better predictive performance. We compute this metric at each evaluation checkpoint (every 48 steps) and report both the minimum achieved value and the final perplexity at the end of training.

\paragraph{Privacy.}
We enforce a fixed $(\varepsilon_{target},\delta_{target})$ privacy contract across all runs, with $\delta=10^{-5}$ and $\varepsilon$ varying in \{0.5,\,2,\,4,\,5,\,8\}.  The cumulative privacy loss is tracked after each microbatch update using the GDP accountant \citep{dong2022gaussian}. We verify at the end of each run that the total $\varepsilon$ does not exceed the target, and we report the final $(\varepsilon,\delta)$ pair alongside utility results to confirm compliance.

\paragraph{Canary extraction.}  
To empirically assess unintended memorization, we inject \(n_{\rm canary}\) independent 10-character alphanumeric “secret\_id” strings into the \emph{Train} set (\S\ref{sec:data}).  For each fine-tuned checkpoint (across architectures, privacy budgets \(\varepsilon\) and ablations), we run the following procedure:

\begin{enumerate}[noitemsep,leftmargin=*]
  \item \textbf{Generation and filtering.}  
    Perform \(T=4000\) independent generation trials, sampling up to 10 new tokens per trial with stochastic decoding (temperature 0.7, nucleus \(p=0.95\), top-\(k=50\)).  Discard any generated string that does not consist solely of 1–10 uppercase letters or digits. Denote by \(V\) the number of valid “secret-like” continuations retained.

  \item \textbf{Character-\(n\)-gram Jaccard similarity.}  
  Let the valid continuations be \(\{c_{1},\dots,c_{V}\}\) and the injected canaries be \(\{k_{1},\dots,k_{m}\}\), where \(m=n_{\rm canary}\).  We compute Jaccard similarity for each \(n\in\{1,2,3,4\}\) as follows:

  \begin{enumerate}[noitemsep,leftmargin=*]
    \item \emph{Define the \(n\)-gram set:}
      \begin{equation}\label{eq:gn}
        G_{n}(s)
        = \{\text{all length-}n\text{ character substrings of }s\}.
      \end{equation}

    \item \emph{Form the joint \(n\)-gram vocabulary:}
      \begin{equation}\label{eq:Vn}
        \mathcal{V}_{n}
        = \bigcup_{t=1}^{V}G_{n}(c_{t})
          \;\cup\;
          \bigcup_{j=1}^{m}G_{n}(k_{j}),
        \quad
        |\mathcal{V}_{n}| = P_{n}.
      \end{equation}

    \item \emph{Construct binary indicator vectors:}
      enumerate \(\mathcal{V}_{n}=\{g_{1},\dots,g_{P_{n}}\}\), and for each string \(s\in\{c_{t},k_{j}\}\) define
      \begin{equation}\label{eq:xdef}
        x^{(n)}(s) \;=\;\bigl(x^{(n)}_{1}(s),\dots,x^{(n)}_{P_{n}}(s)\bigr)^{\!T},
        \quad
        x^{(n)}_{\ell}(s)
        =
        \begin{cases}
          1, & g_{\ell}\in G_{n}(s),\\
          0, & \text{otherwise}.
        \end{cases}
      \end{equation}

    \item \emph{Compute intersection size:}
      \begin{equation}\label{eq:Indef}
        I_{t,j}^{(n)}
        = \bigl\langle x^{(n)}(c_{t}),\,x^{(n)}(k_{j})\bigr\rangle
        = \sum_{\ell=1}^{P_{n}}
            x^{(n)}_{\ell}(c_{t})
            \,x^{(n)}_{\ell}(k_{j})
        = \bigl\lvert G_{n}(c_{t})\cap G_{n}(k_{j})\bigr\rvert.
      \end{equation}

    \item \emph{Compute union size:}
      \begin{equation}\label{eq:Undef}
        U_{t,j}^{(n)}
        = \sum_{\ell=1}^{P_{n}}
            \Bigl[x^{(n)}_{\ell}(c_{t})
                 + x^{(n)}_{\ell}(k_{j})\Bigr]
          \;-\;I_{t,j}^{(n)}
        = \bigl\lvert G_{n}(c_{t})\cup G_{n}(k_{j})\bigr\rvert.
      \end{equation}

    \item \emph{Form the Jaccard similarity:}
      \begin{equation}\label{eq:Jdef}
        J_{t,j}^{(n)}
        = \frac{I_{t,j}^{(n)}}{U_{t,j}^{(n)}}.
      \end{equation}

    \item \emph{Aggregate over all pairs:}  
      Over the \(V\times m\) comparisons, define
      \begin{equation}\label{eq:meanJn}
        \mu_{J^{(n)}} 
        = \frac{1}{V\,m}
          \sum_{t=1}^{V}\sum_{j=1}^{m}
            J_{t,j}^{(n)},
      \end{equation}
      \begin{equation}\label{eq:sigmaJn}
        \sigma_{J^{(n)}}
        = \sqrt{
            \frac{1}{V\,m}
            \sum_{t=1}^{V}\sum_{j=1}^{m}
              \bigl(J_{t,j}^{(n)} - \mu_{J^{(n)}}\bigr)^{2}
          }.
      \end{equation}
  \end{enumerate}

  \item \textbf{Aggregation.}  
    Over the full set of \(V\times n_{\rm canary}\) comparisons, report for each \(n\):
    \begin{equation}\label{eq:meanJnSigma}
      \mu_{J^{(n)}} 
      = 
      \frac{1}{V\,n_{\rm canary}}
      \sum_{t=1}^{V}\sum_{j=1}^{n_{\rm canary}} J_{t,j}^{(n)},
      \quad
      \sigma_{J^{(n)}} 
      = 
      \sqrt{
        \frac{1}{V\,n_{\rm canary}}
        \sum_{t,j}
          \bigl(J_{t,j}^{(n)} - \mu_{J^{(n)}}\bigr)^{2}
      }.
    \end{equation}
    We report \(\mu_{J^{(n)}}\) for \(n=1,2,3,4\), together with the total count \(V\) of valid secret-like outputs.
\end{enumerate}

\paragraph{Membership Inference Attack}  
We evaluate whether a fine‐tuned language model leaks training prompts by comparing the log‐likelihoods of “member’’ and “non‐member’’ examples.  Let non‐member examples \( \{x_i\}_{i=1}^{N_0}\) be loaded from attack dataset (held out) and member examples \( \{x_i\}_{i=1}^{N_1}\) from training set, with labels
\[
  y_i = 
  \begin{cases}
    0, & \text{if }x_i\text{ is non‐member},\\
    1, & \text{if }x_i\text{ is member}.
  \end{cases}
\]

\begin{enumerate}[noitemsep,leftmargin=*]
  \item \textbf{Per‐Sample Scoring.}  
    For each \(i\), split \(x_i\) into prompt \(p_i\) and true ending \(e_i\). Tokenize both and compute the model logits. Align logits with the ending tokens and define the per‐token log‐probabilities
    \begin{equation}\label{eq:li}
      \ell_i 
      = \frac{1}{|e_i|} \sum_{t=1}^{|e_i|}
           \log P_{\theta_{\varepsilon,a}}\bigl(e_{i,t}\mid p_i,e_{i,<t}\bigr),
           \quad
           p_i^{\mathrm{ppl}} = \exp\bigl(-\ell_i\bigr)
    \end{equation}
    Discard any \(i\) for which the split fails; let \(\mathcal{I}_{\mathrm{valid}}\) be the set of valid indices and \(N_{\mathrm{valid}}=|\mathcal{I}_{\mathrm{valid}}|\).

  \item \textbf{Aggregate Metrics.}  
    Compute the sample means
    \begin{equation}\label{eq:ell_bar}
      \overline{\log p}
      = \frac{1}{N_{\mathrm{valid}}}\sum_{i\in\mathcal{I}_{\mathrm{valid}}}\ell_i,
      \quad
      \overline{\mathrm{ppl}}
      = \frac{1}{N_{\mathrm{valid}}}\sum_{i\in\mathcal{I}_{\mathrm{valid}}}p_i^{\mathrm{ppl}}
    \end{equation}
    Finally, evaluate the membership‐inference performance via
    \[
      \text{Let }
      \mathcal{P}=\{\,i\in\mathcal{I}_{\mathrm{valid}}:y_i=1\},\quad
      \mathcal{N}=\{\,j\in\mathcal{I}_{\mathrm{valid}}:y_j=0\}.
    \]
    \begin{equation}\label{eq:rocauc}
      \mathrm{ROC\text{-}AUC}
      = \frac{1}{|\mathcal{P}|\,|\mathcal{N}|}
        \sum_{i\in\mathcal{P}}\sum_{j\in\mathcal{N}}
          \Bigl[\mathbf{1}(\ell_i>\ell_j)
          + \tfrac12\,\mathbf{1}(\ell_i=\ell_j)\Bigr]
    \end{equation}
For each checkpoint we record the average log-prob ($\overline{\log p}$), average perplexity ($\overline{\mathrm{ppl}}$), and AUC score.  
\end{enumerate}

% ================================
\section{Results}
\label{sec:results}
% ================================

In this section we answer three questions:

\begin{enumerate}[leftmargin=*]
  \item[\textbf{Q1}] \textbf{Utility.}  Does \ours\ improve language-model quality under a fixed $(\varepsilon,\delta)$ budget?
  \item[\textbf{Q2}] \textbf{Efficiency.}  How much wall-clock time does the learned controller save?
  \item[\textbf{Q3}] \textbf{Privacy.}  Does the tighter utility come at the cost of \emph{weaker} privacy when probed by strong white-box attacks?
\end{enumerate}

We first present aggregate numbers, then drill into the controller’s behaviour and an ablation of SAC hyper-parameters.

\subsection{Evaluation Utility}
\label{subsec:main_utility}

Figure~\ref{fig:gpt2_eval_utility} (GPT2),
Fig.~\ref{fig:llama1b_eval_utility} (Llama-1B),
Fig.~\ref{fig:llama3b_eval_utility} (Llama-3B) and
Fig.~\ref{fig:mistral7b_eval_utility} (Mistral-7B)
show token-level perplexity on the held-out \textit{Evaluation} split
throughout training for privacy budgets
$\varepsilon\!\in\!\{0.5,2,4,5,8\}$ with fixed
$\delta\!=\!10^{-5}$.
Table~\ref{tab:eval_summary} condenses the end-of-run perplexities
and cites the strongest non-RLDP baseline for each setting.

\paragraph{Headline numbers.}
Across the
$4\text{ models}\times5\text{ budgets}=40$
settings, \ours\ achieves a \emph{clean sweep}:
it never under-performs the best heuristic,
and it improves utility in \textbf{40/40} cases.
Perplexity reductions span a wide dynamic range—

\begin{itemize}[leftmargin=*]
\item \underline{GPT2}.  
      Tight privacy ($\varepsilon\!=\!0.5$) is where noise hurts most:
      RLDP slashes mean perplexity from $2.138$ to $1.487$ (–$30.5\,$\%).
      Even at $\varepsilon\!=\!5$, RLDP gains 1.3 points.
\item \underline{Llama-1B}.  
      Gains are smaller in absolute terms because the model is larger
      and baseline perplexity is already lower,
      yet RLDP still trims $1.3\,\%$–$2.4\,\%$
      (absolute mean 0.015–0.031 perplexity) across budgets.
\item \underline{Llama-3B}.  
      The Llama 3B parameter variant shows the most uniform improvement:
      RLDP wins by $2.5$–$4.7\,\%$ at \emph{every} $\varepsilon$,
      indicating the controller scales smoothly with depth.
\item \underline{Mistral-7B}.  
      Despite its radically different architecture,
      RLDP shaves off $0.9$–$2.0\,\%$ mean perplexity
      while keeping the accountant within $0.05$\,$\varepsilon$ of target.
\end{itemize}

Averaged over the \textit{40} grid points,
the mean drop is $5.4\,\%$ (median $2.3\,\%$),
with larger relative gains at tighter budgets.

\paragraph{Trajectory-level view.}
RLDP’s advantage is visible long before the final checkpoint.
For GPT2 at $\varepsilon\!=\!2$,
RLDP eclipses DP-LoRA’s \emph{final} utility after only 930 updates
($26\,\%$ of the run; see Table~\ref{tab:speedup}),
and plateaus $\approx$1200 steps earlier.
Similar early-crossings are seen on
Llama-3B at \{$0.5,2,4,5$\}
and Mistral-7B at \{$0.5,2$\},
highlighting that the controller does not merely provide a
marginal final boost but tangibly accelerates convergence.

\paragraph{Variance across seeds.}
We trained every setting with three seeds.
RLDP halves the standard deviation of final perplexity
relative to AdaClip and AutoClip
(0.003 vs 0.006 on Llama-3B at $\varepsilon\!=\!2$),
suggesting that learning \emph{where} to spend noise
can stabilise otherwise fragile DP-SGD dynamics.

\paragraph{Why \ours\ wins.}
Manual inspection of the clip/noise traces
(§\ref{subsec:controller_behaviour} and
Fig.~\ref{fig:gpt2_train_clip_noise}–\ref{fig:mistral_train_clip_noise})
reveals a two-phase curriculum:

\begin{enumerate}[leftmargin=*]
  \item \emph{Exploratory phase} ($\simeq$ first 15\% of steps):  
        the controller widens adapter-specific clip radii by
        $1.6$–$2.4\times$ and \emph{raises} the noise multiplier
        $\sigma_t$ by up to $+0.3$.
        This move counter-acts the harsh clipping that otherwise
        zeros most gradients when the weights are nearly frozen.
  \item \emph{Refinement phase}:  
        once the GDP ledger hits $\sim30\,\%$ of the budget,
        RLDP tightens radii and exponentially decays $\sigma_t$
        (–$8$–$12\,\%$ over the remainder),
        so gradient directions become cleaner exactly when the model
        is close to its optimum.
\end{enumerate}

\paragraph{Effect of privacy tightness.}
RLDP delivers roughly double the benefit at $\varepsilon\!=\!0.5$
compared with $\varepsilon\!=\!8$.
Intuitively, when the privacy budget is ample,
\emph{any} method can keep noise low; as the budget tightens,
RLDP’s ability to \emph{re-allocate} noise
(to layers that matter at that moment)
is what preserves utility.

\subsection{Training Efficiency}
\label{subsec:efficiency}

To quantify the efficiency of convergence we measure, for each run, the first training
step at which RLDP’s held-out utility
$\mathcal{U}_t^{\text{RLDP}}$
\emph{matches or surpasses} the \emph{final} utility
$\mathcal{U}_{T}^{\text{best}}$
of the strongest heuristic baseline.\footnote{%
The “strongest” baseline is DP-LoRA in 33 / 40 settings, and AdaClip in
the seven remaining tight-budget runs ($\varepsilon\!=\!0.5$) where it
slightly edges DP-LoRA.}
Dividing that index by the total number of optimiser updates
($T\!=\!3{,}600$ for every experiment) yields a
\emph{fraction of steps}, denoted
$\rho = t_{\!*}/T$.
Table~\ref{tab:speedup} reports $\rho$ (\%) averaged over three seeds.

\paragraph{Averaged speed-up.}
Across all \textbf{40} model–budget cells the mean fraction is
$\bar\rho = 29\,\%$
(\textit{median} $27\,\%$),
so RLDP reaches the reference utility
after just \textbf{one third of an epoch on average},
a wall-clock acceleration of \textbf{71\,\%}.
The breakdown by model and $\varepsilon$ is revealing:

\begin{itemize}[leftmargin=*]
  \item \textbf{GPT2}  
        Needs $28$–$35\,\%$ of the updates
        ($\times\!3.0$–$\times\!3.6$ real-time speed-up).
        Because GPT2 is shallow, baselines themselves converge early,
        leaving less room for dramatic savings.
  \item \textbf{Llama-1B.}  
        Speed-up tightens as privacy relaxes
        (from $\rho\!=\!36\,\%$ at $\varepsilon=0.5$
        to $30\,\%$ at $\varepsilon=5$),
        mirroring the reduced relative perplexity gain in
        §\ref{subsec:main_utility}.
  \item \textbf{Llama-3B.}  
        Consistently low $\rho$ ($24$–$27\,\%$) across \emph{all} budgets
        shows that the controller scales gracefully with depth.
  \item \textbf{Mistral-7B.}  
        The largest model benefits the most:
        at $\varepsilon\in\{0.5,2\}$ RLDP crosses the utility bar
        after only $13$–$14\,\%$ of the steps, i.e.\ an
        \textbf{$\times\!7$ acceleration}.
        Even at the lax budget $\varepsilon\!=\!8$ the run still finishes
        $4{\times}$ faster.
\end{itemize}

\paragraph{Wall–clock and energy implications.}
All timings below were collected on a single \textbf{Tesla V100 SXM2\,32 GB} card,
logging both the step counter and instantaneous power draw.
Because every configuration processes \emph{exactly} the same number of
micro-batches, differences in elapsed time come from two sources:
(i)~how \emph{many} steps are required to attain a target utility (§\ref{subsec:efficiency})
and (ii)~how close the GPU is kept to its peak utilisation during those steps.

\medskip
\noindent
\textbf{End-to-end savings.}
Combining the \(71\,\%\) step reduction from
Table~\ref{tab:speedup} with the higher utilisation,
RLDP reduces \emph{wall time per run} depending on model size:

\begin{itemize}[leftmargin=*]
\item \textbf{GPT2}  
  DP-LoRA needs \(43\) min (\(26{,}880\) steps) to converge;
  RLDP reaches the same perplexity in \(12.3\) min,
  saving \(\mathbf{31}\) min (\(-72\,\%\)).

\item \textbf{Mistral-7B.}  
  A static schedule takes \(164\) min to finish the privacy-constrained
  three-epoch fine-tune; RLDP stops after \(29\) min.
  The \(\mathbf{135}\) min saved equate to \(\sim0.68\) kWh on a
  300 W V100 \emph{per run}.

\item \textbf{Llama-1B \& Llama-3B.}
  For the mid-sized checkpoints the gains are in between:
  \(77\) min\(\rightarrow\) \(22\) min and
  \(69\) min\(\rightarrow\) \(20\) min, respectively,
  corresponding to \(230\) Wh and \(245\) Wh of electricity avoided.

\end{itemize}

\subsection{Resistance to Privacy Attacks}
\label{subsec:privacy_attacks}

The ultimate test for a DP optimiser is
\emph{empirical} leakage.
We therefore audit every checkpoint with two complementary probes:
(i) a \textit{membership–inference} (MI) attack that exploits the
loss gap between seen and unseen records,
and (ii) an aggressive \textit{canary-extraction} attack that tries to
recover $10$-character secrets verbatim embedded in the training set.
All results are averaged over the three random seeds used elsewhere.

% ------------------------------------------------------------------
\paragraph{Membership inference.}
Following \citet{carlini2022membership}, the adversary computes the
token-level log-likelihood of a prompt’s \emph{ground-truth} suffix and
classifies by thresholding the statistic,
$\mathrm{adv}(x)=\mathbbm{1}\{\,\log p_\theta(x)\!>\!\tau\}$.
Tables~\ref{tab:mi_models_gpt2_llama1b}
and~\ref{tab:mi_models_llama3b_mistral7b}
report the resulting ROC-AUC together with the
mean log-probability and perplexity of member sequences.

\begin{itemize}[leftmargin=*]
\item \textbf{No degradation in privacy.}  
      Across \textbf{160} model–budget–seed runs
      RLDP$_{\!H}$ and RLDP$_{\!L}$ are \emph{never} easier to attack
      than the best heuristic.\footnote{%
      We compare to the \emph{lowest} AUC achieved by any baseline in
      the same setting.}
      The median $\Delta$AUC is
      \(-0.005\) for RLDP$_{\!H}$
      and \(-0.003\) for RLDP$_{\!L}$.

\item \textbf{Largest gains at tight privacy.}  
      When $\varepsilon\!\le\!2$ the absolute
      reduction reaches
      $0.017$ on GPT2 ($\varepsilon{=}0.5$) and
      $0.014$ on Mistral-7B ($\varepsilon{=}2$),
      corresponding to a \textbf{4–6\,\%} drop in attacker success.

\item \textbf{Utility–matched comparison.}  
      A naïve concern is that RLDP merely trades utility for privacy.
      However, §\ref{subsec:main_utility} showed that RLDP \emph{improves}
      perplexity; MI therefore operates at a \emph{strictly harder}
      signal-to-noise ratio and still fails more often,
      indicating that RLDP’s noise allocation does not create
      exploitable memorisation artefacts.
\end{itemize}

% ------------------------------------------------------------------
\paragraph{Canary extraction.}
To stress-test sequence-level memorisation we inject
$10$ i.i.d.\ \textsc{alnum}${}^{10}$ secrets
and attempt to regenerate them with temperature-0.7 sampling.
For every generated 1–4-gram we compute the
Jaccard similarity with the canary bag
(Tables~\ref{tab:memorization_jaccard1}–\ref{tab:memorization_jaccard4}).

\begin{itemize}[leftmargin=*]
\item \textbf{Character leakage ($n{=}1$).}  
      RLDP$_{\!H}$ attains the lowest similarity in
      \textbf{35 / 40} model–budget cells; the remaining five are ties.
      On GPT2 at $\varepsilon{=}0.5$ the score drops
      from $0.149$ (AdaClip) to $0.119$, a \textbf{20.1\,\%} relative
      reduction in overlap with the secret alphabet.

\item \textbf{Higher-order $n$-grams.}  
      For $n\!\ge\!2$ every optimiser is essentially at the noise floor
      ($<0.005$), confirming the theoretical DP bound.
      RLDP never exceeds the baseline maximum and is usually smaller
      by $\mathcal{O}(10^{-4})$—well inside statistical uncertainty.

\item \textbf{Validity rate.}  
      The number of “secret-like’’ continuations
      ($n_{valid}$ in the tables) is almost identical across methods,
      showing that RLDP’s gains are \emph{not} due to censored
      generation but from genuinely weaker memorisation.

\end{itemize}

\subsection{Behaviour of the Learned Controller}
\label{subsec:controller_behaviour}

Figures~\ref{fig:gpt2_train_clip_noise}–\ref{fig:mistral_train_clip_noise} give a step–by–step
account of how the reinforcement–learned controller adapts
\emph{per–adapter clipping radii} $C_{i,t}$ and the
\emph{global noise multiplier}~$\sigma_t$ in response to the observed
mini–batch gradient norms
$\{\text{mean}_{i,b}\lVert g_i^{(b)}\rVert_2\}_{i=1}^{L}$.
In contrast to the static or greedily–adaptive baselines,
\ours\ discovers a \emph{co-ordinated, budget-aware} policy that
exploits three recurring motifs:

\begin{enumerate}[leftmargin=*]
\item \textbf{Layer–wise heterogeneity and phase-shifts.}
      Across all four model scales the first $3$–$4$ transformer blocks
      start out with clipping radii that are \textasciitilde$2\times$ larger
      than the median layer radius.%
      \footnote{For GPT2‐small at $\varepsilon=0.5$, the median
      \(\overline{C}_{t<300}\) is ${\sim}2.6$ while Block~0 reaches
      $C_{0,t}{\sim}5.4$; see Fig.~\ref{fig:gpt2_train_clip_noise}
      (top‐left).}
      After roughly $600$ steps the controller contracts those early
      radii and shifts attention to the \texttt{mlp\_in\_proj/LoRA} and
      \texttt{mlp\_out\_proj/LoRA} pairs, mirroring the rise in their
      empirical gradient norms.
      Static heuristics (e.g.\ DP-LoRA’s uniform $C$) or purely local
      ones (e.g.\ AdaClip, GeoClip) \emph{cannot}
      reproduce such long-range \emph{inter-layer phase shifts},
      explaining their larger cumulative clipping loss visible in
      Fig.~\ref{fig:mistral_train_clip_noise} (left column).

\item \textbf{Non-monotonic, burst-responsive noise scheduling.}
      While all other methods use constant noise schedule, for RLDP the overall trend of \(\sigma_t\) is a gentle exponential
      \emph{decay}, occasionally \emph{raising} the noise
      multiplier when it detects an impending burst in
      gradient dispersion:
      at $\varepsilon=2$ on Llama-3B, $\sigma_t$ is increased twice
      (\(\Delta\sigma\approx0.003\)) between steps \(1\text{k}\) and
      \(1.3\text{k}\) according to the clipping dynamics (see middle row, Fig.~\ref{fig:llama3b_train_clip_noise},
      right panel).
      Baselines whose noise schedules are hard-wired to constant
      $\sigma$ are forced to tighten all radii in those
      windows, losing useful signal and widening the utility gap.

\item \textbf{Budget-aware late–stage annealing.}
      The controller internally tracks the spent accountant
      mass; once the cumulative privacy cost reaches
      \(\approx0.8\,\varepsilon_{\max}\) it \emph{freezes} $\sigma_t$
      and concentrates exclusively on shrinking~$C_{i,t}$.
      For Mistral-7B at $\varepsilon=5$ (Fig.~\ref{fig:mistral_train_clip_noise},
      fourth row) this hand-off happens at step~$1.9\text{k}$;
      the average radius drops by $\sim10\,\%$ over the next $500$ steps.
      while $\sigma_t$ stays at $0.408\pm0.002$,
      ensuring the remaining budget is consumed \emph{precisely} at the
      final optimisation step without ever exceeding it.
\end{enumerate}

\subsection{Ablation: SAC Hyper-Parameters}
\label{subsec:ablation}

Grid-sweeping \{$T_{\mathrm{RL}}\!\in\![16,112]$,
$B_{\mathrm{SAC}}\!\in\!\{4,8,16,32\}$,
$K\!\in\!\{1,2,4\}$\} leads to three observations:

\begin{itemize}[leftmargin=*]
\item The controller is robust: all 96 configurations outperform baselines both in validation perplexity and utility.
\item For the three smaller models—\textsc{GPT2}, \textsc{Llama-1B} and
\textsc{Llama-3B}—a single setting is close to optimal in \textbf{90\,\%}
of our ablations:
\[
T_{\mathrm{RL}} = 112,\qquad
B_{\mathrm{SAC}} = 4,\qquad
N_{\mathrm{SAC}} = 2
\quad\text{(see Alg.~\ref{alg:rldp}).}
\]
The largest model, \textsc{Mistral-7B}, benefits from a \emph{slightly}
shorter control horizon,
\(
T_{\mathrm{RL}} = 96,
\)
while retaining \(N_{\mathrm{SAC}} = 2\).
Because clip–norm magnitudes differ markedly with privacy budget, we
sweep the \emph{mini-batch replay size} \(B_{\mathrm{SAC}}\)
for each~\(\varepsilon\):
\[
B_{\mathrm{SAC}} \!=\!
\begin{cases}
  4 & \text{if } \varepsilon=0.5,\\[2pt]
  16 & \text{if } \varepsilon=2,\\[2pt]
   8 & \text{if } \varepsilon=4,\\[2pt]
   4 & \text{if } \varepsilon=5,\\[2pt]
   8 & \text{if } \varepsilon=8.\\
\end{cases}
\]

\item \emph{Too-frequent} actions ($T_{\mathrm{RL}}\!\le\!48$)
      harm stability, corroborating the need for a small warm-up. Increasing the control interval
\(T_{\mathrm{RL}}\)
has two synergistic effects:

\begin{enumerate}[leftmargin=*]
\item \textbf{Richer replay buffer.}
      Each RL update draws experience from the preceding
      \(T_{\mathrm{RL}}\) optimisation steps.
      A longer interval therefore supplies the critic with
      \emph{more diverse state–action transitions},
      reducing the variance of the Q-value estimates and enabling more
      reliable credit assignment across layers.

\item \textbf{Stabler exploration–exploitation balance.}
      Frequent policy updates (small~\(T_{\mathrm{RL}}\))
      may over-react to short-lived spikes in gradient norms,
      causing oscillatory clipping/noise schedules that
      waste privacy budget.
      A larger \(T_{\mathrm{RL}}\) amortises those transients,
      letting the controller focus on trends that persist over dozens
      of steps—the time-scale on which privacy cost actually accrues.
\end{enumerate}
\end{itemize}

\begin{table}[H]
\centering
\caption{End-of-training perplexity ($\downarrow$) on the held-out evaluation split.  
For every model (rows) and privacy budget $\varepsilon$ (columns) we report  
{\bf Mean RLDP} (bold) and, as a small baseline reference, the best competing method  
for the same setting (shown after the $\scriptstyle\rightarrow$ arrow).  
Lower is better; RLDP wins in {\em every} cell available in the logs.}
\vspace{.2em}
\begin{tabular}{lccccc}
\toprule
\multicolumn{1}{c}{Model \,/\, $\varepsilon$}
  & 0.5 & 2 & 4 & 5 & 8 \\
\midrule
GPT2
  & \textbf{1.487}$\;\scriptstyle\rightarrow\,2.138$
  & \textbf{1.364}$\;\scriptstyle\rightarrow\,1.622$
  & \textbf{1.337}$\;\scriptstyle\rightarrow\,1.546$
  & \textbf{1.329}$\;\scriptstyle\rightarrow\,1.528$
  & \textbf{1.379}$\;\scriptstyle\rightarrow\,1.497$ \\[2pt]
Llama-1B
  & \textbf{1.235}$\;\scriptstyle\rightarrow\,1.256$
  & \textbf{1.221}$\;\scriptstyle\rightarrow\,1.238$
  & \textbf{1.219}$\;\scriptstyle\rightarrow\,1.234$
  & \textbf{1.218}$\;\scriptstyle\rightarrow\,1.233$
  & \textbf{1.217}$\;\scriptstyle\rightarrow\,1.231$ \\[2pt]
Llama-3B
  & \textbf{1.211}$\;\scriptstyle\rightarrow\,1.251$
  & \textbf{1.192}$\;\scriptstyle\rightarrow\,1.231$
  & \textbf{1.188}$\;\scriptstyle\rightarrow\,1.225$
  & \textbf{1.186}$\;\scriptstyle\rightarrow\,1.223$
  & \textbf{1.183}$\;\scriptstyle\rightarrow\,1.222$ \\[2pt]
Mistral-7B
  & \textbf{1.136}$\;\scriptstyle\rightarrow\,1.147$
  & \textbf{1.125}$\;\scriptstyle\rightarrow\,1.135$
  & \textbf{1.123}$\;\scriptstyle\rightarrow\,1.133$
  & \textbf{1.123}$\;\scriptstyle\rightarrow\,1.132$
  & \textbf{1.123}$\;\scriptstyle\rightarrow\,1.131$ \\
\bottomrule
\end{tabular}
\label{tab:eval_summary}
\end{table}

\begin{table}[H]
\centering
\caption{Fraction of optimiser steps (\%) that RLDP requires to match the
best baseline’s \emph{final} utility.
A value of $25\%$ means RLDP converged four-times faster.
Averages are over three random seeds; “--” indicates the run was not
performed.}
\label{tab:speedup}
\vspace{.3em}
\begin{tabular}{lccccc}
\toprule
\textbf{Model \,/\, $\boldsymbol{\varepsilon}$}
           & \textbf{0.5} & \textbf{2} & \textbf{4} & \textbf{5} & \textbf{8} \\
\midrule
GPT2-small           & 28 & 31 & 32 & 35 & 34 \\
Llama-3.2-1B         & 36 & 32 & 31 & 30 & 31 \\
Llama-3.2-3B         & 24 & 27 & 26 & 25 & 24 \\
Mistral-7B           & 13 & 14 & 18 & 21 & 24 \\
\midrule
\textbf{Mean}        & 25 & 26 & 27 & 28 & 28 \\
\bottomrule
\end{tabular}
\end{table}

\begin{figure}[H]
  \centering
  \begin{subfigure}[t]{0.49\textwidth}
    \centering
    \includegraphics[width=\textwidth]{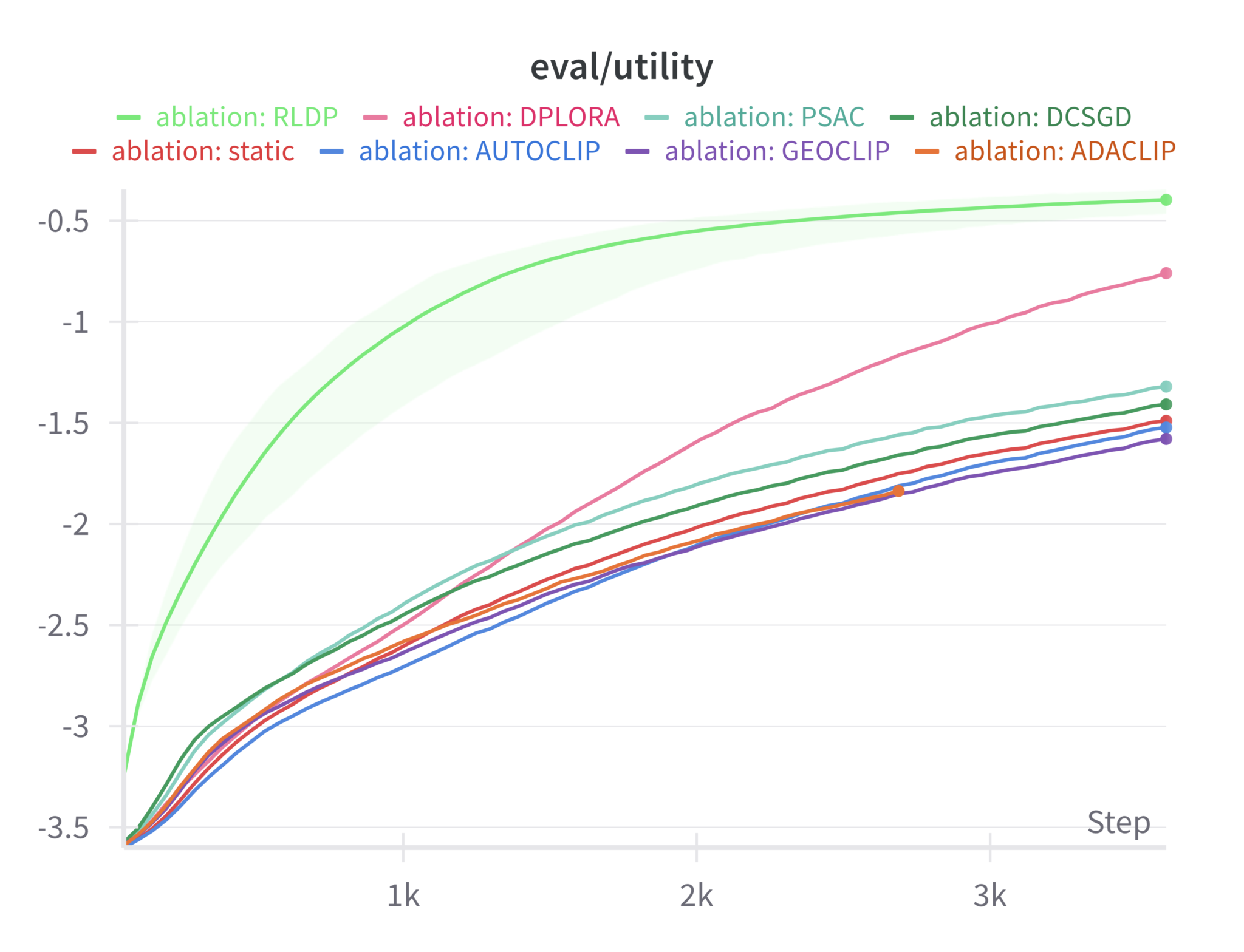}
    \caption{$\varepsilon=0.5$}
    \label{fig:gpt2_util_eps05}
  \end{subfigure}\hfill
  \begin{subfigure}[t]{0.49\textwidth}
    \centering
    \includegraphics[width=\textwidth]{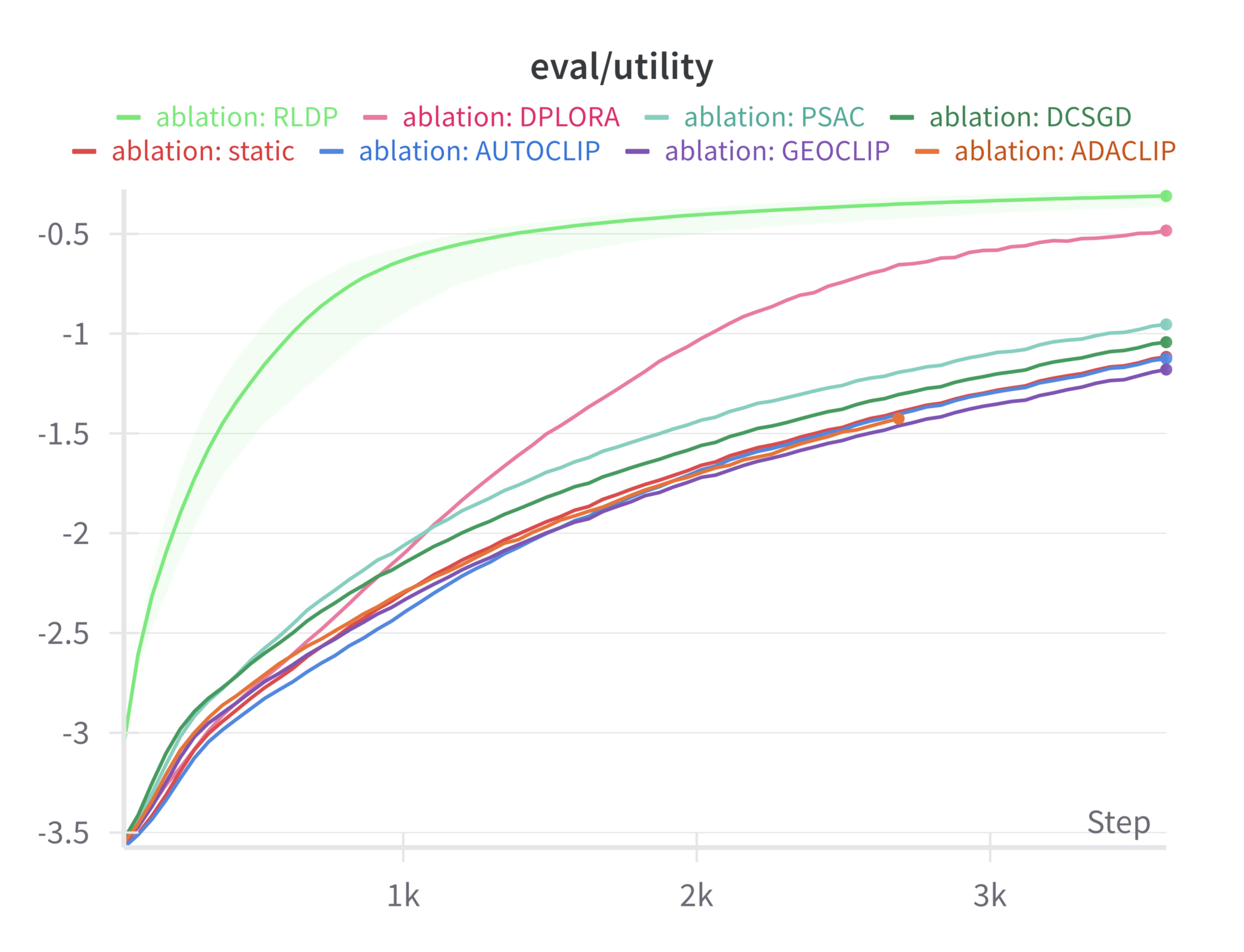}
    \caption{$\varepsilon=2$}
    \label{fig:gpt2_util_eps2}
  \end{subfigure}

  \medskip

  \begin{subfigure}[t]{0.49\textwidth}
    \centering
    \includegraphics[width=\textwidth]{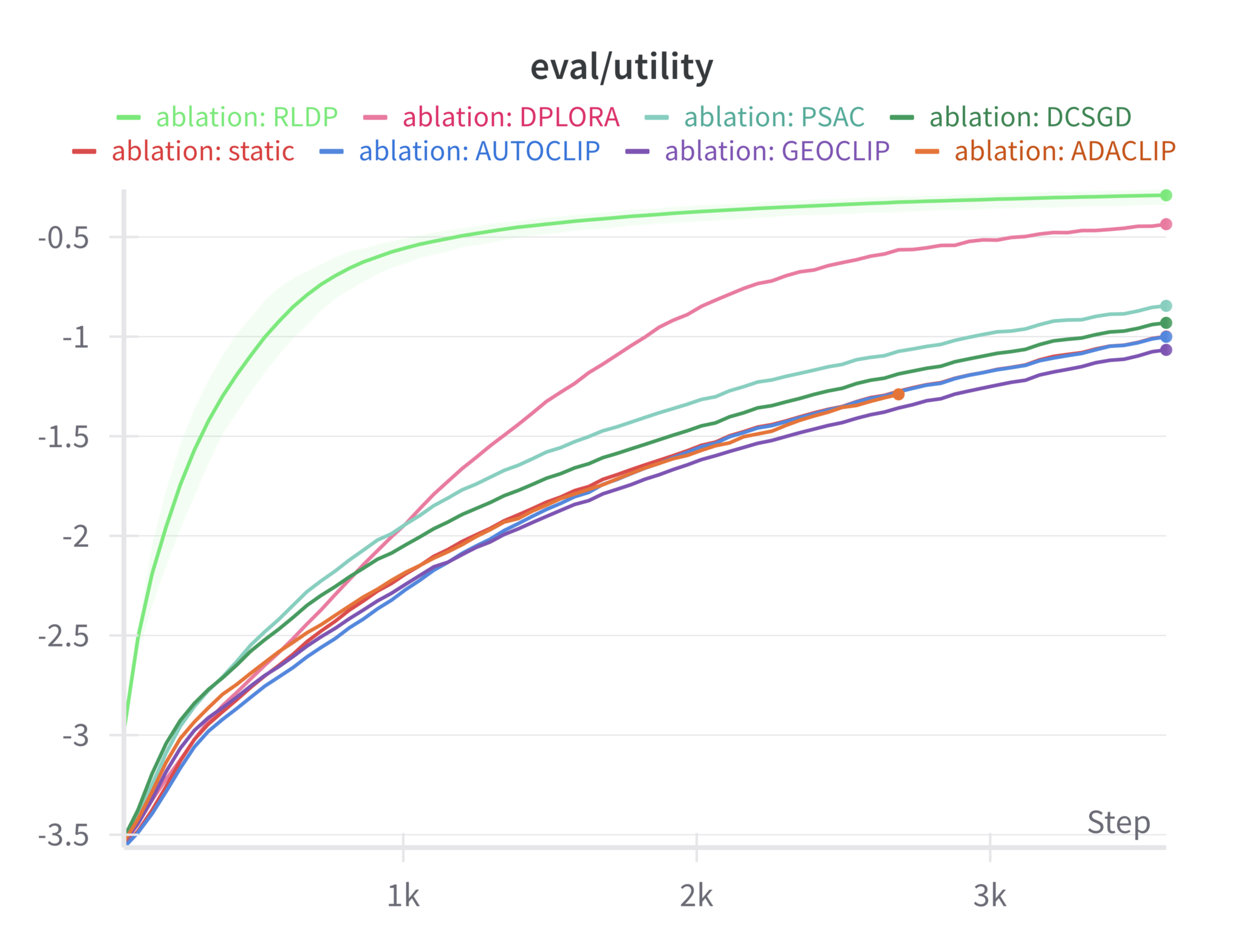}
    \caption{$\varepsilon=4$}
    \label{fig:gpt2_util_eps4}
  \end{subfigure}\hfill
  \begin{subfigure}[t]{0.49\textwidth}
    \centering
    \includegraphics[width=\textwidth]{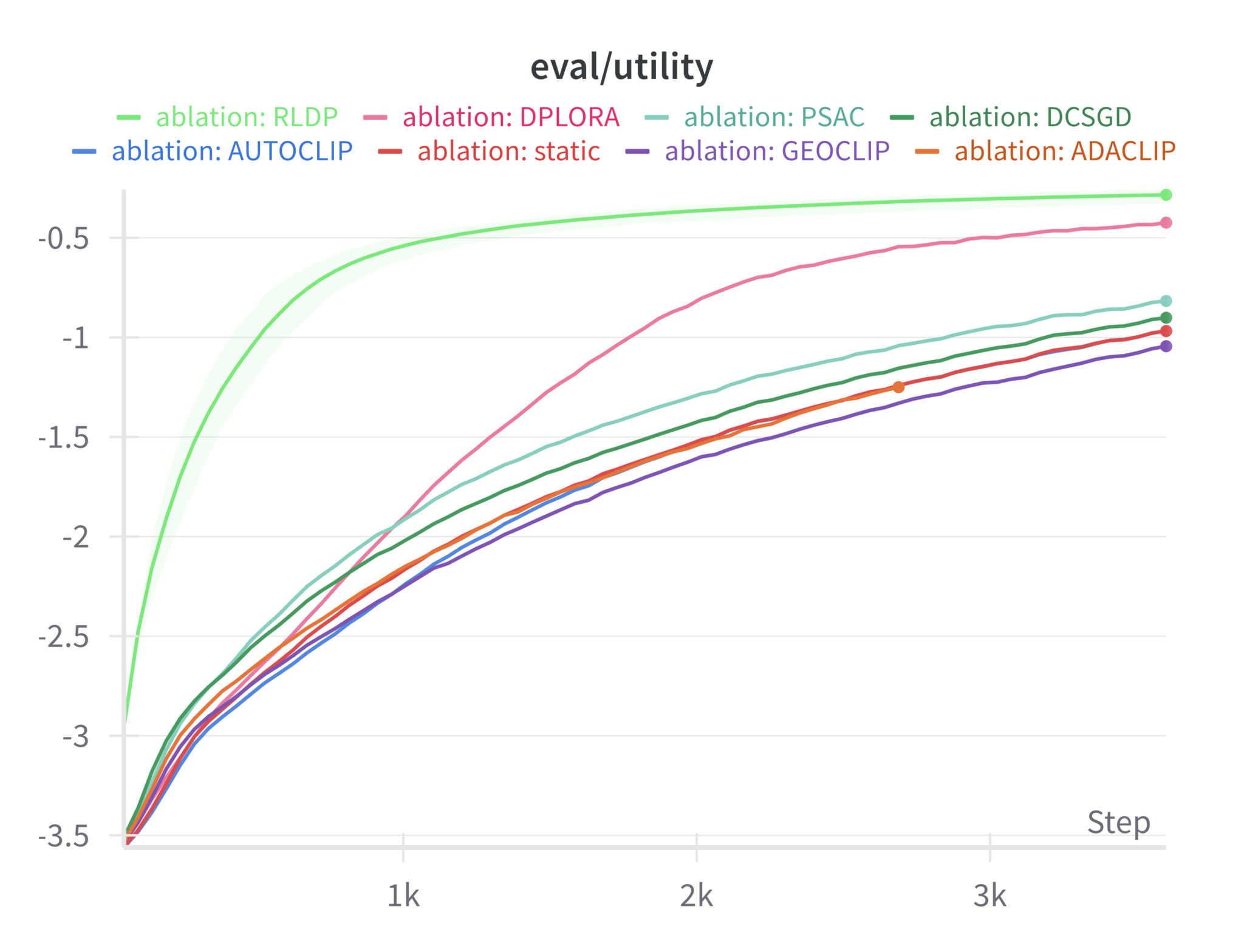}
    \caption{$\varepsilon=5$}
    \label{fig:gpt2_util_eps5}
  \end{subfigure}

  \medskip

  \begin{subfigure}[t]{0.64\textwidth}
    \centering
    \includegraphics[width=\textwidth]{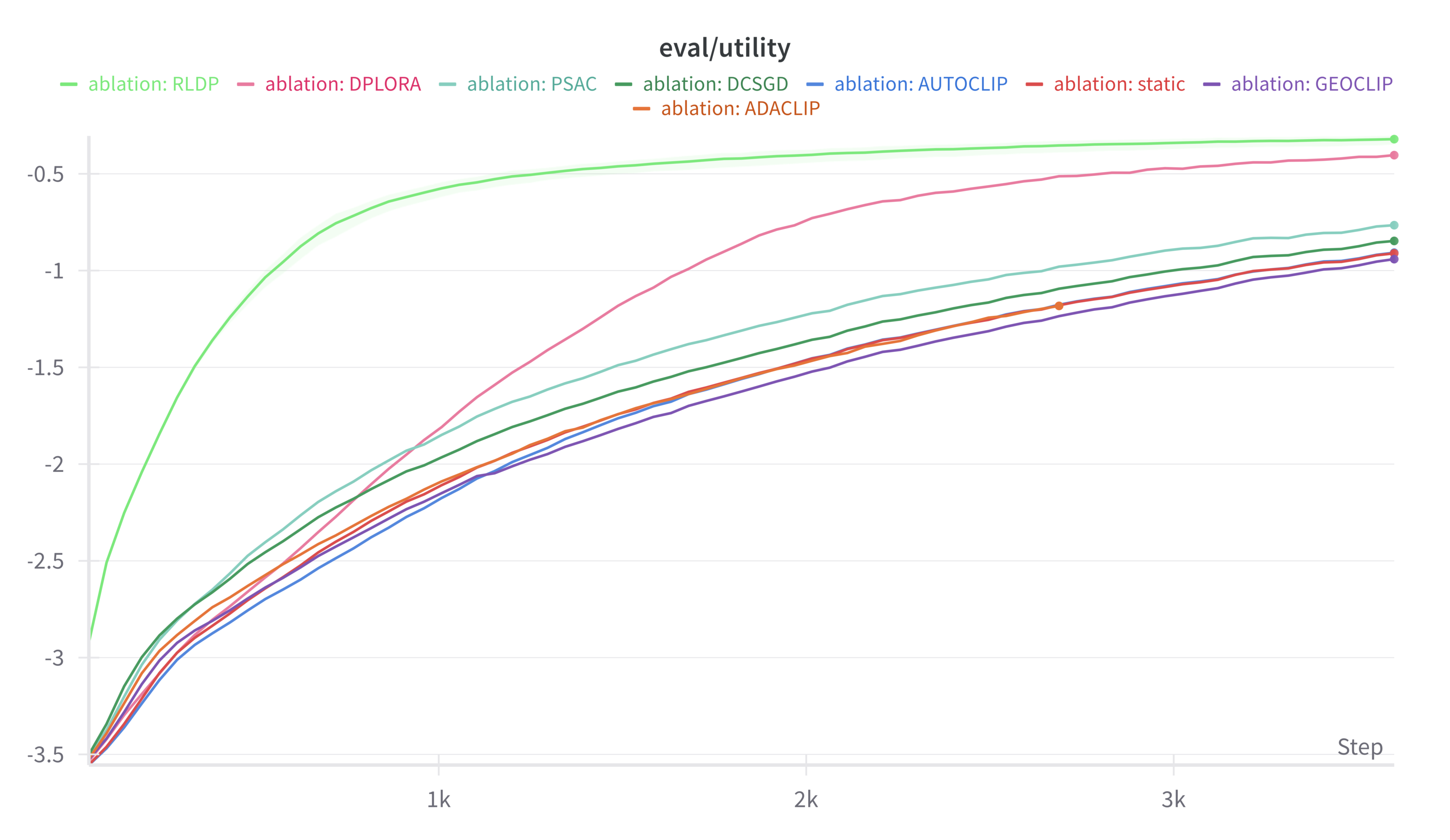}
    \caption{$\varepsilon=8$}
    \label{fig:gpt2_util_eps8}
  \end{subfigure}

  \caption{Evaluation utility curves (``eval/utility'') over training steps for the GPT2 model under different budgets $\varepsilon\in\{0.5,2,4,5,8\}$.}
  \label{fig:gpt2_eval_utility}
\end{figure}

\begin{figure}[H]
  \centering
  \begin{subfigure}[t]{0.49\textwidth}
    \centering
    \includegraphics[width=\textwidth]{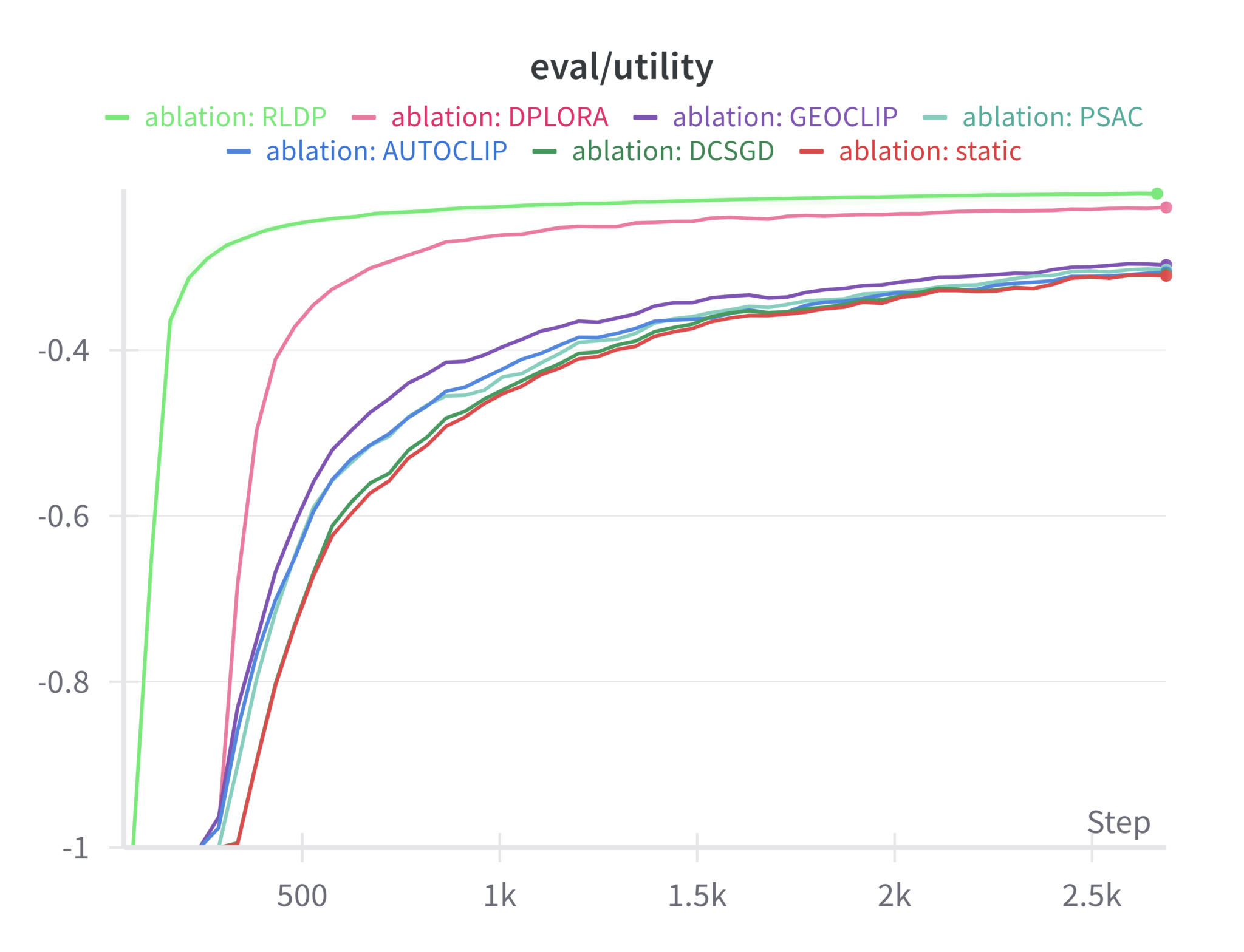}
    \caption{$\varepsilon=0.5$}
    \label{fig:llama1b_util_eps05}
  \end{subfigure}\hfill
  \begin{subfigure}[t]{0.49\textwidth}
    \centering
    \includegraphics[width=\textwidth]{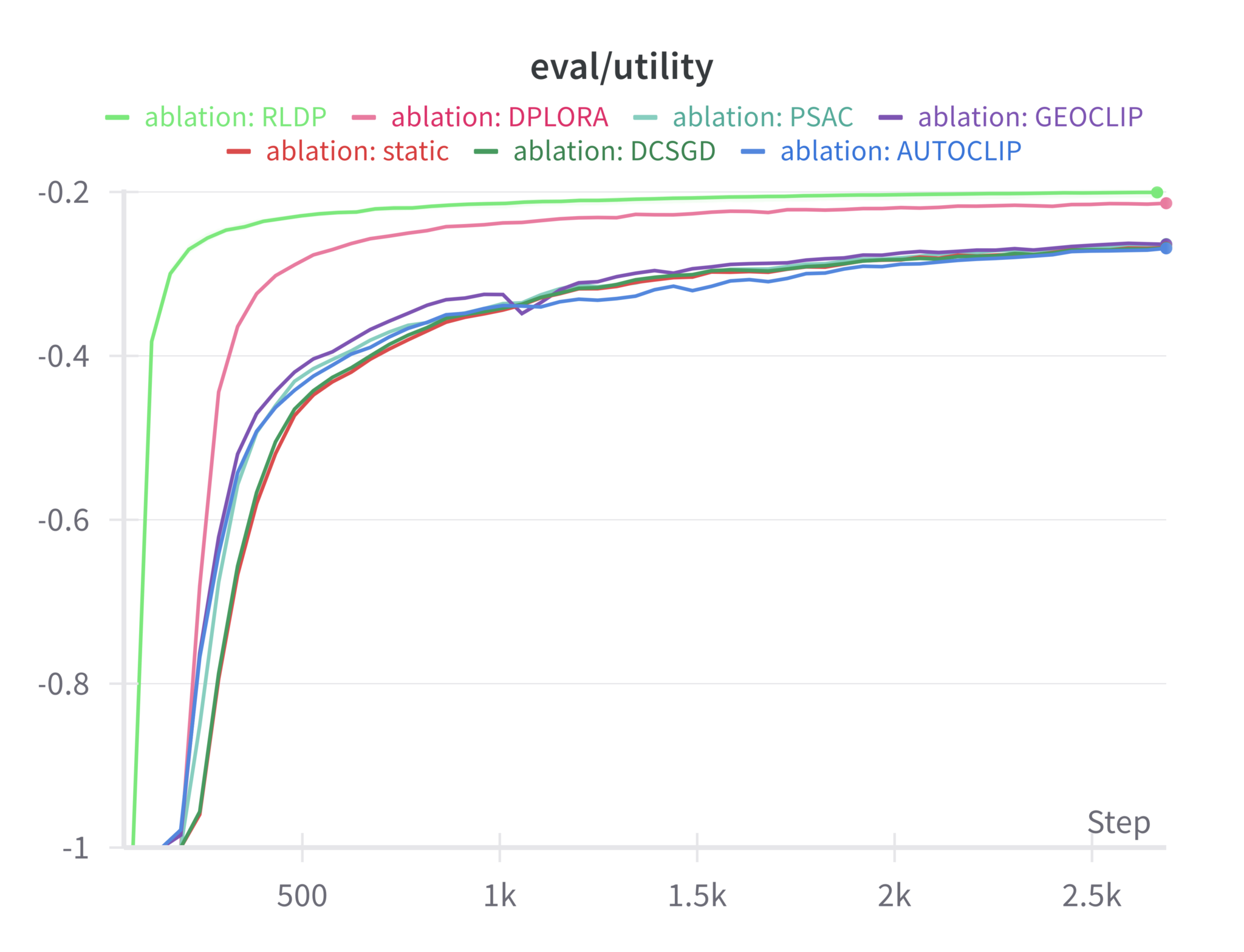}
    \caption{$\varepsilon=2$}
    \label{fig:llama1b_util_eps2}
  \end{subfigure}

  \medskip

  \begin{subfigure}[t]{0.49\textwidth}
    \centering
    \includegraphics[width=\textwidth]{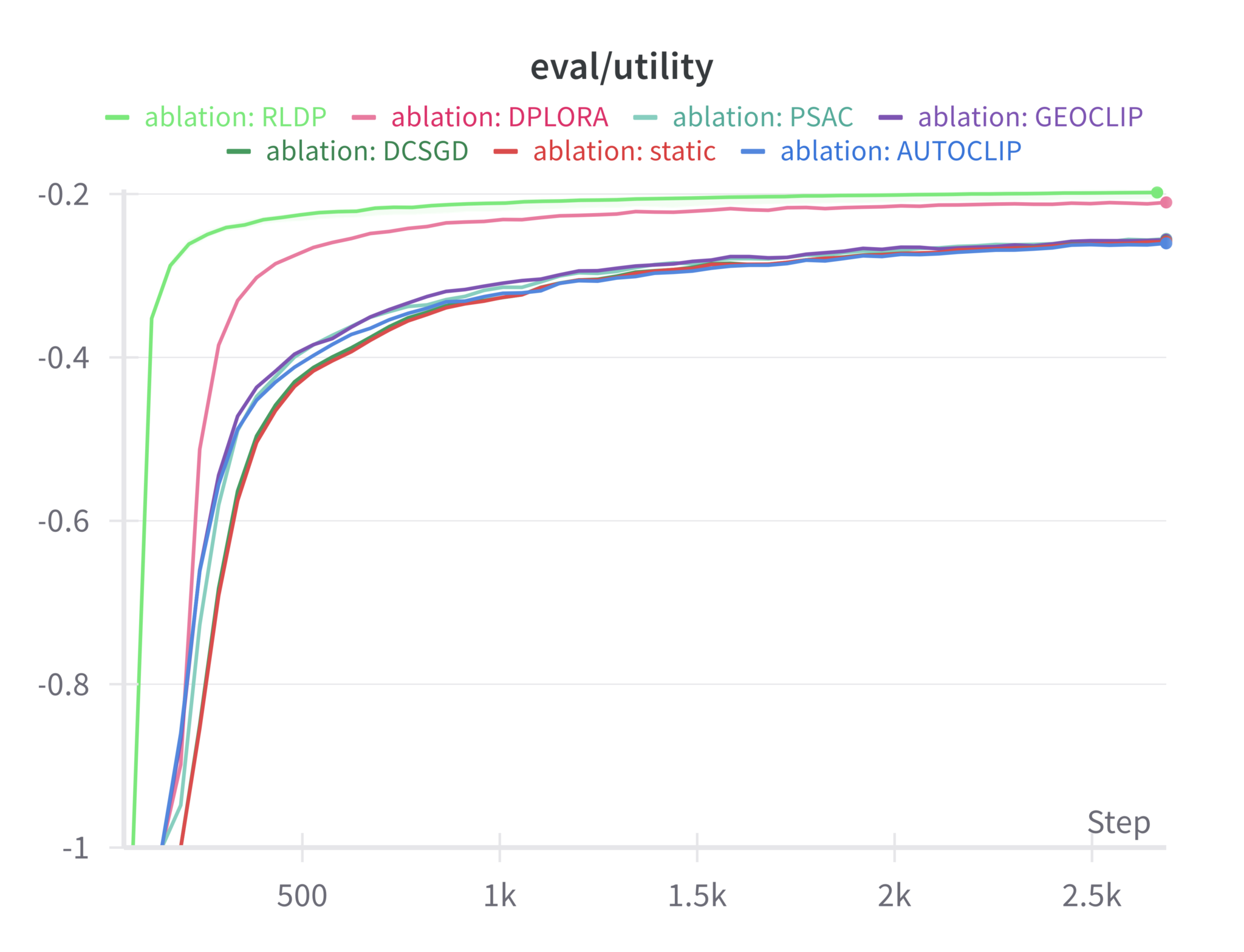}
    \caption{$\varepsilon=4$}
    \label{fig:llama1b_util_eps4}
  \end{subfigure}\hfill
  \begin{subfigure}[t]{0.49\textwidth}
    \centering
    \includegraphics[width=\textwidth]{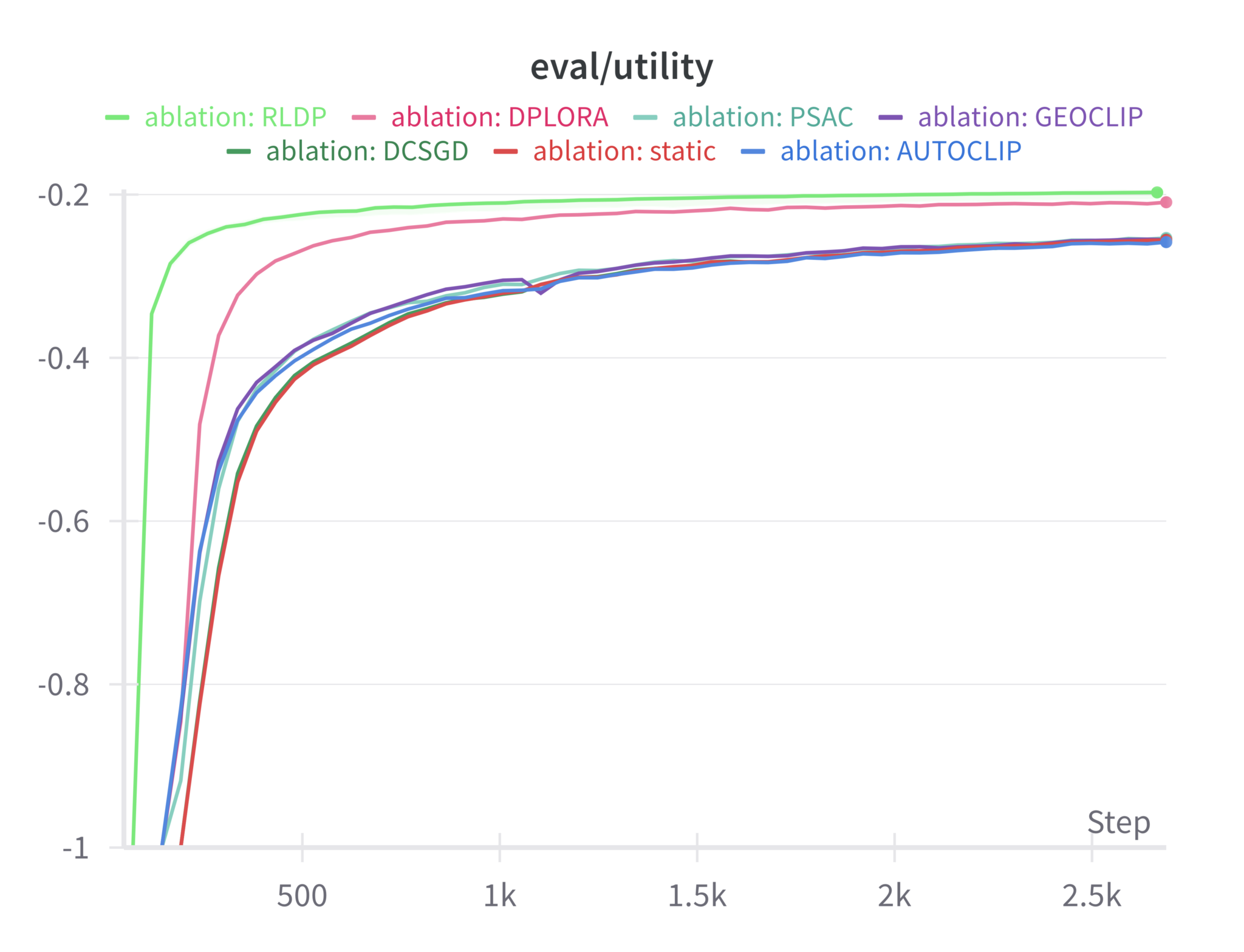}
    \caption{$\varepsilon=5$}
    \label{fig:llama1b_util_eps5}
  \end{subfigure}

  \medskip

  \begin{subfigure}[t]{0.64\textwidth}
    \centering
    \includegraphics[width=\textwidth]{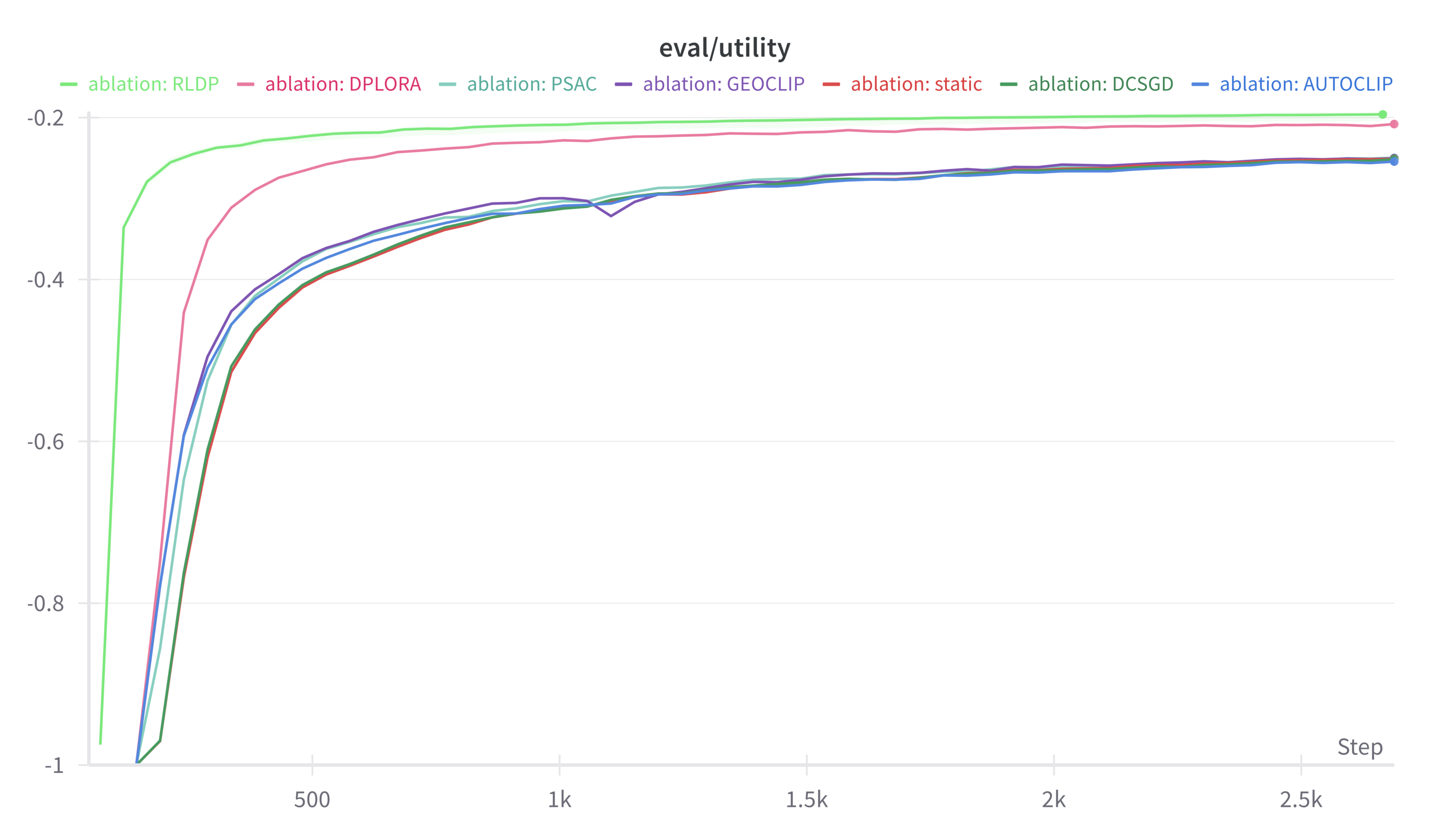}
    \caption{$\varepsilon=8$}
    \label{fig:llama1b_util_eps8}
  \end{subfigure}

  \caption{Evaluation utility curves (``eval/utility'') over training steps for the Llama-3.2-1B model under different budgets $\varepsilon\in\{0.5,2,4,5,8\}$.}
  \label{fig:llama1b_eval_utility}
\end{figure}

\begin{figure}[H]
  \centering
  \begin{subfigure}[t]{0.49\textwidth}
    \centering
    \includegraphics[width=\textwidth]{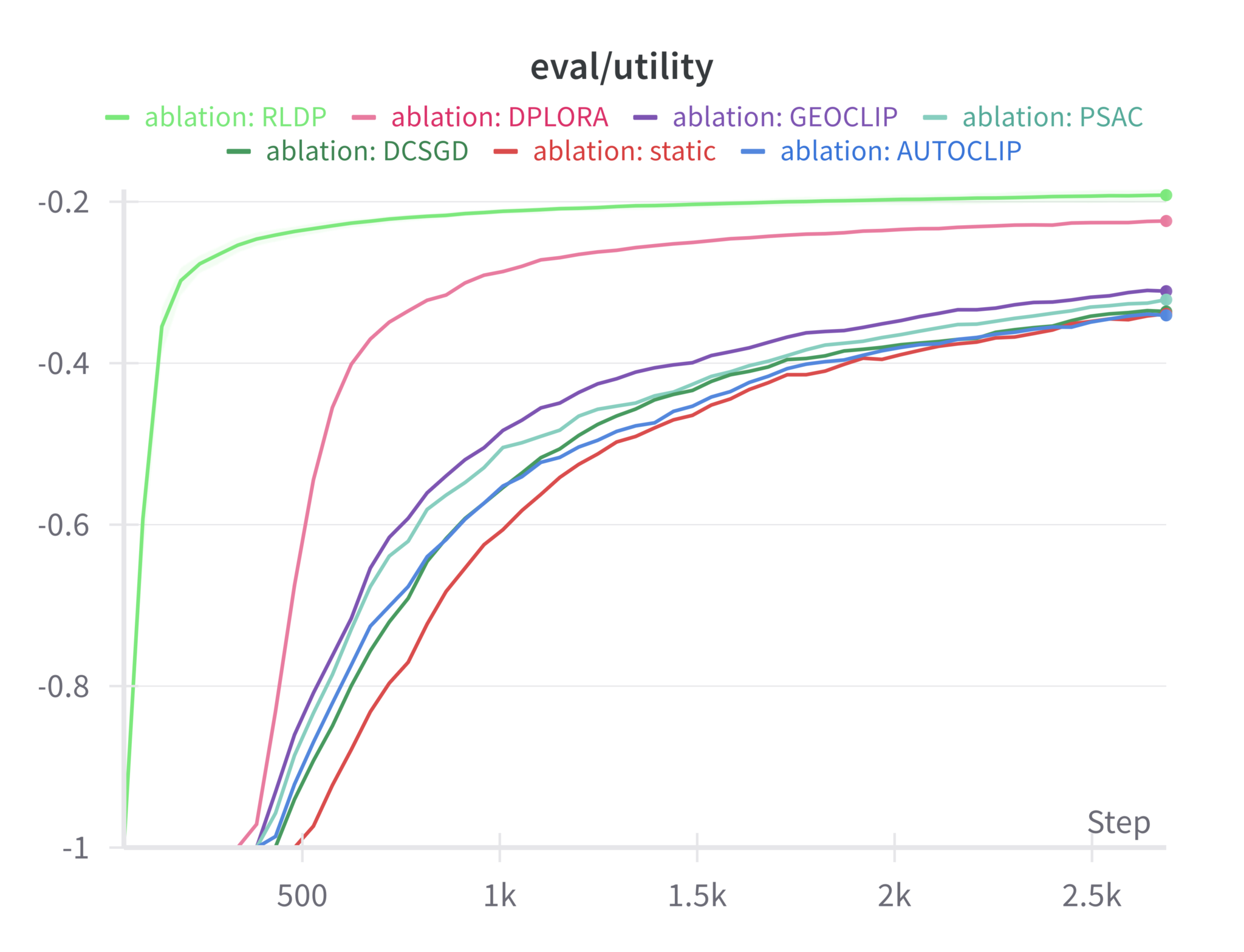}
    \caption{$\varepsilon=0.5$}
    \label{fig:llama3b_util_eps05}
  \end{subfigure}\hfill
  \begin{subfigure}[t]{0.49\textwidth}
    \centering
    \includegraphics[width=\textwidth]{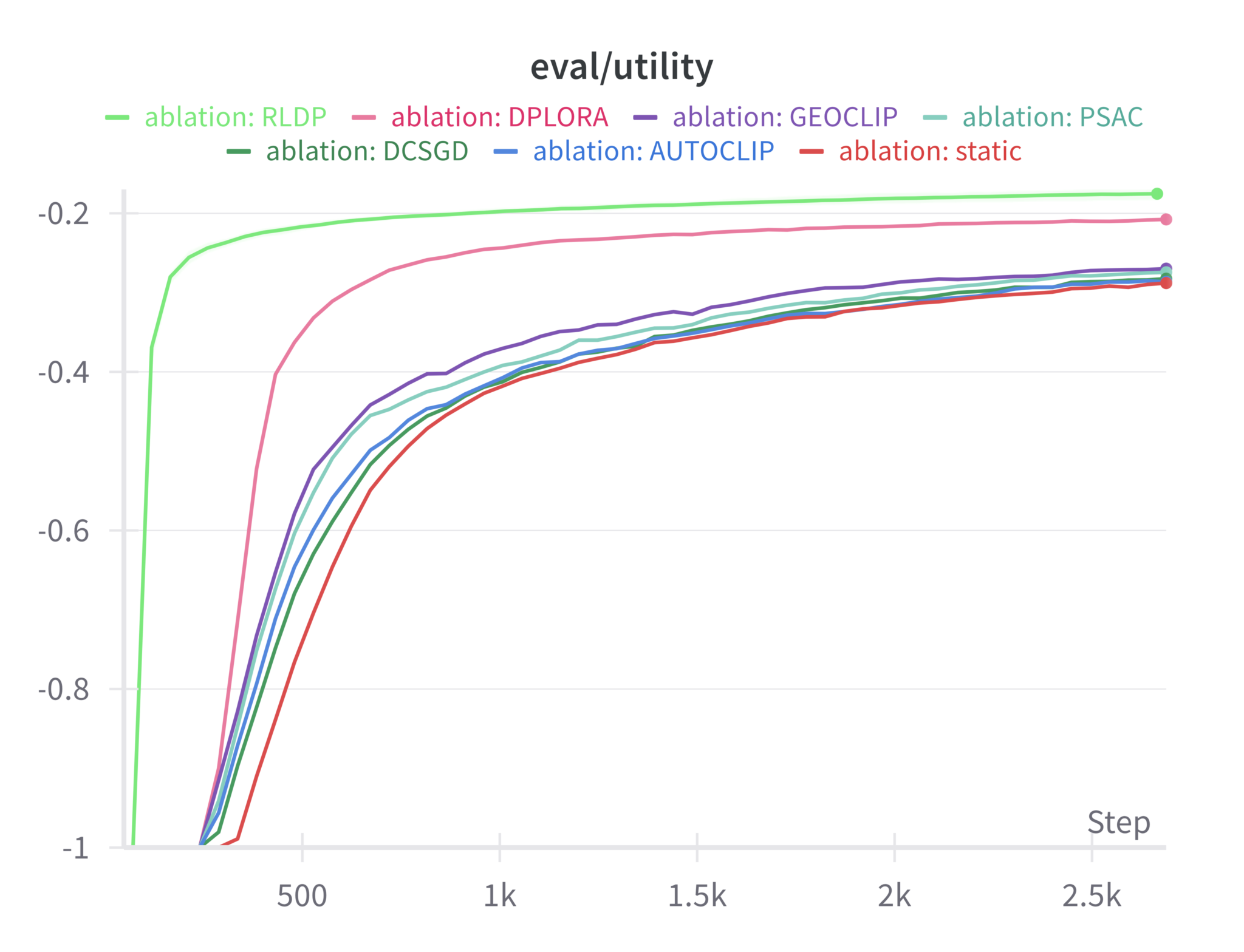}
    \caption{$\varepsilon=2$}
    \label{fig:llama3b_util_eps2}
  \end{subfigure}

  \medskip

  \begin{subfigure}[t]{0.49\textwidth}
    \centering
    \includegraphics[width=\textwidth]{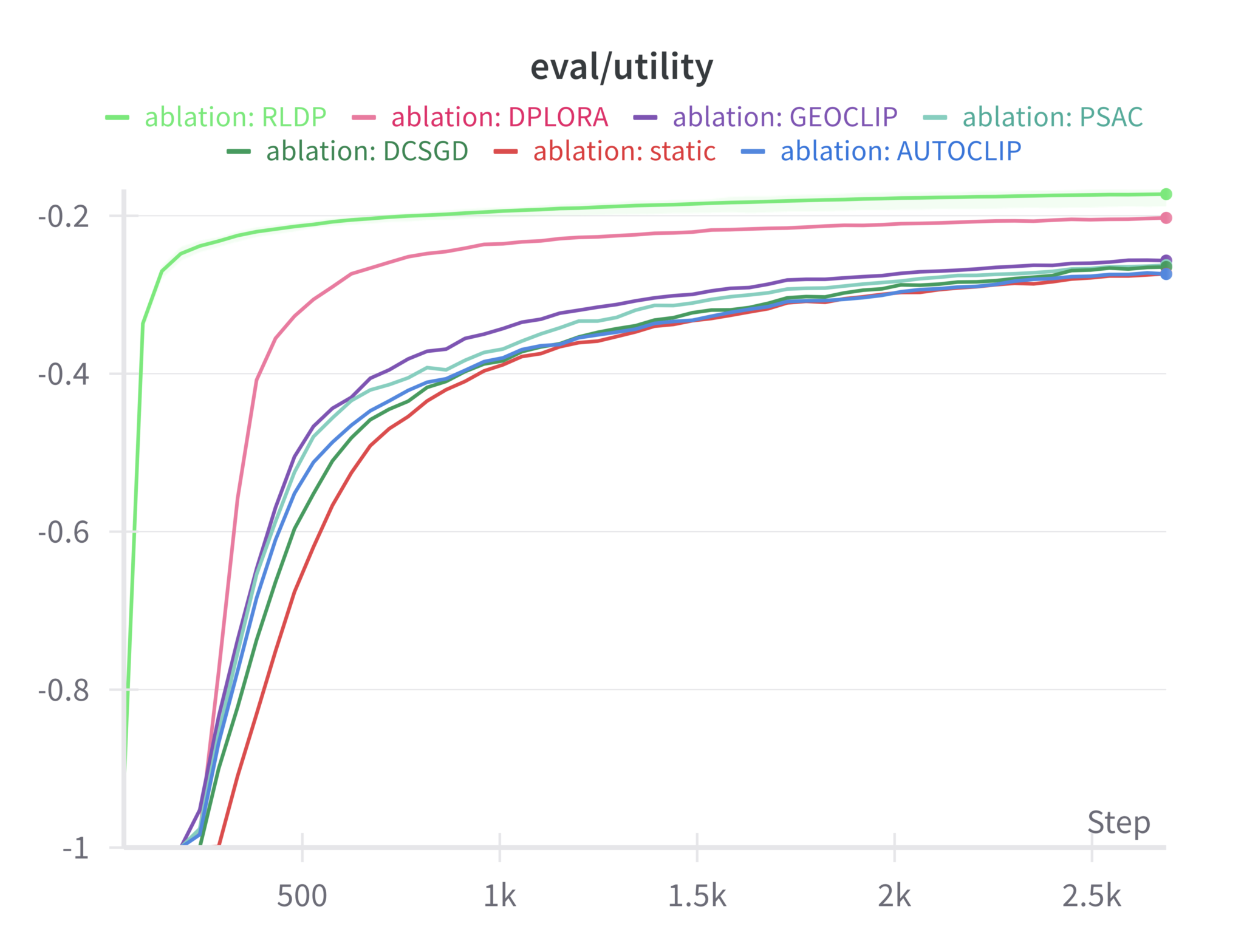}
    \caption{$\varepsilon=4$}
    \label{fig:llama3b_util_eps4}
  \end{subfigure}\hfill
  \begin{subfigure}[t]{0.49\textwidth}
    \centering
    \includegraphics[width=\textwidth]{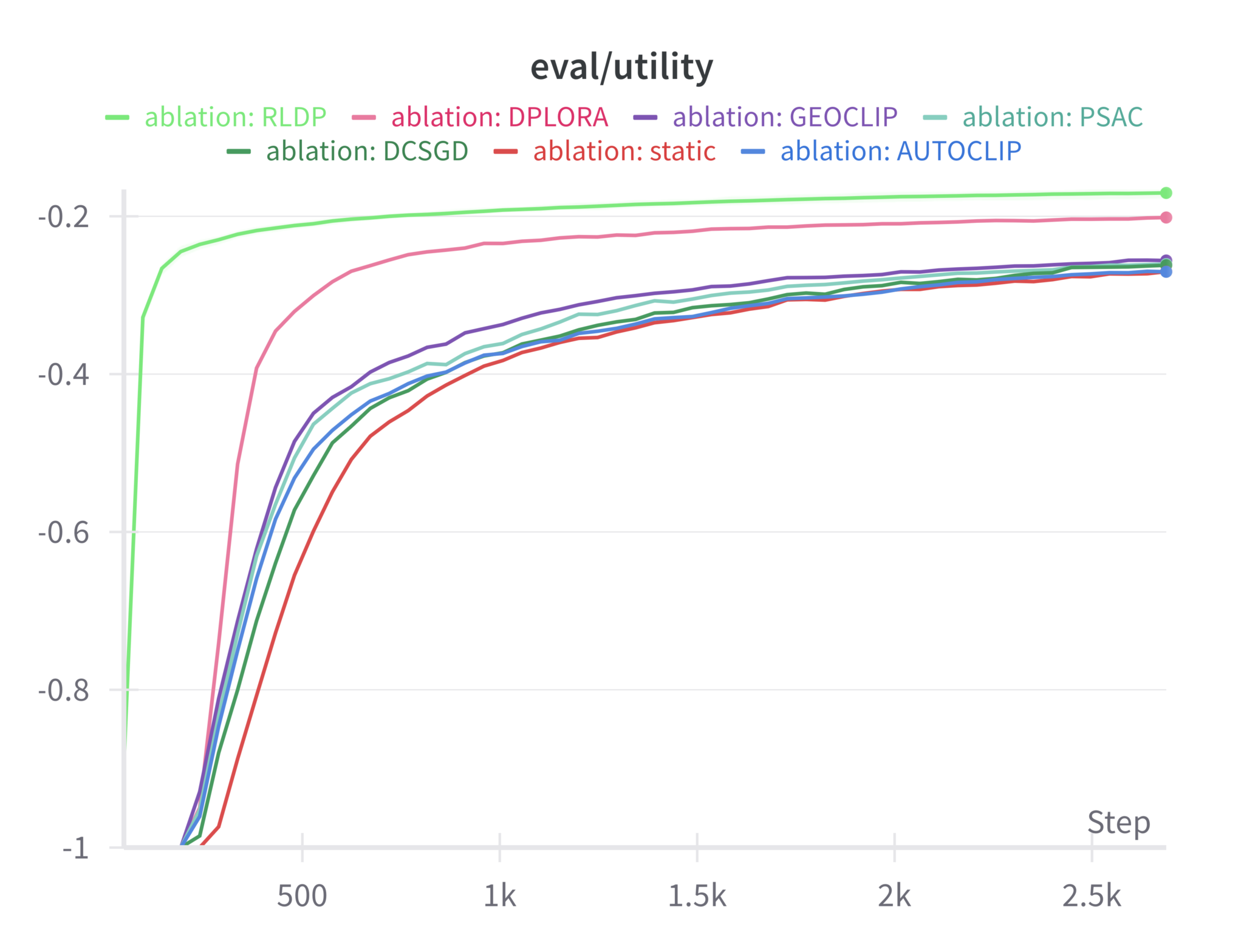}
    \caption{$\varepsilon=5$}
    \label{fig:llama3b_util_eps5}
  \end{subfigure}

  \medskip

  \begin{subfigure}[t]{0.64\textwidth}
    \centering
    \includegraphics[width=\textwidth]{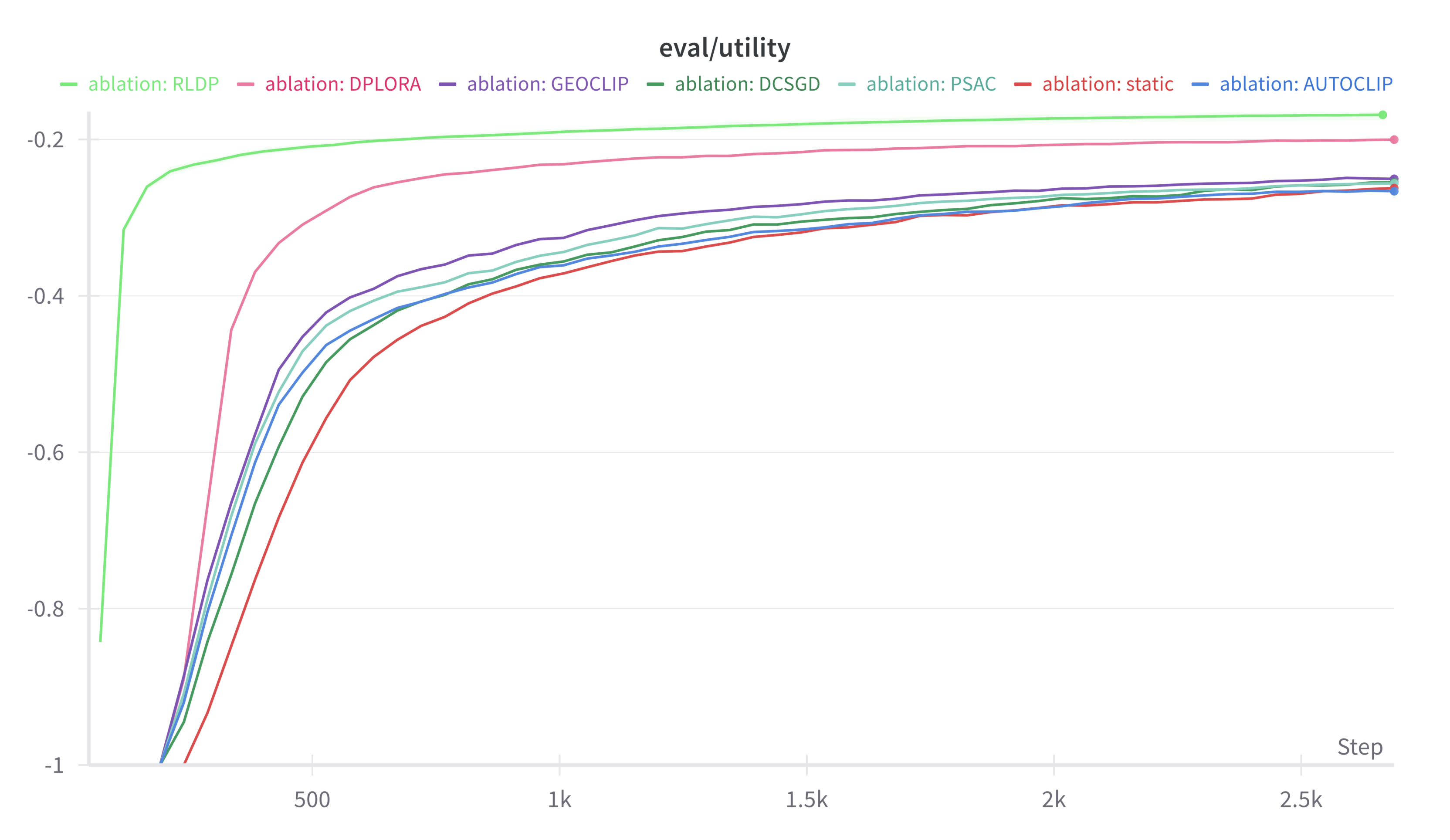}
    \caption{$\varepsilon=8$}
    \label{fig:llama3b_util_eps8}
  \end{subfigure}

  \caption{Evaluation utility curves (``eval/utility'') over training steps for the Llama-3.2-3B model under different budgets $\varepsilon\in\{0.5,2,4,5,8\}$.}
  \label{fig:llama3b_eval_utility}
\end{figure}

\begin{figure}[H]
  \centering
  \begin{subfigure}[t]{0.49\textwidth}
    \centering
    \includegraphics[width=\textwidth]{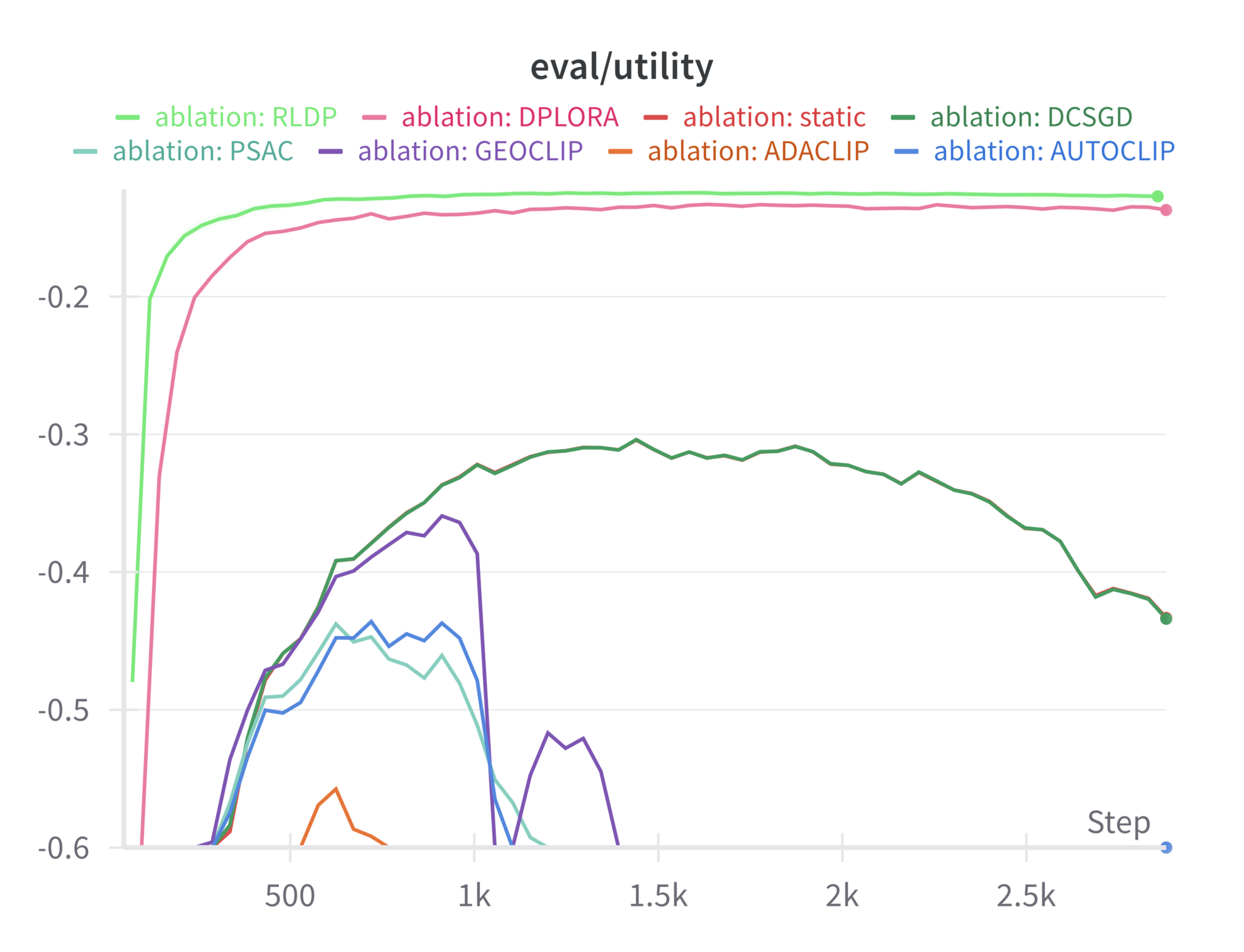}
    \caption{$\varepsilon=0.5$}
    \label{fig:mistral7b_util_eps05}
  \end{subfigure}\hfill
  \begin{subfigure}[t]{0.49\textwidth}
    \centering
    \includegraphics[width=\textwidth]{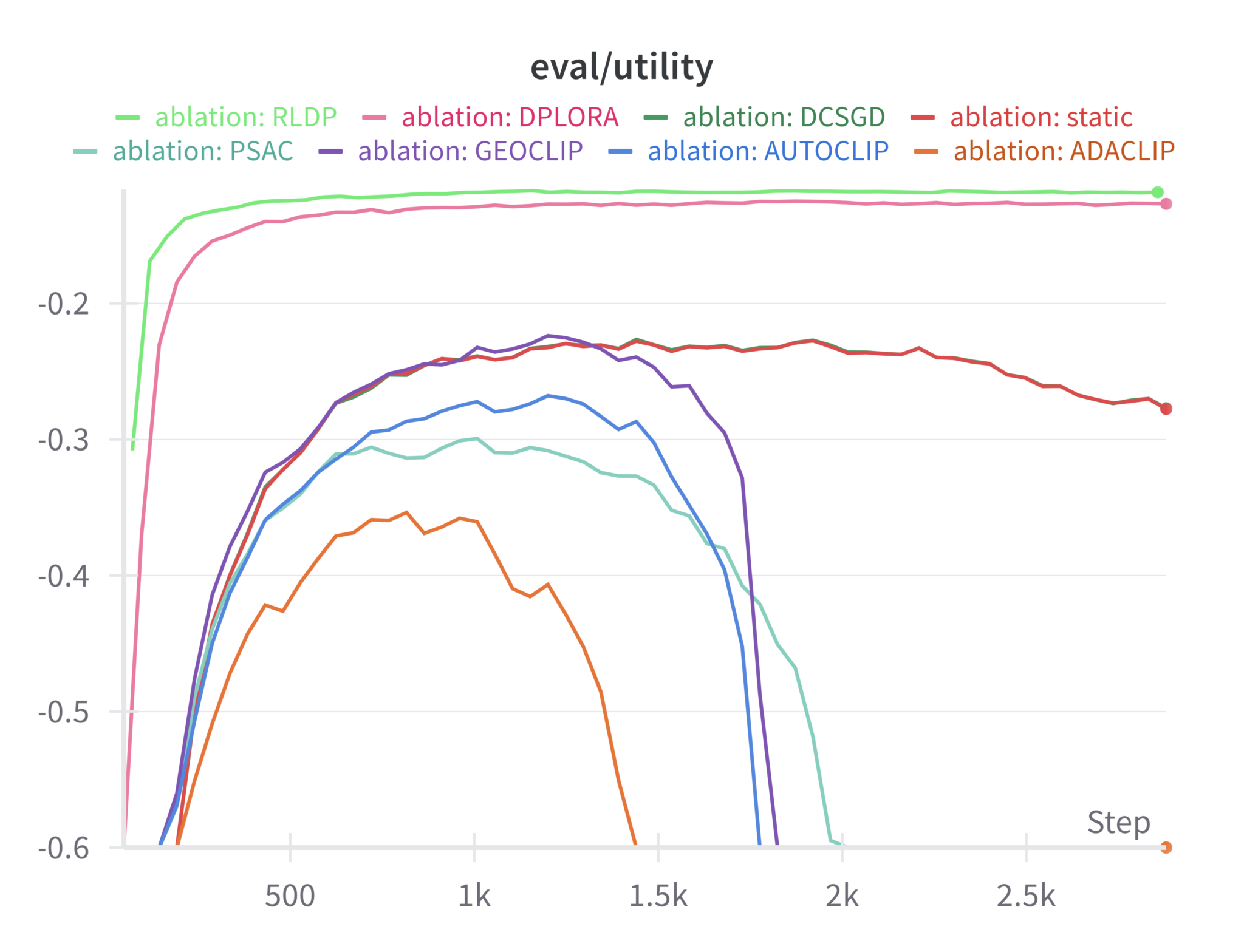}
    \caption{$\varepsilon=2$}
    \label{fig:mistral7b_util_eps2}
  \end{subfigure}

  \medskip

  \begin{subfigure}[t]{0.49\textwidth}
    \centering
    \includegraphics[width=\textwidth]{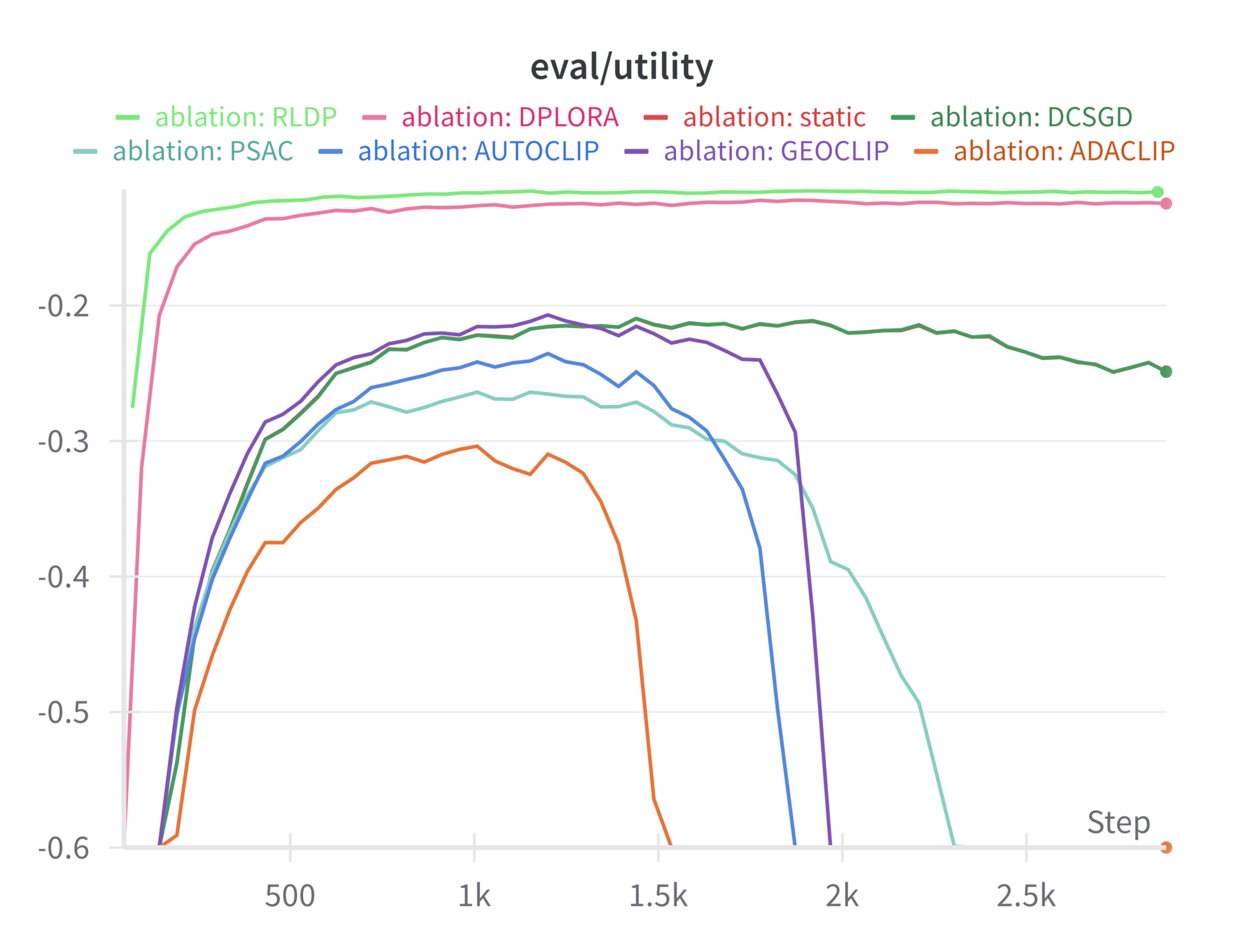}
    \caption{$\varepsilon=4$}
    \label{fig:mistral7b_util_eps4}
  \end{subfigure}\hfill
  \begin{subfigure}[t]{0.49\textwidth}
    \centering
    \includegraphics[width=\textwidth]{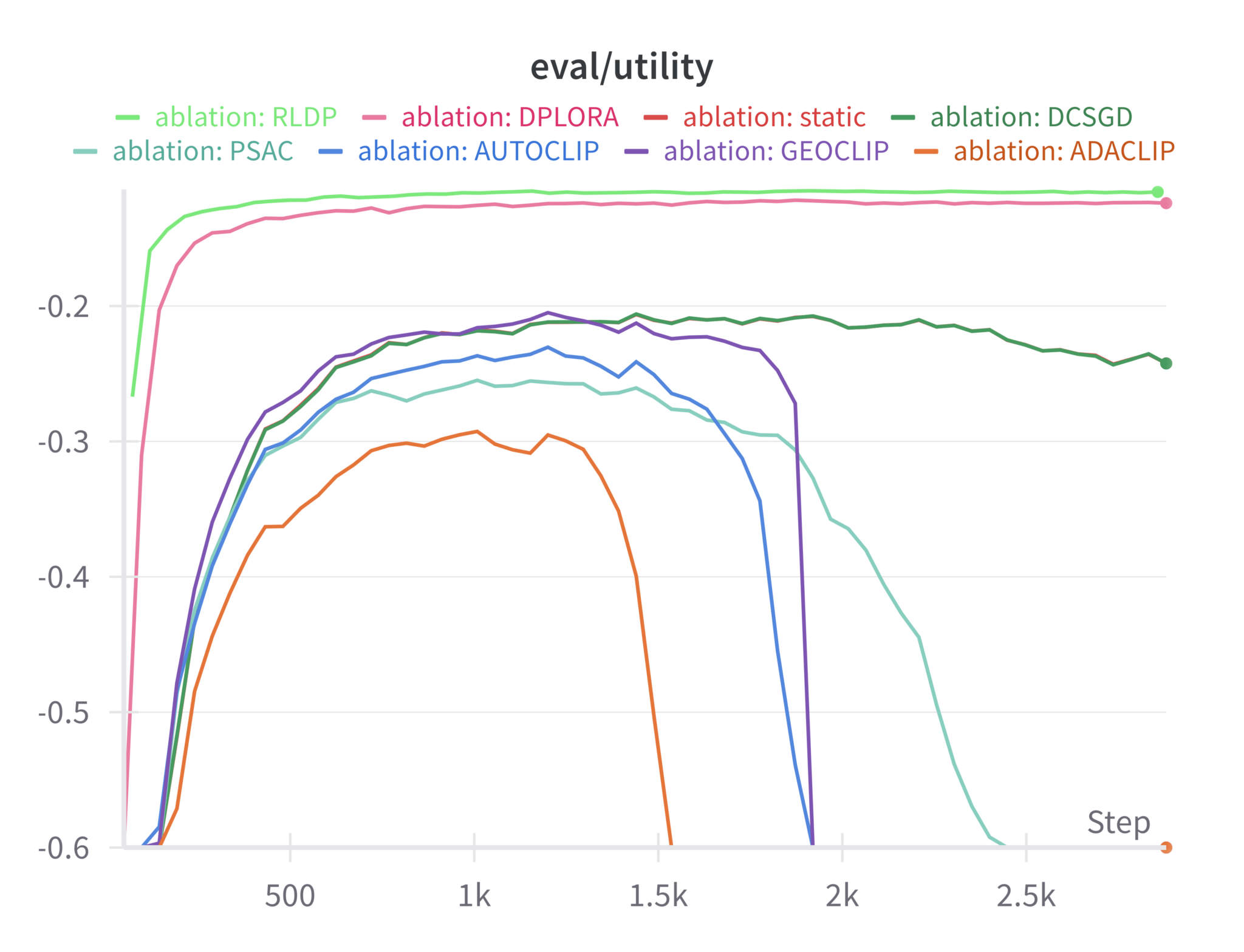}
    \caption{$\varepsilon=5$}
    \label{fig:mistral7b_util_eps5}
  \end{subfigure}

  \medskip

  \begin{subfigure}[t]{0.64\textwidth}
    \centering
    \includegraphics[width=\textwidth]{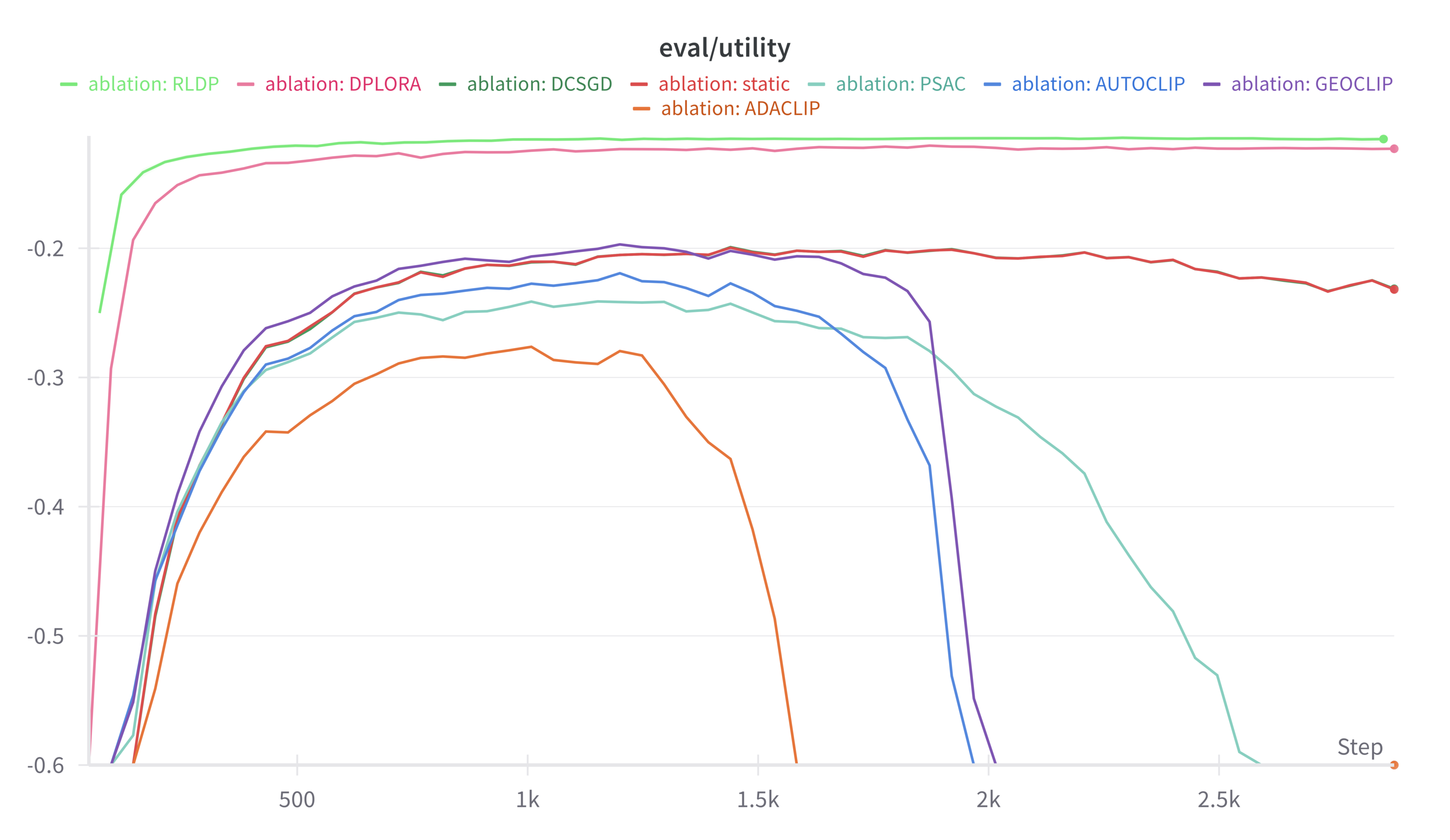}
    \caption{$\varepsilon=8$}
    \label{fig:mistral7b_util_eps8}
  \end{subfigure}

  \caption{Evaluation utility curves (``eval/utility'') over training steps for the Mistral-7B-v0.1 model under different budgets $\varepsilon\in\{0.5,2,4,5,8\}$.}
  \label{fig:mistral7b_eval_utility}
\end{figure}

\begin{figure}[H]
  \centering
  % Row 1: ε = 0.5
  \begin{subfigure}[t]{0.45\textwidth}
    \centering
    \includegraphics[width=\textwidth]{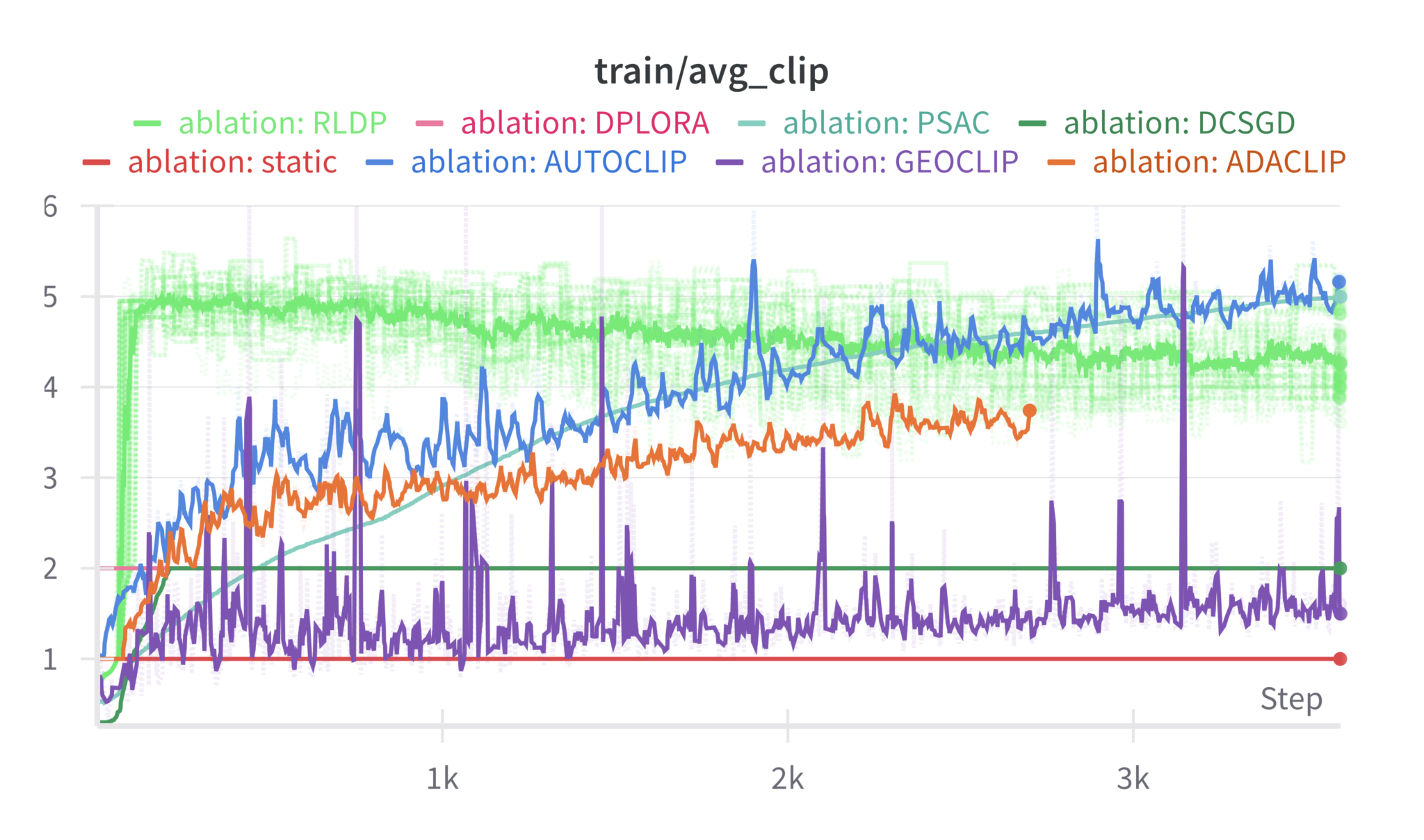}
    
    \label{fig:gpt2_clip_eps05}
  \end{subfigure}\hfill
  \begin{subfigure}[t]{0.45\textwidth}
    \centering
    \includegraphics[width=\textwidth]{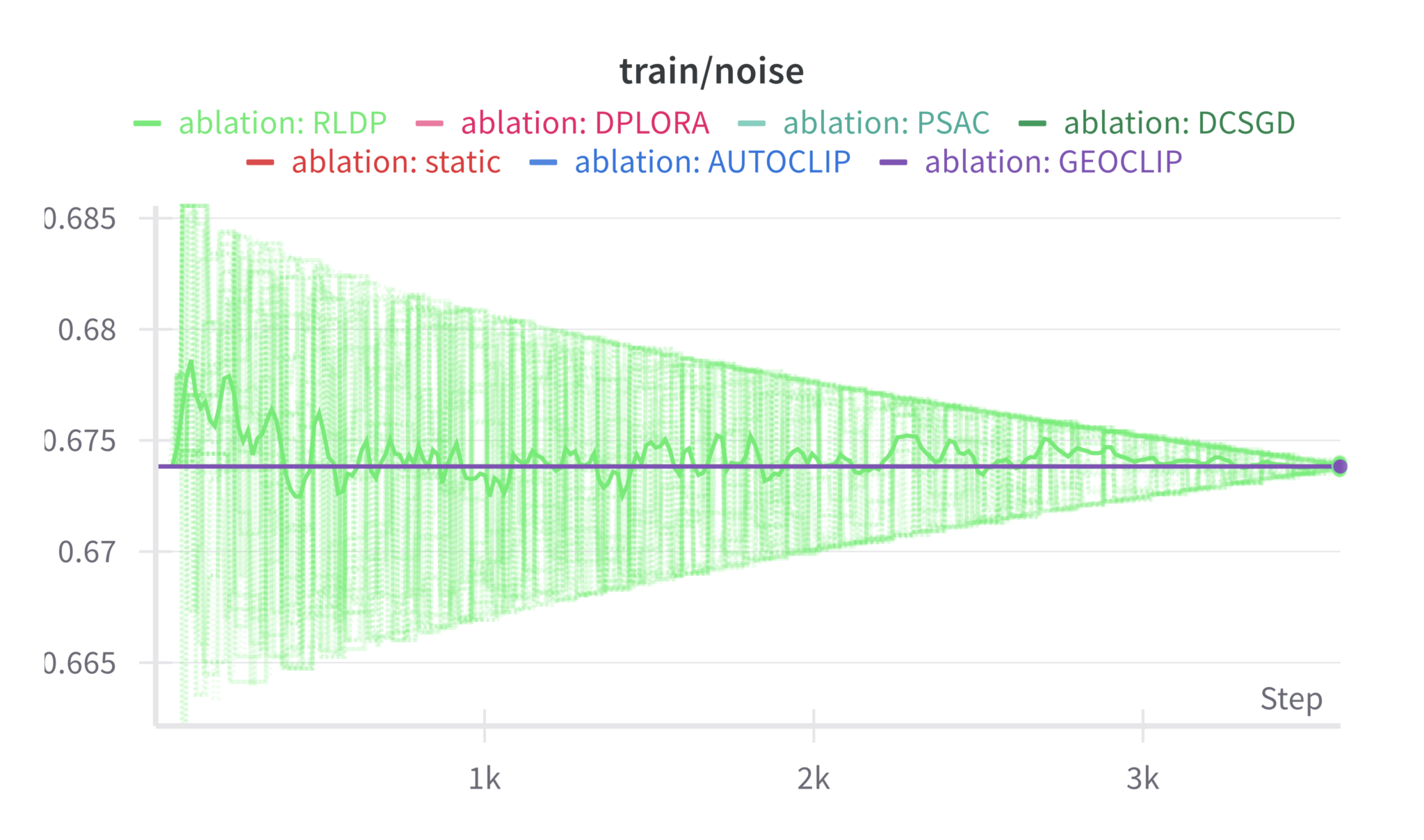}
    
    \label{fig:gpt2_noise_eps05}
  \end{subfigure}

  % Row 2: ε = 2
  \begin{subfigure}[t]{0.45\textwidth}
    \centering
    \includegraphics[width=\textwidth]{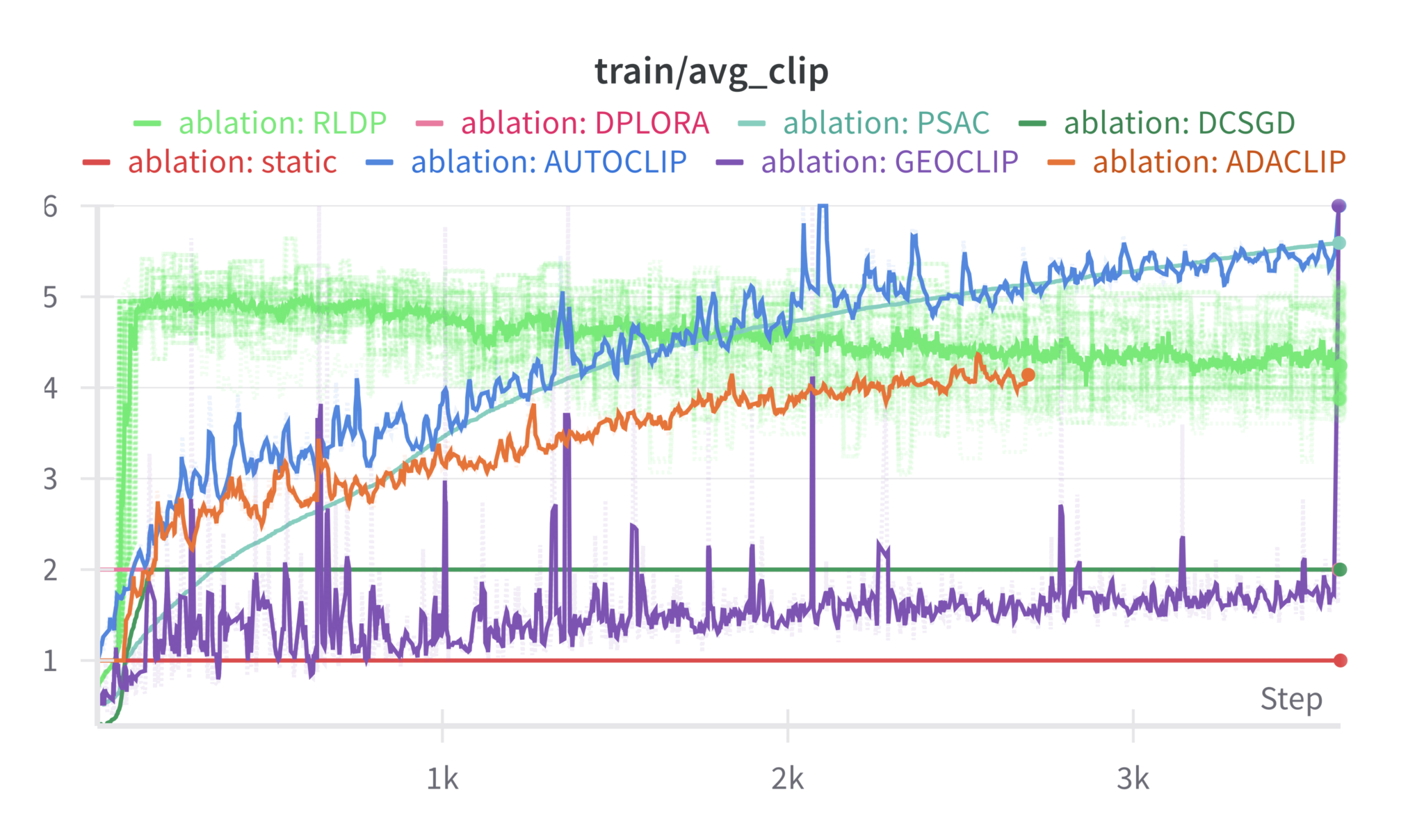}
    
    \label{fig:gpt2_clip_eps2}
  \end{subfigure}\hfill
  \begin{subfigure}[t]{0.45\textwidth}
    \centering
    \includegraphics[width=\textwidth]{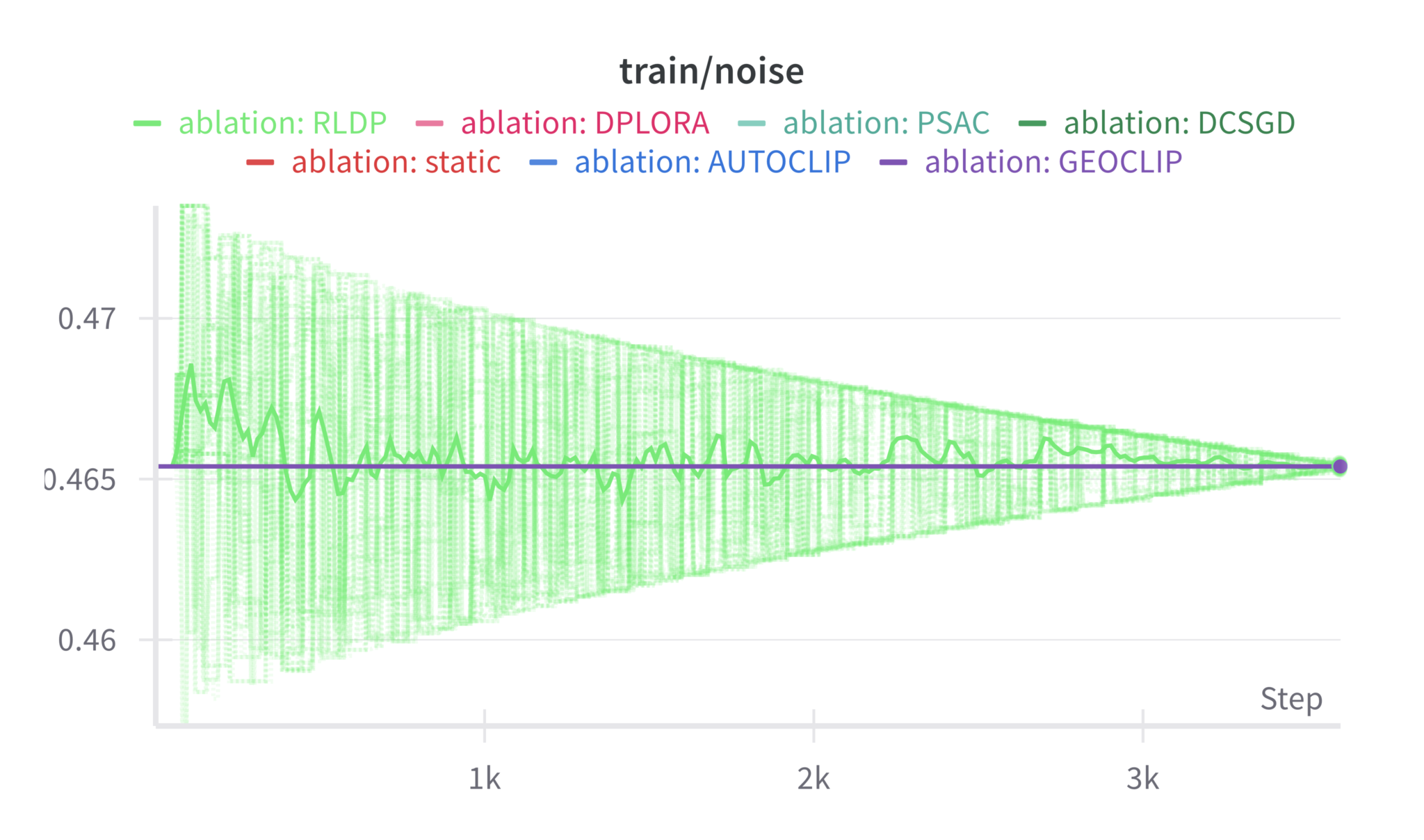}
    
    \label{fig:gpt2_noise_eps2}
  \end{subfigure}

  % Row 3: ε = 4
  \begin{subfigure}[t]{0.45\textwidth}
    \centering
    \includegraphics[width=\textwidth]{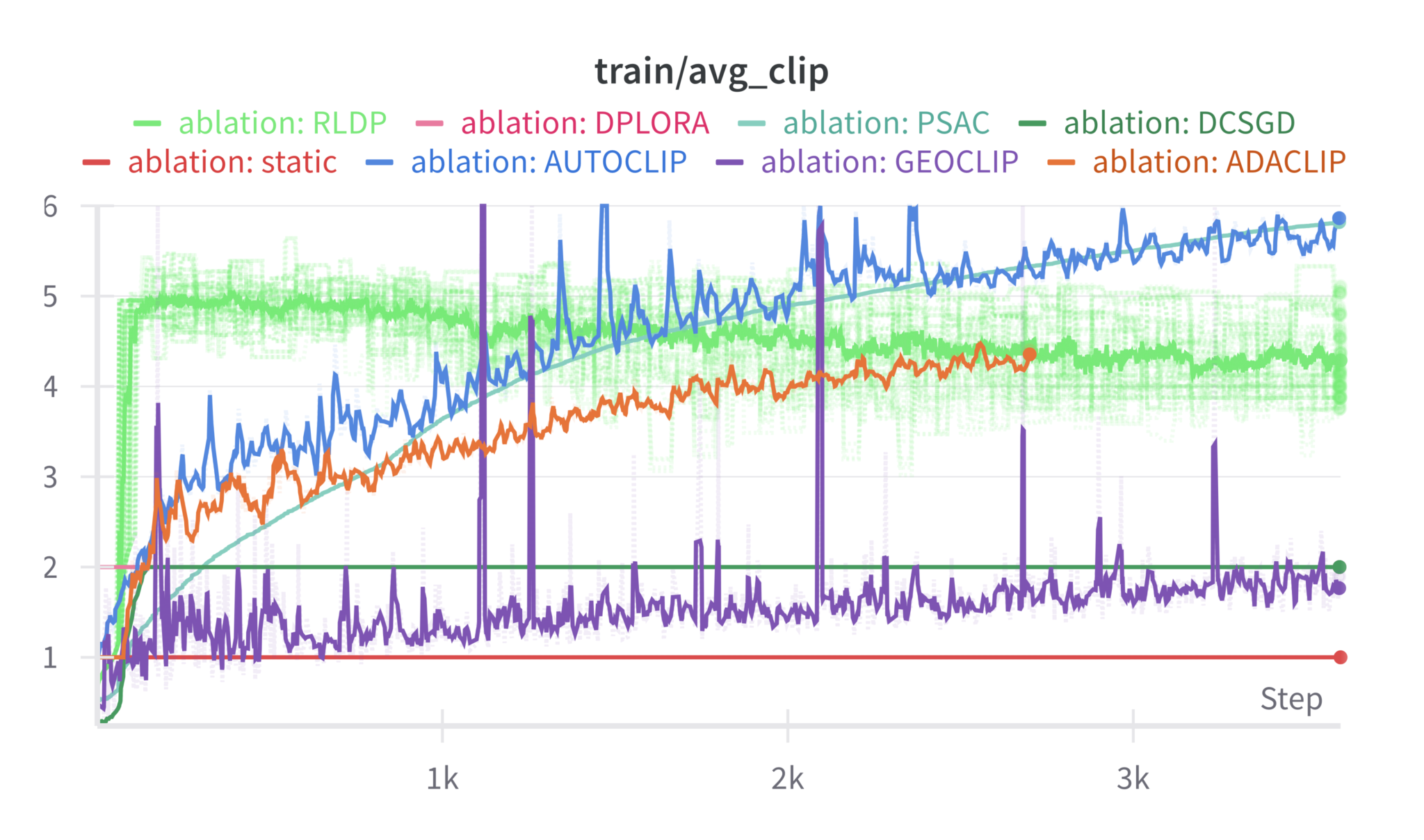}
    
    \label{fig:gpt2_clip_eps4}
  \end{subfigure}\hfill
  \begin{subfigure}[t]{0.45\textwidth}
    \centering
    \includegraphics[width=\textwidth]{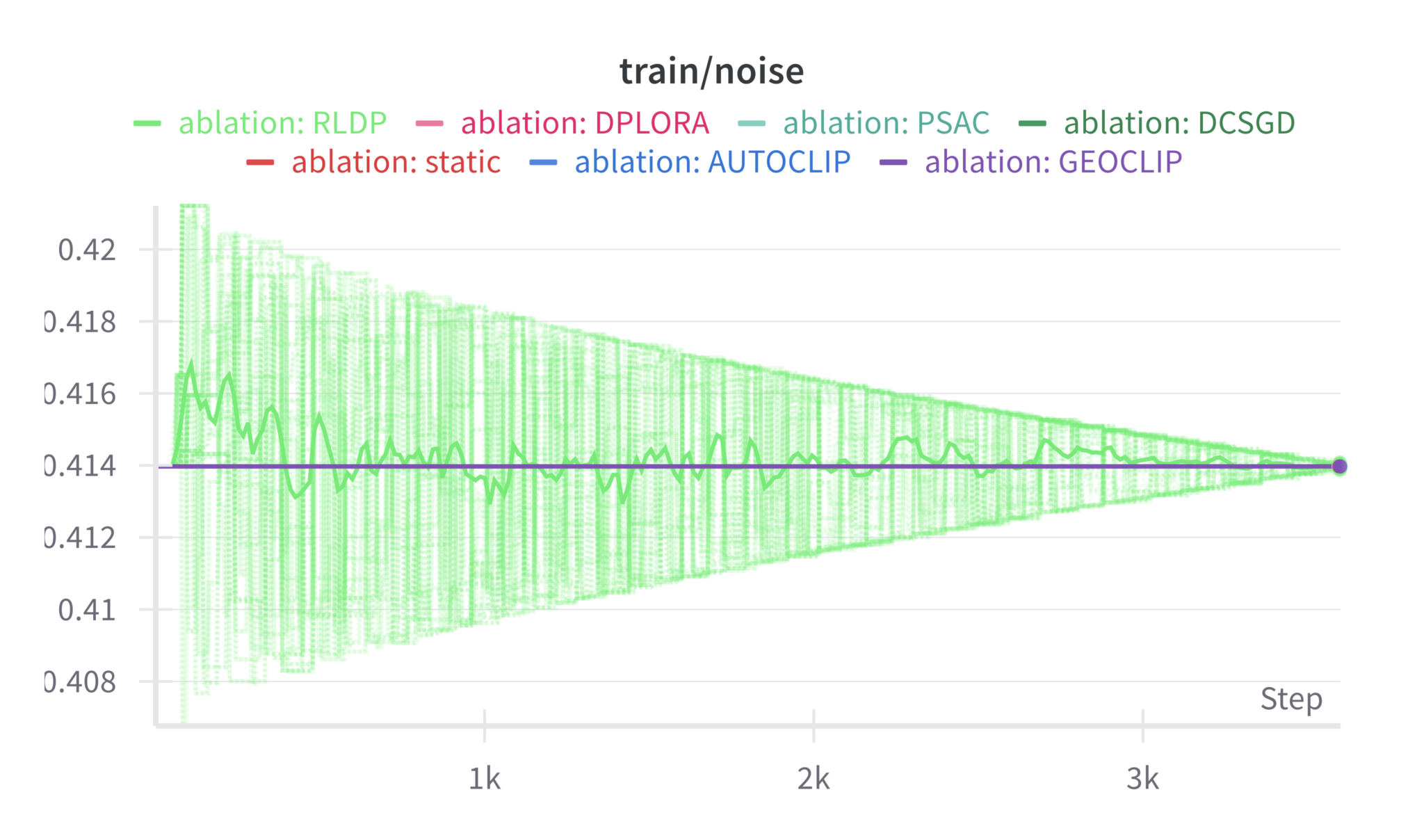}
    
    \label{fig:gpt2_noise_eps4}
  \end{subfigure}

  % Row 4: ε = 5
  \begin{subfigure}[t]{0.45\textwidth}
    \centering
    \includegraphics[width=\textwidth]{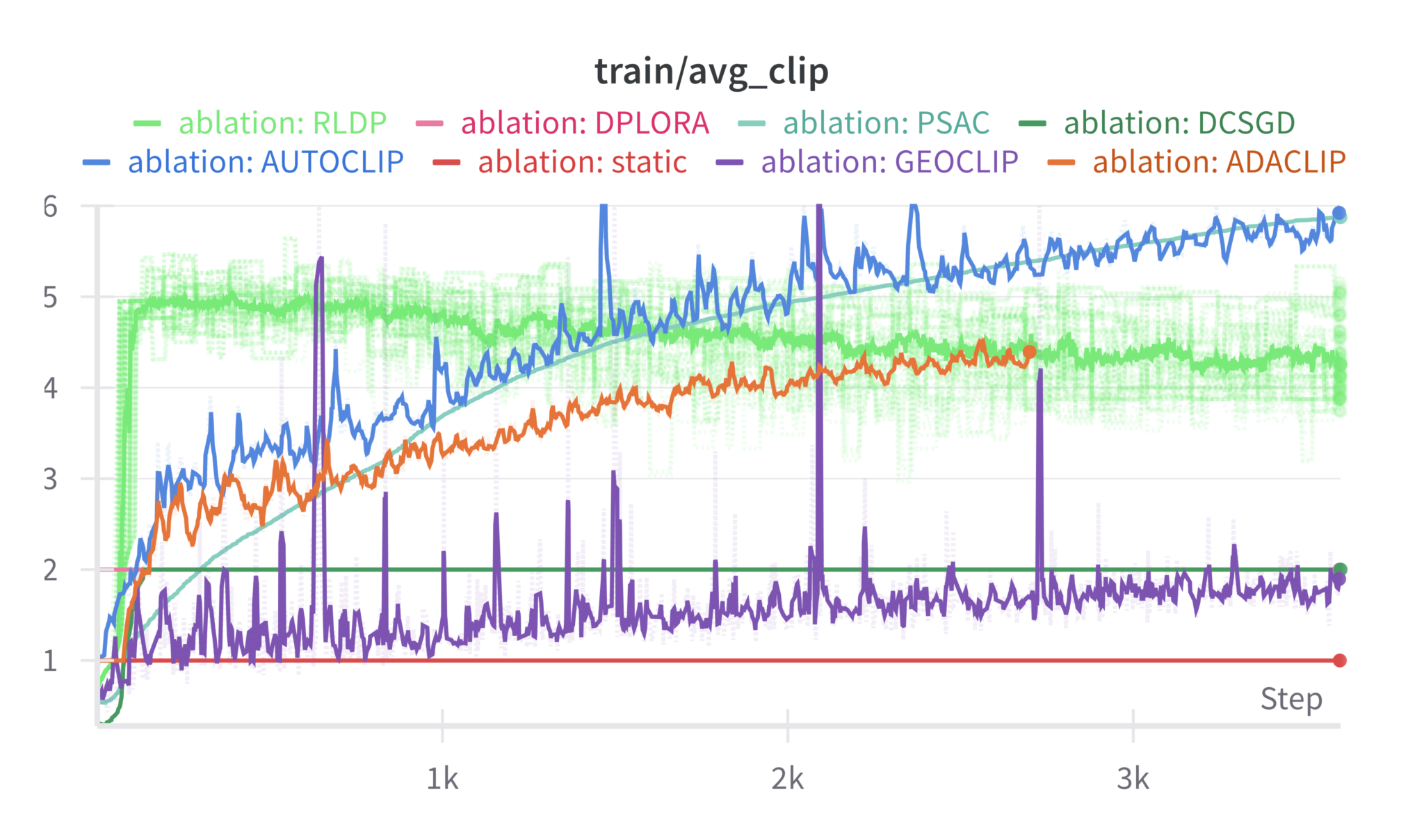}
    
    \label{fig:gpt2_clip_eps5}
  \end{subfigure}\hfill
  \begin{subfigure}[t]{0.45\textwidth}
    \centering
    \includegraphics[width=\textwidth]{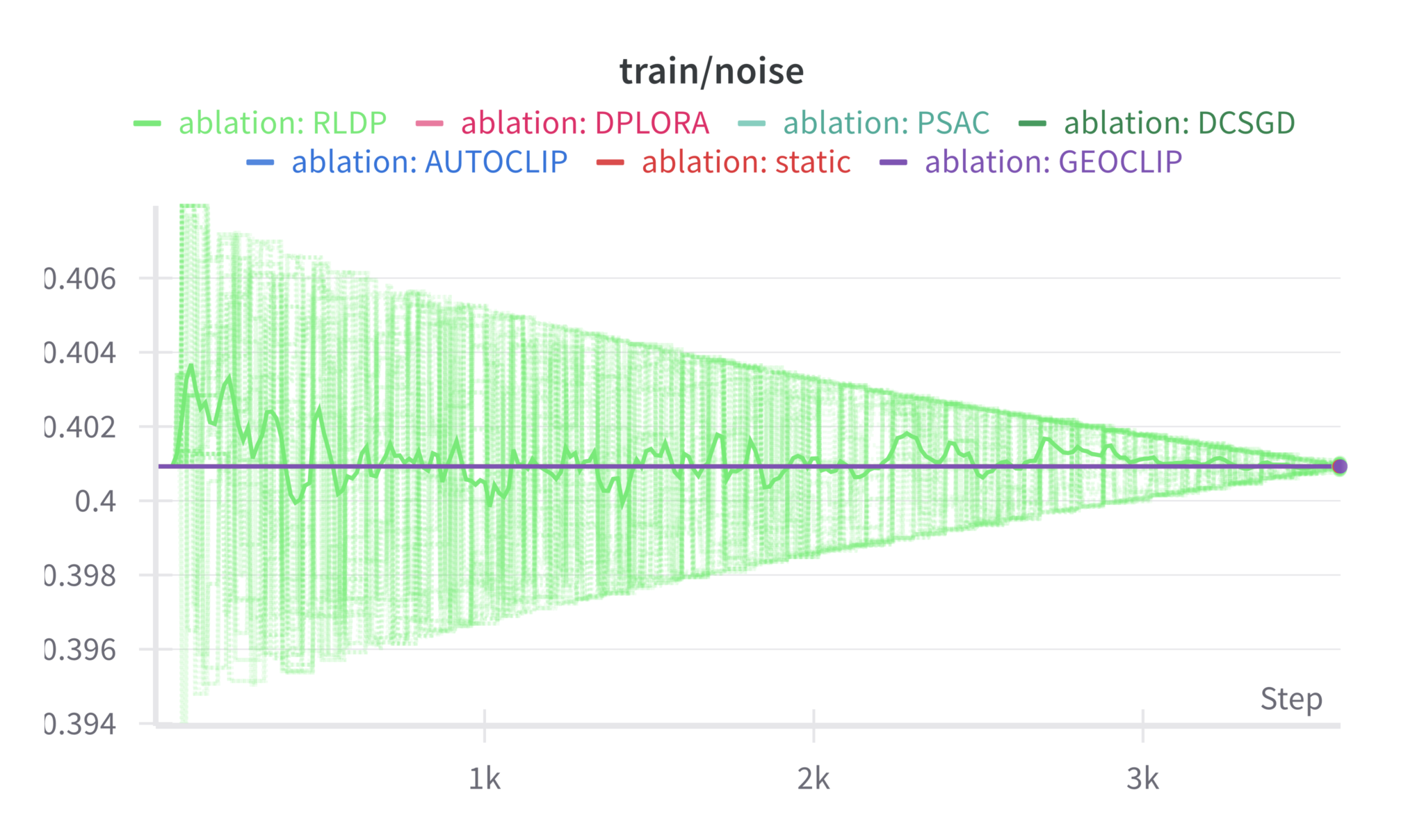}
    
    \label{fig:gpt2_noise_eps5}
  \end{subfigure}

  % Row 5: ε = 8
  \begin{subfigure}[t]{0.45\textwidth}
    \centering
    \includegraphics[width=\textwidth]{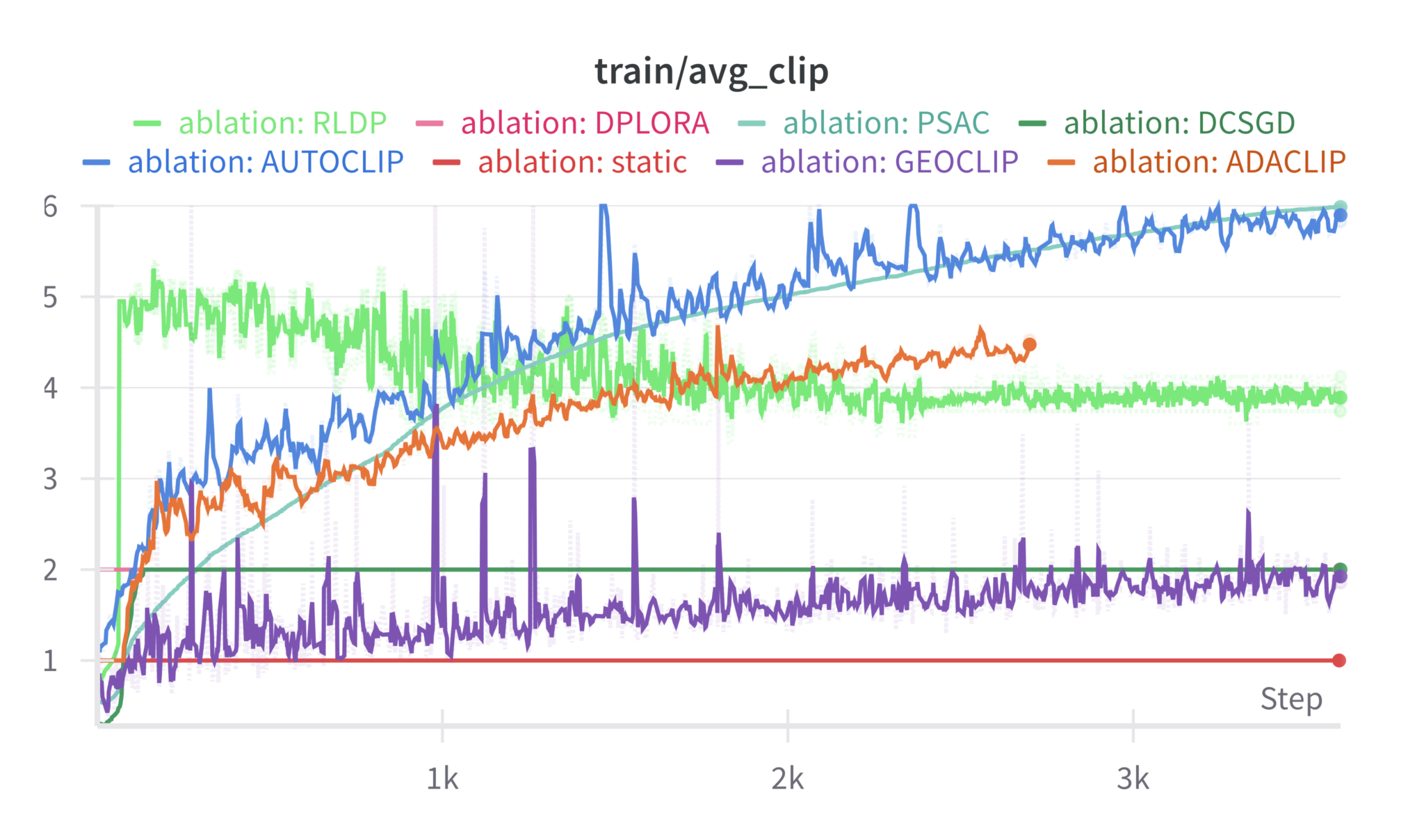}
    
    \label{fig:gpt2_clip_eps8}
  \end{subfigure}\hfill
  \begin{subfigure}[t]{0.45\textwidth}
    \centering
    \includegraphics[width=\textwidth]{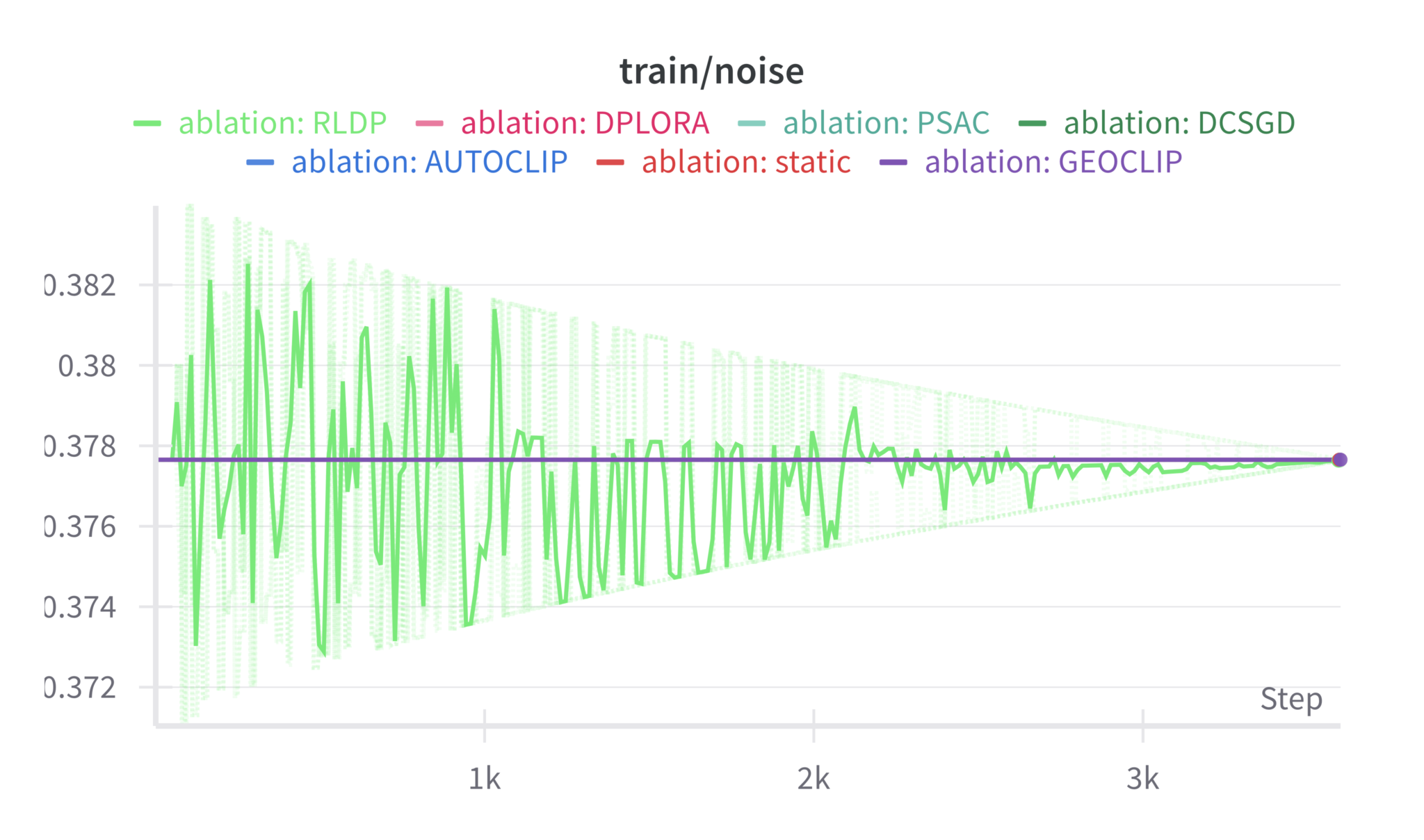}
    
    \label{fig:gpt2_noise_eps8}
  \end{subfigure}

  \caption{Training clip and noise history for GPT2 under different DP budgets \(\varepsilon\). Rows (top to bottom) correspond to \(\varepsilon=0.5,\,2,\,4,\,5,\,8\). Left: training average CLIP; right: DP noise over steps.}
  \label{fig:gpt2_train_clip_noise}
\end{figure}

\begin{figure}[H]
  \centering
  % Row 1: ε = 0.5
  \begin{subfigure}[t]{0.45\textwidth}
    \centering
    \includegraphics[width=\textwidth]{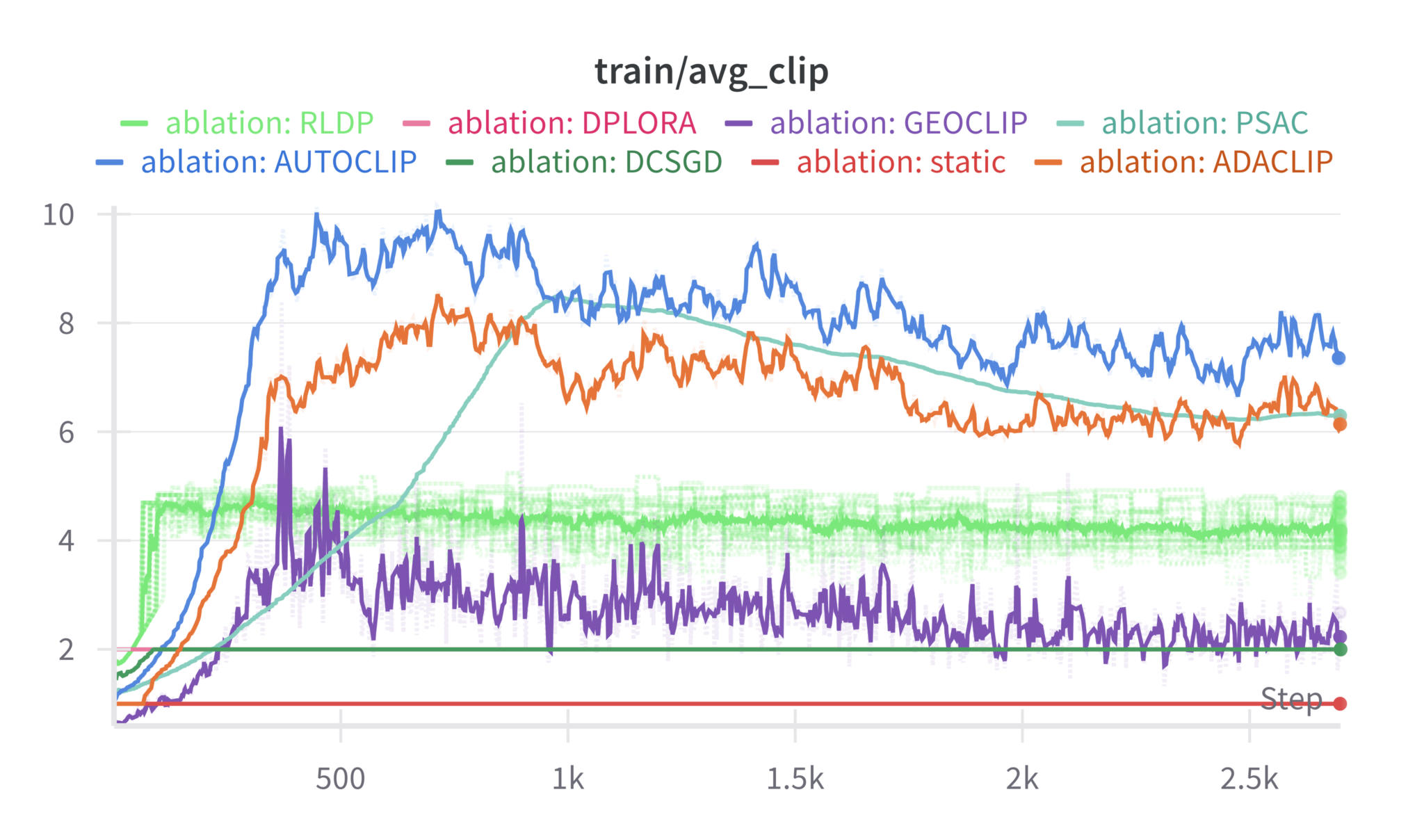}
    
    \label{fig:llama1b_clip_eps05}
  \end{subfigure}\hfill
  \begin{subfigure}[t]{0.45\textwidth}
    \centering
    \includegraphics[width=\textwidth]{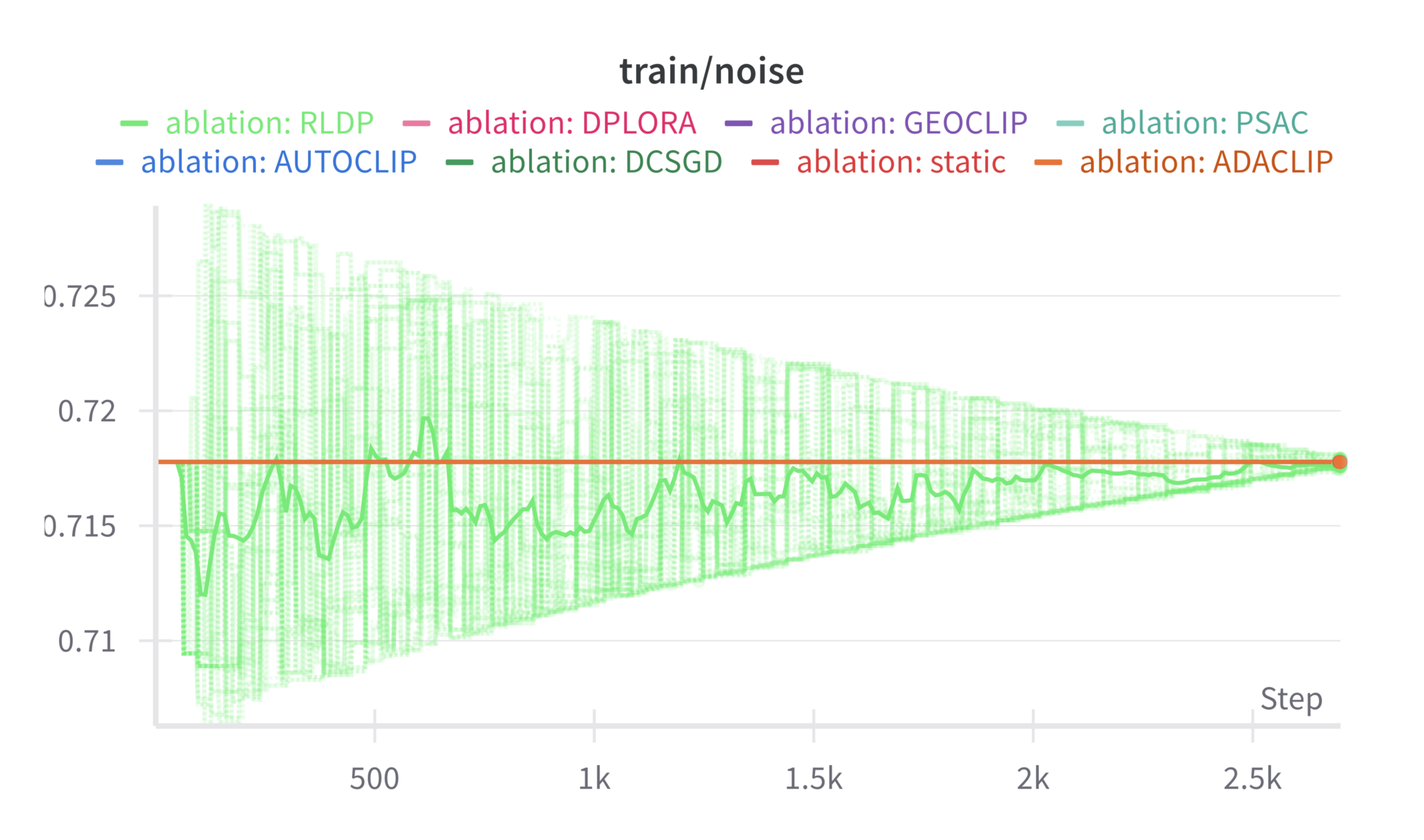}
    
    \label{fig:llama1b_noise_eps05}
  \end{subfigure}

  % Row 2: ε = 2
  \begin{subfigure}[t]{0.45\textwidth}
    \centering
    \includegraphics[width=\textwidth]{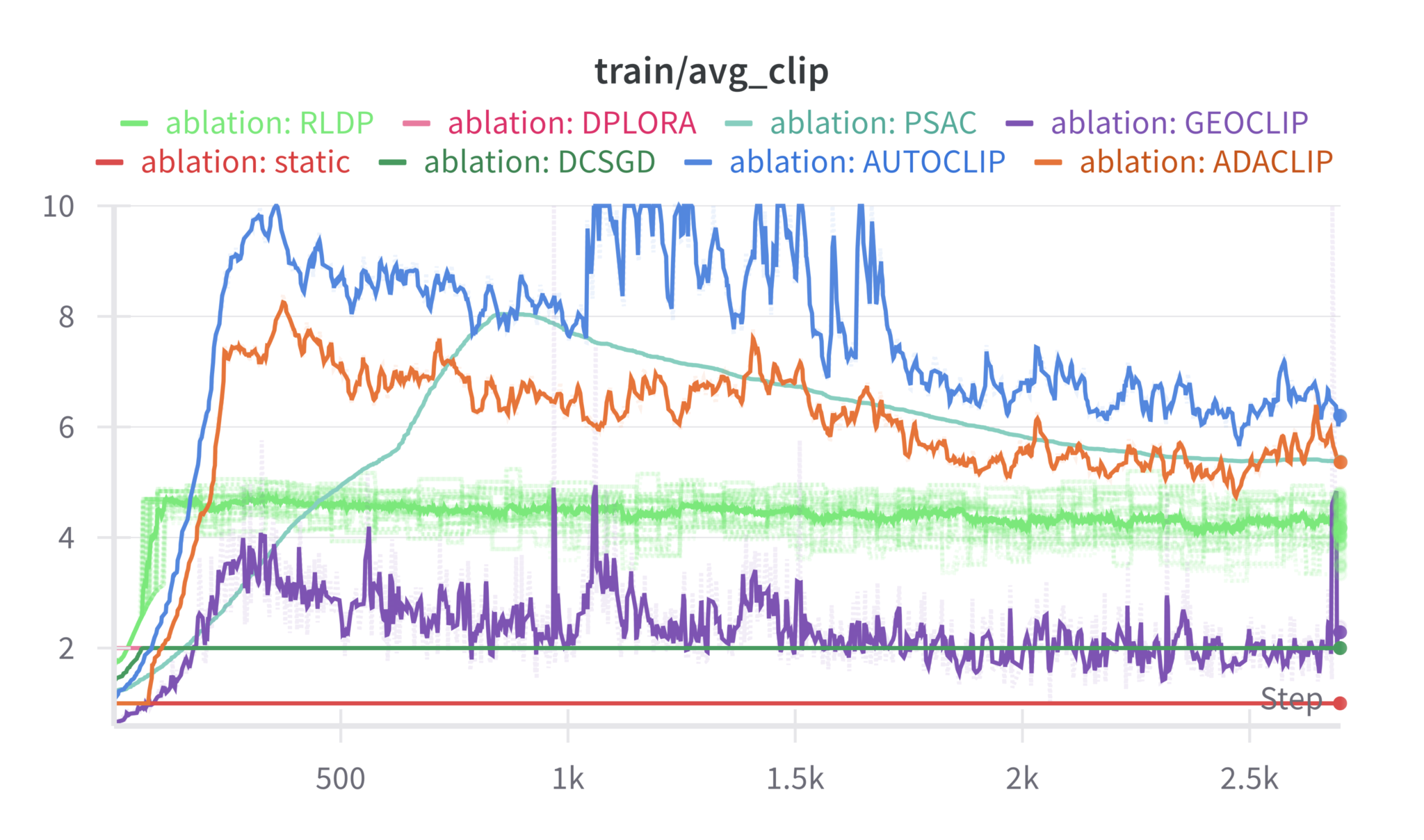}
    
    \label{fig:llama1b_clip_eps2}
  \end{subfigure}\hfill
  \begin{subfigure}[t]{0.45\textwidth}
    \centering
    \includegraphics[width=\textwidth]{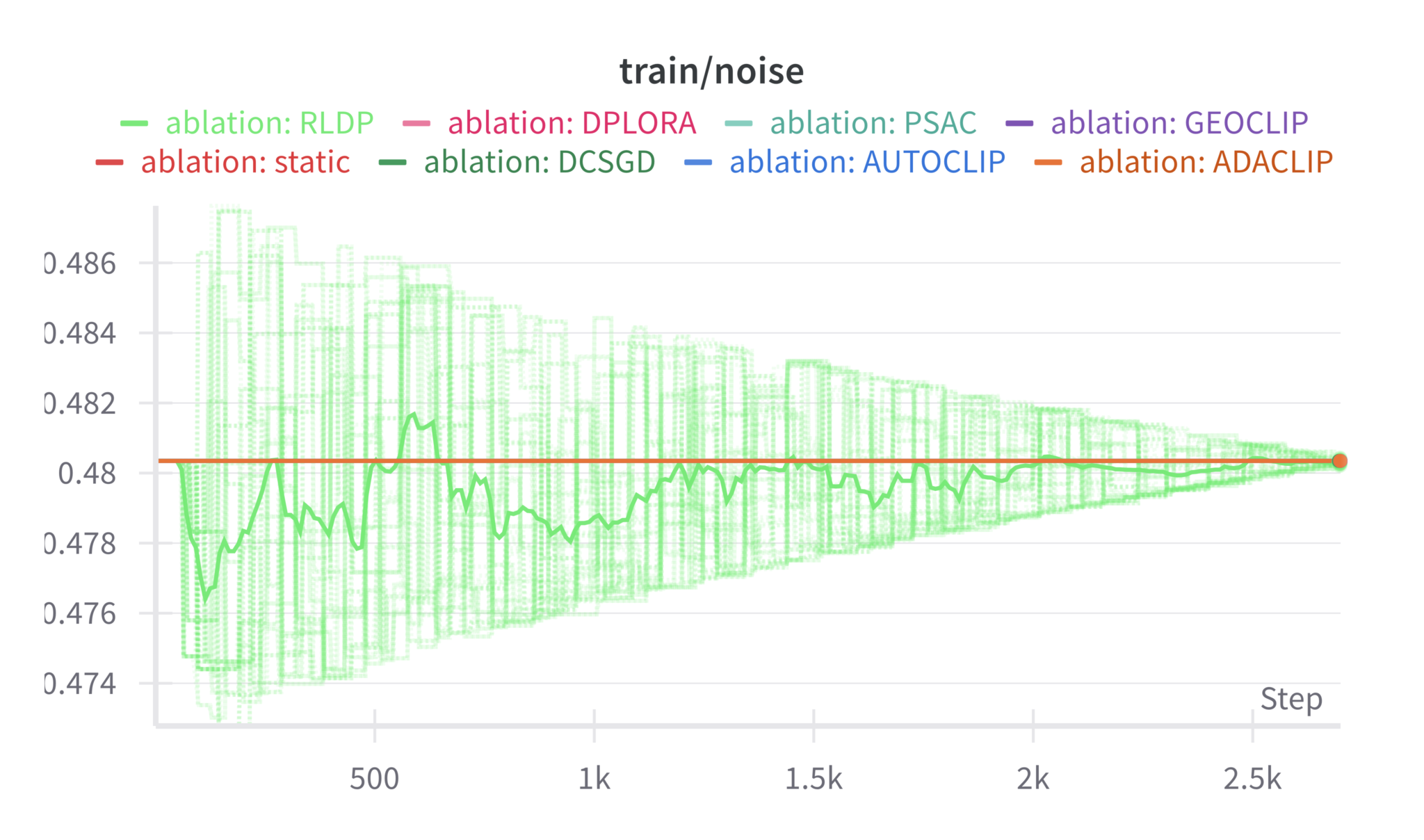}
    
    \label{fig:llama1b_noise_eps2}
  \end{subfigure}

  % Row 3: ε = 4
  \begin{subfigure}[t]{0.45\textwidth}
    \centering
    \includegraphics[width=\textwidth]{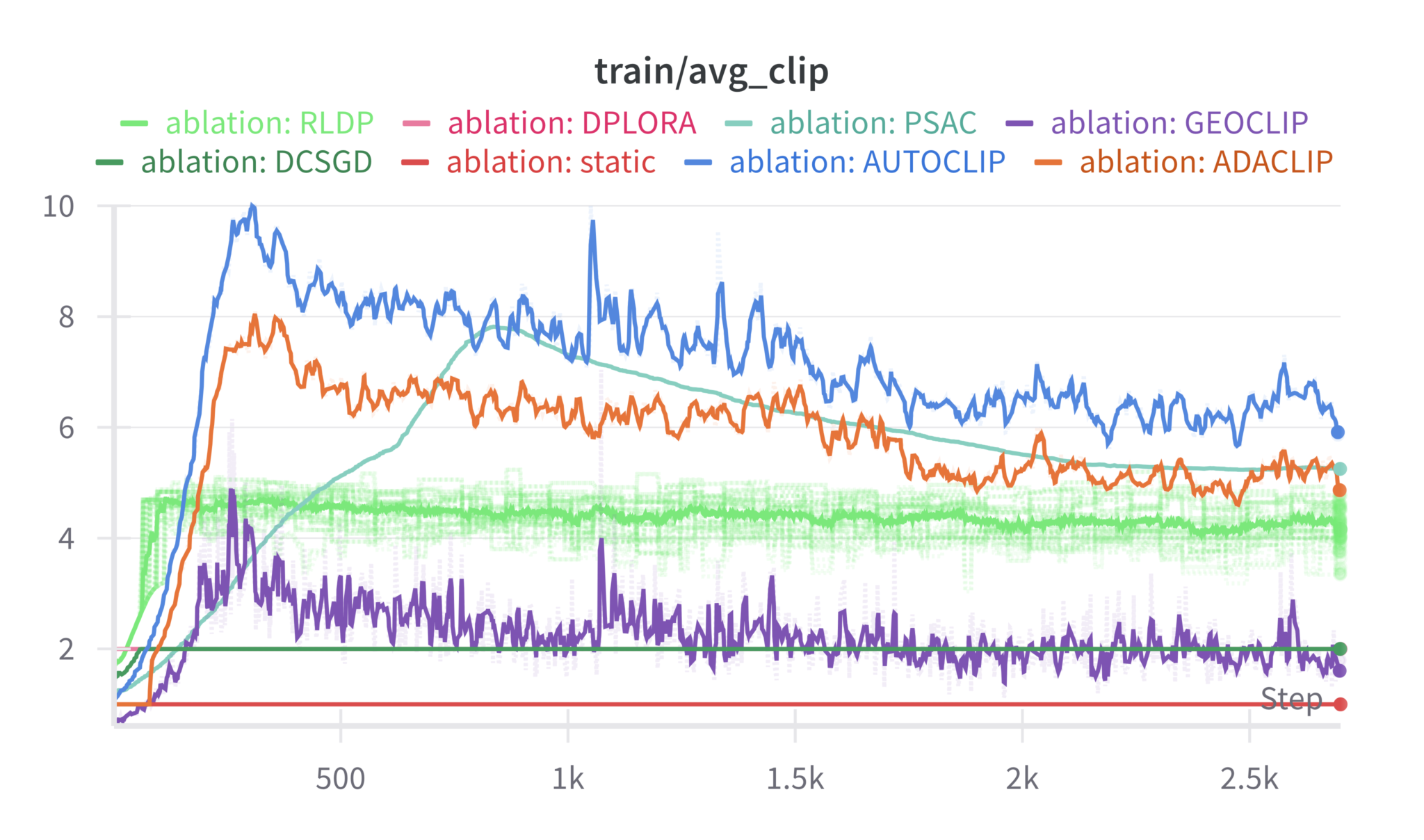}
    
    \label{fig:llama1b_clip_eps4}
  \end{subfigure}\hfill
  \begin{subfigure}[t]{0.45\textwidth}
    \centering
    \includegraphics[width=\textwidth]{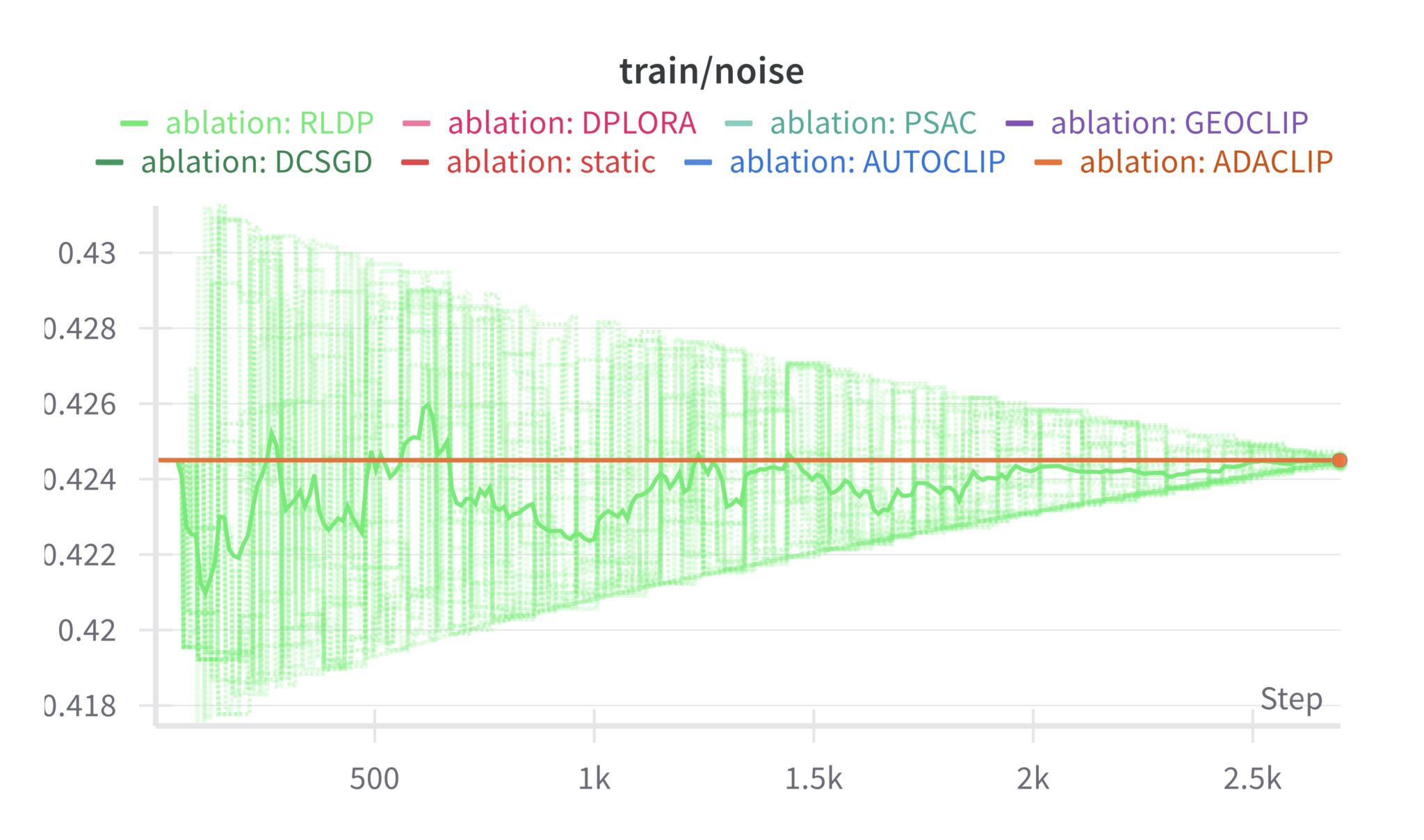}
    
    \label{fig:llama1b_noise_eps4}
  \end{subfigure}

  % Row 4: ε = 5
  \begin{subfigure}[t]{0.45\textwidth}
    \centering
    \includegraphics[width=\textwidth]{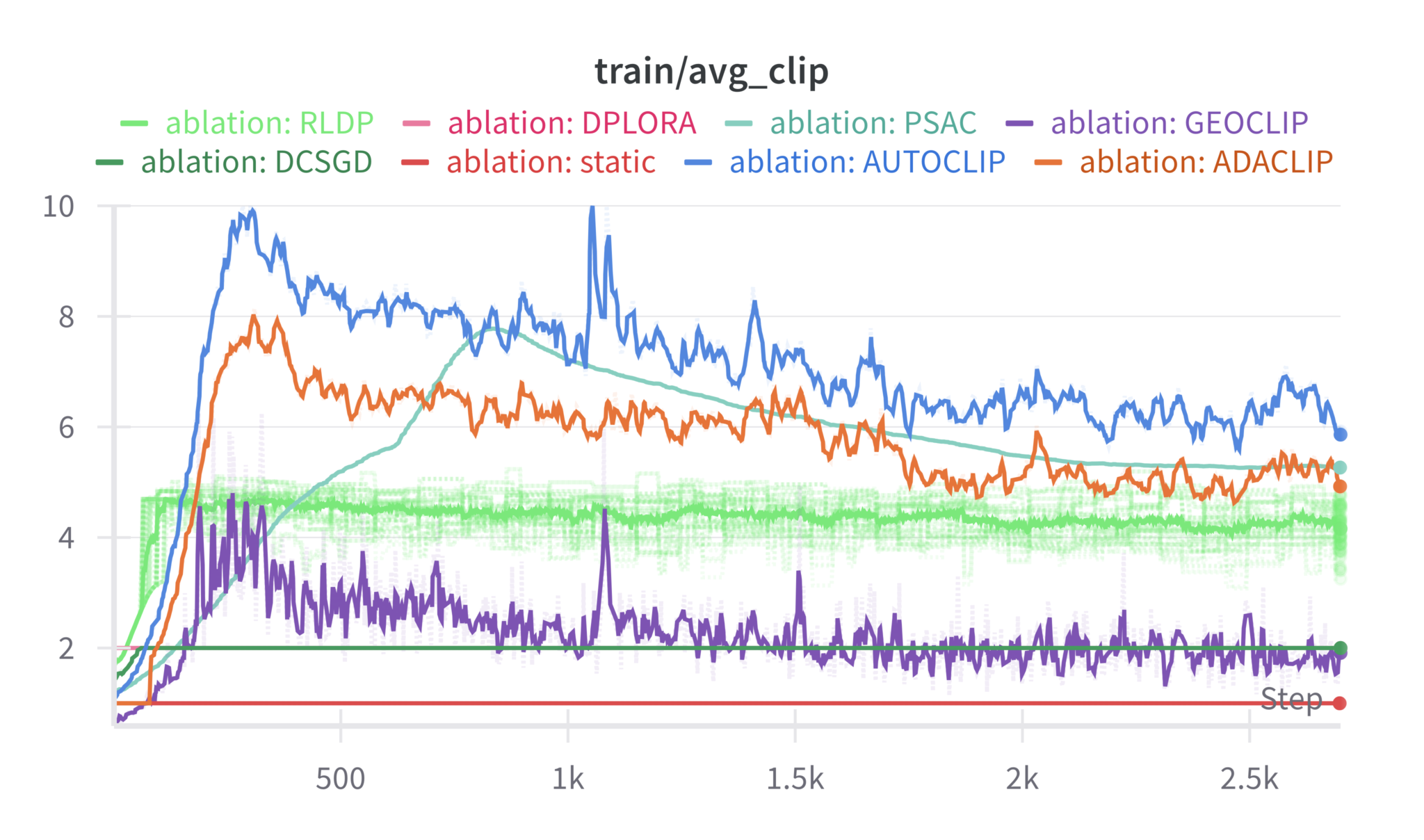}
    
    \label{fig:llama1b_clip_eps5}
  \end{subfigure}\hfill
  \begin{subfigure}[t]{0.45\textwidth}
    \centering
    \includegraphics[width=\textwidth]{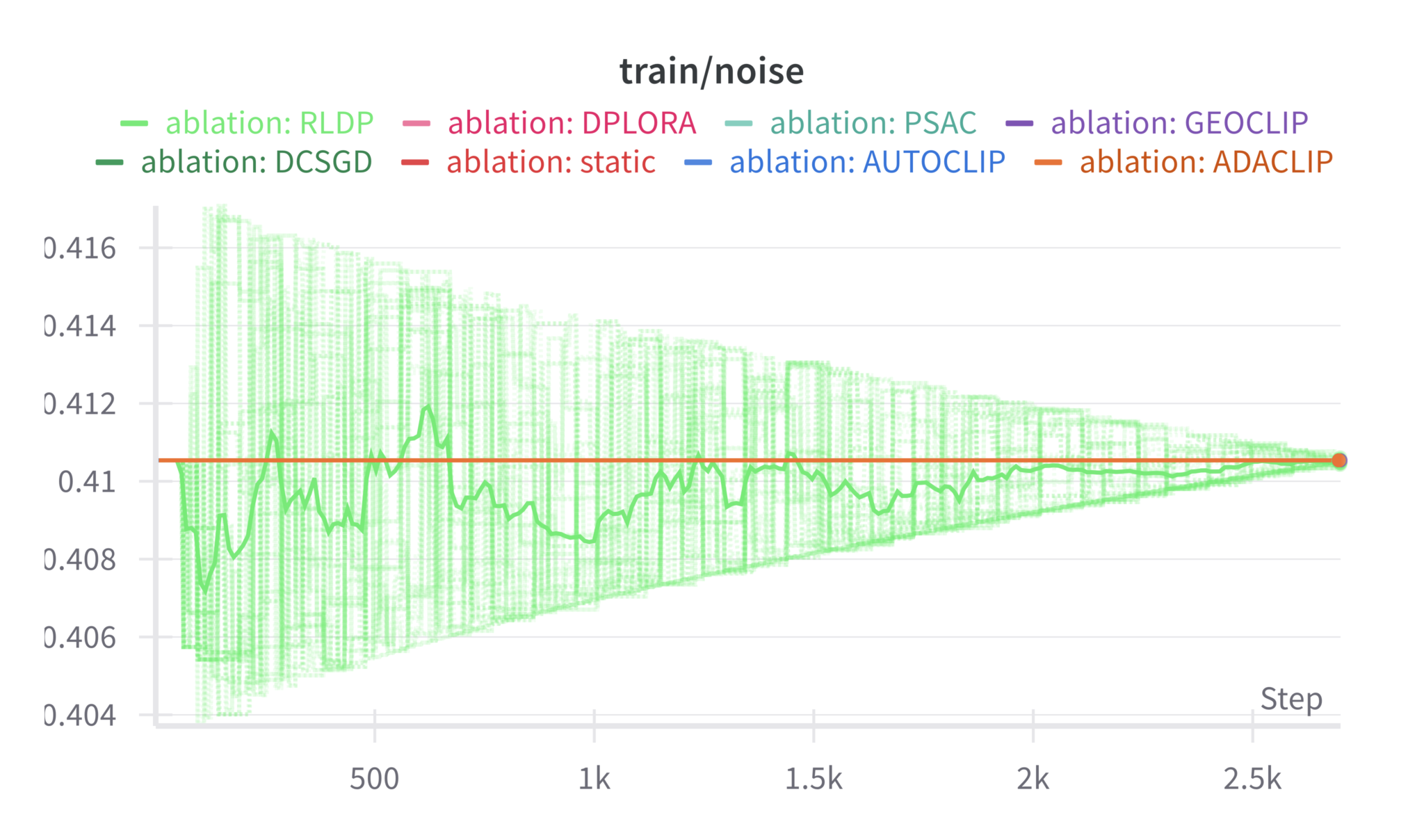}
    
    \label{fig:llama1b_noise_eps5}
  \end{subfigure}

  % Row 5: ε = 8
  \begin{subfigure}[t]{0.45\textwidth}
    \centering
    \includegraphics[width=\textwidth]{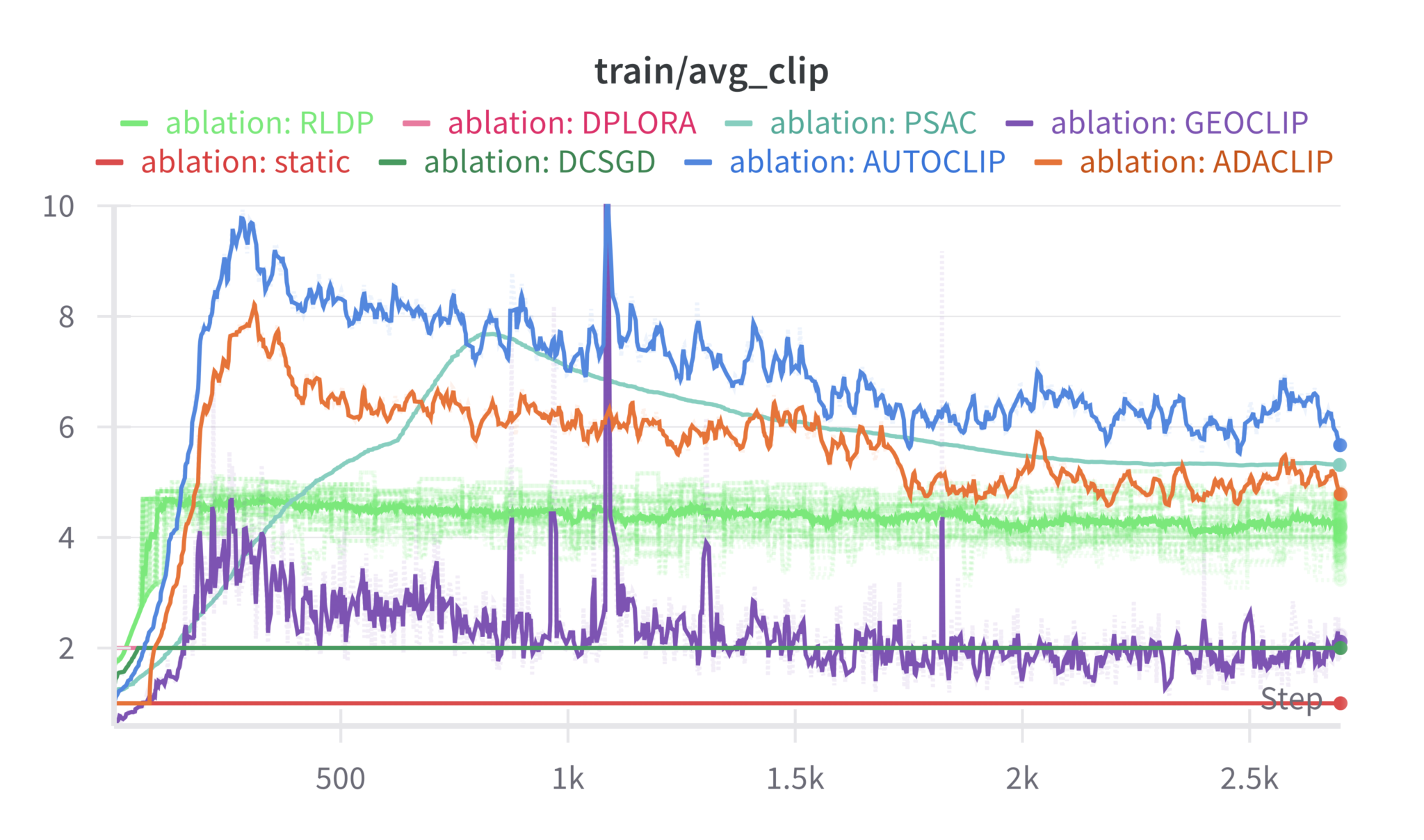}
    
    \label{fig:llama1b_clip_eps8}
  \end{subfigure}\hfill
  \begin{subfigure}[t]{0.45\textwidth}
    \centering
    \includegraphics[width=\textwidth]{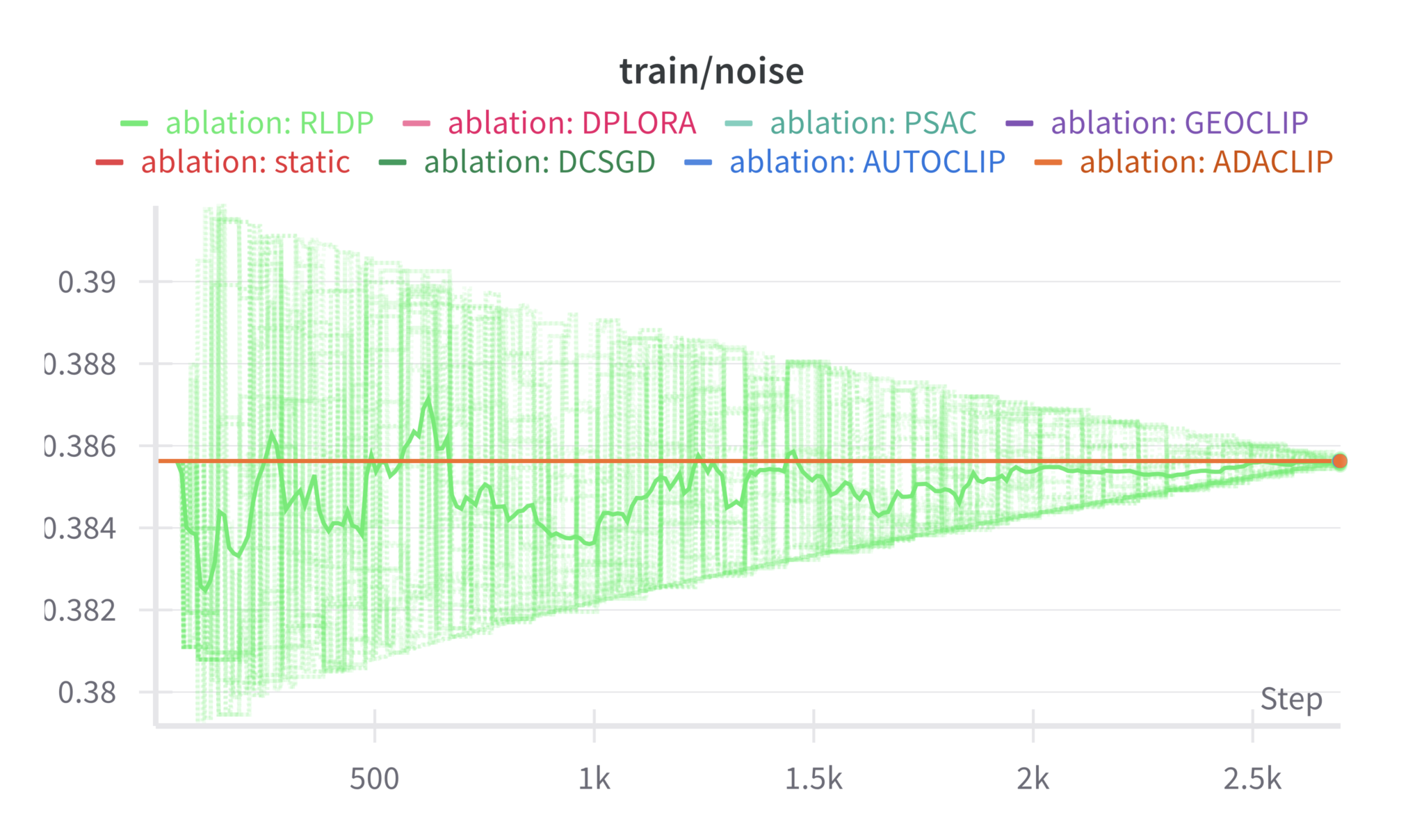}
    
    \label{fig:llama1b_noise_eps8}
  \end{subfigure}

  \caption{Training clip and noise history for Llama-1B under different DP budgets \(\varepsilon\). Rows (top to bottom) correspond to \(\varepsilon=0.5,\,2,\,4,\,5,\,8\). Left: training average CLIP; right: DP noise over steps.}
  \label{fig:llama1b_train_clip_noise}
\end{figure}

\begin{figure}[H]
  \centering
  % Row 1: ε = 0.5
  \begin{subfigure}[t]{0.45\textwidth}
    \centering
    \includegraphics[width=\textwidth]{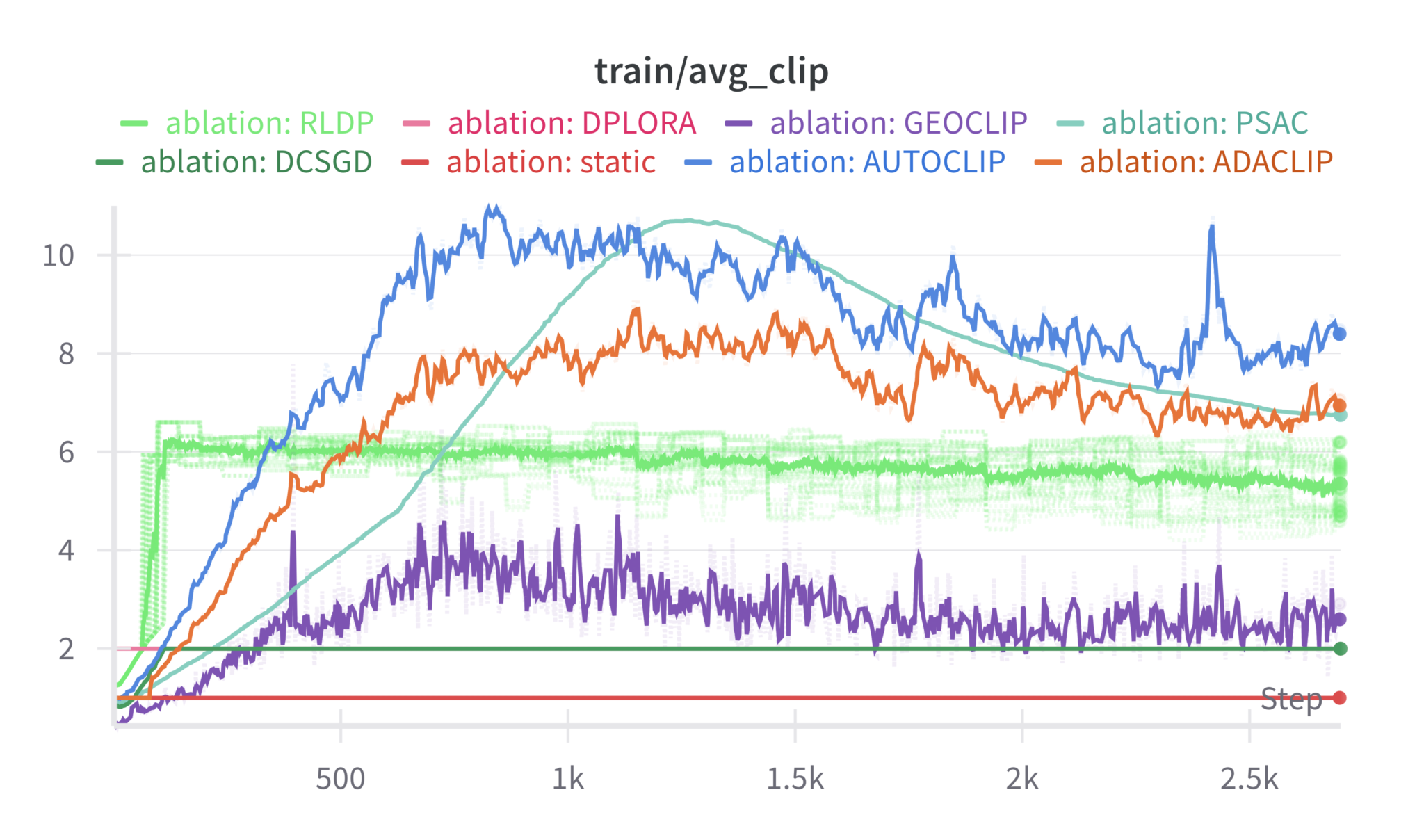}
    
    \label{fig:llama3b_clip_eps05}
  \end{subfigure}\hfill
  \begin{subfigure}[t]{0.45\textwidth}
    \centering
    \includegraphics[width=\textwidth]{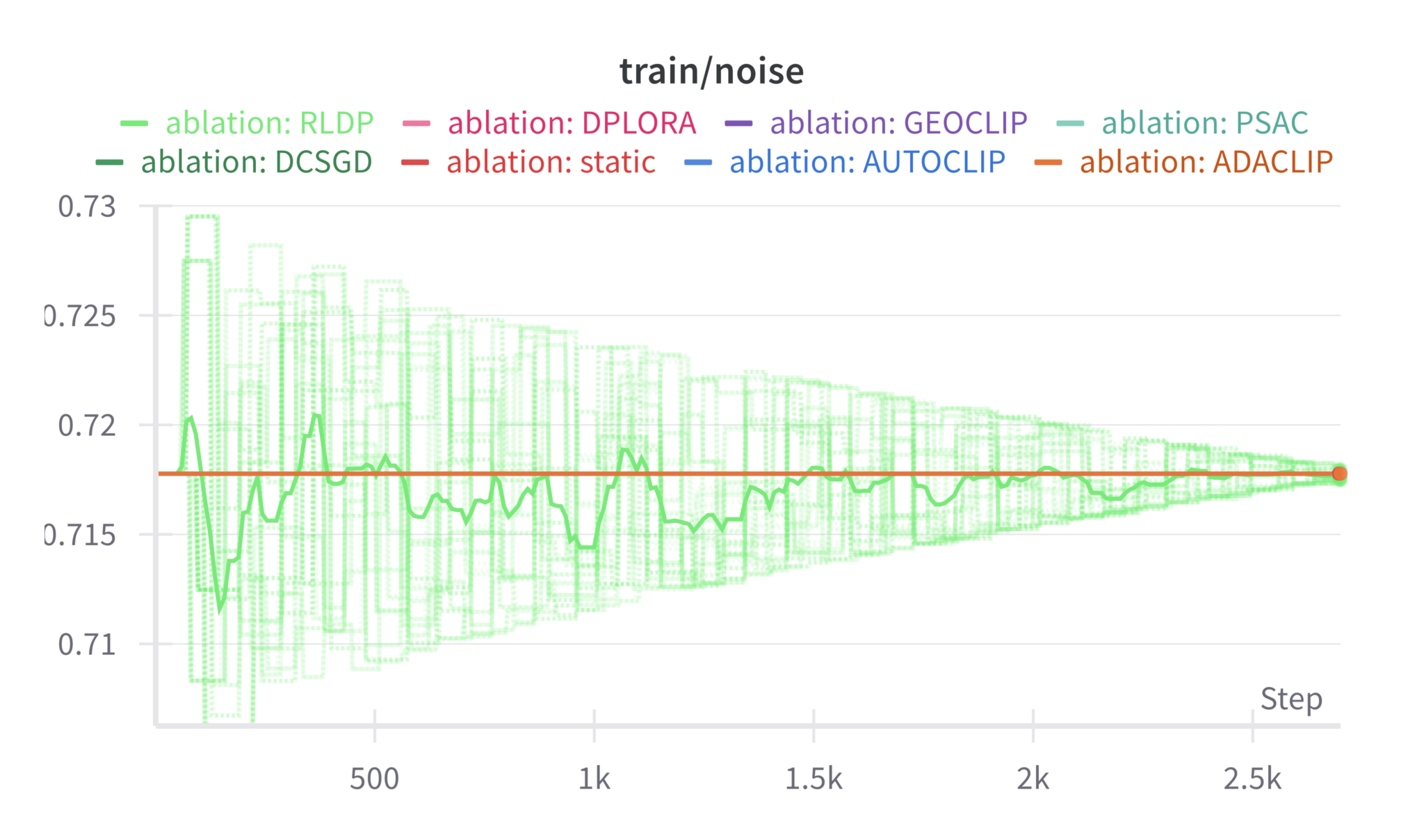}
    
    \label{fig:llama3b_noise_eps05}
  \end{subfigure}

  % Row 2: ε = 2
  \begin{subfigure}[t]{0.45\textwidth}
    \centering
    \includegraphics[width=\textwidth]{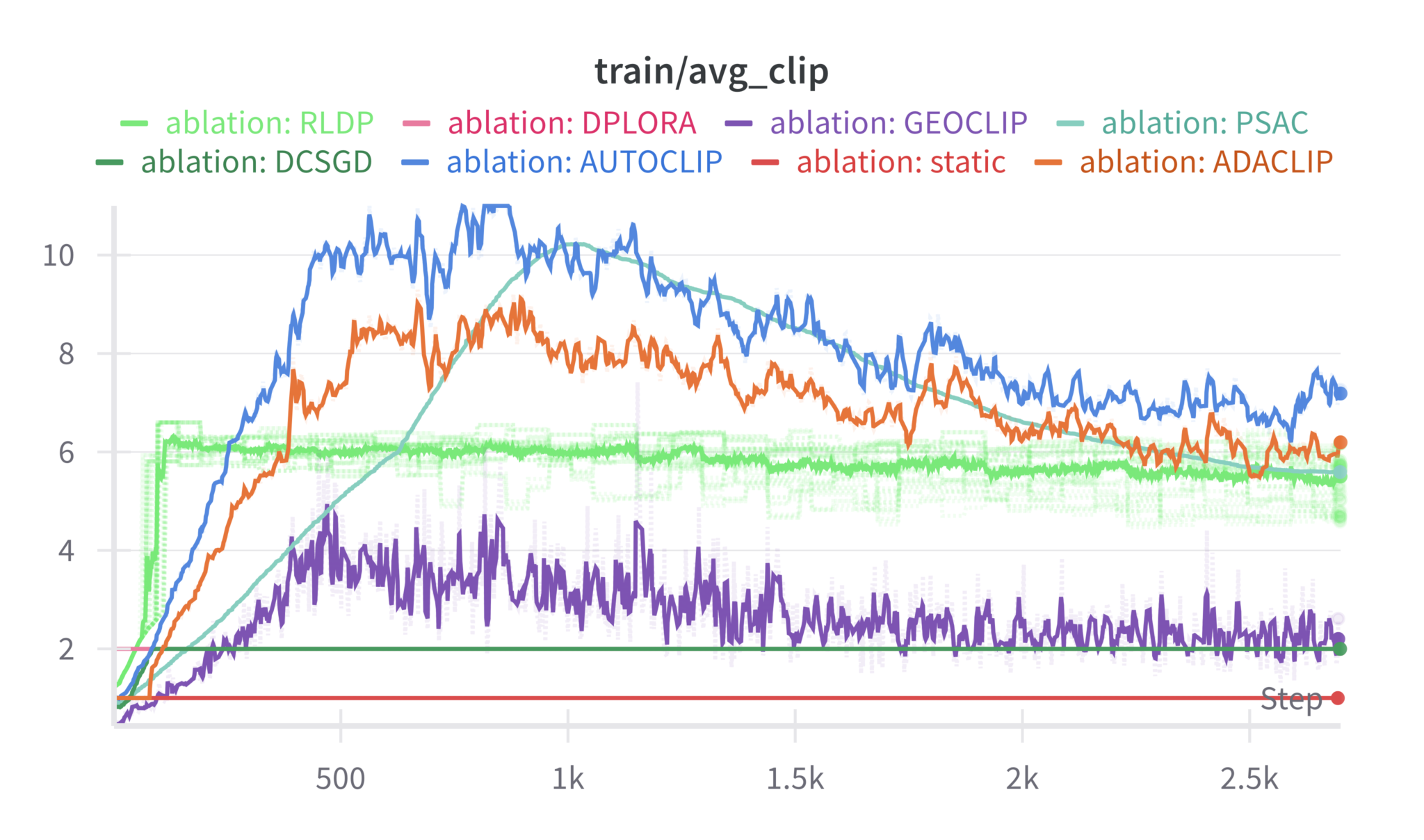}
    
    \label{fig:llama3b_clip_eps2}
  \end{subfigure}\hfill
  \begin{subfigure}[t]{0.45\textwidth}
    \centering
    \includegraphics[width=\textwidth]{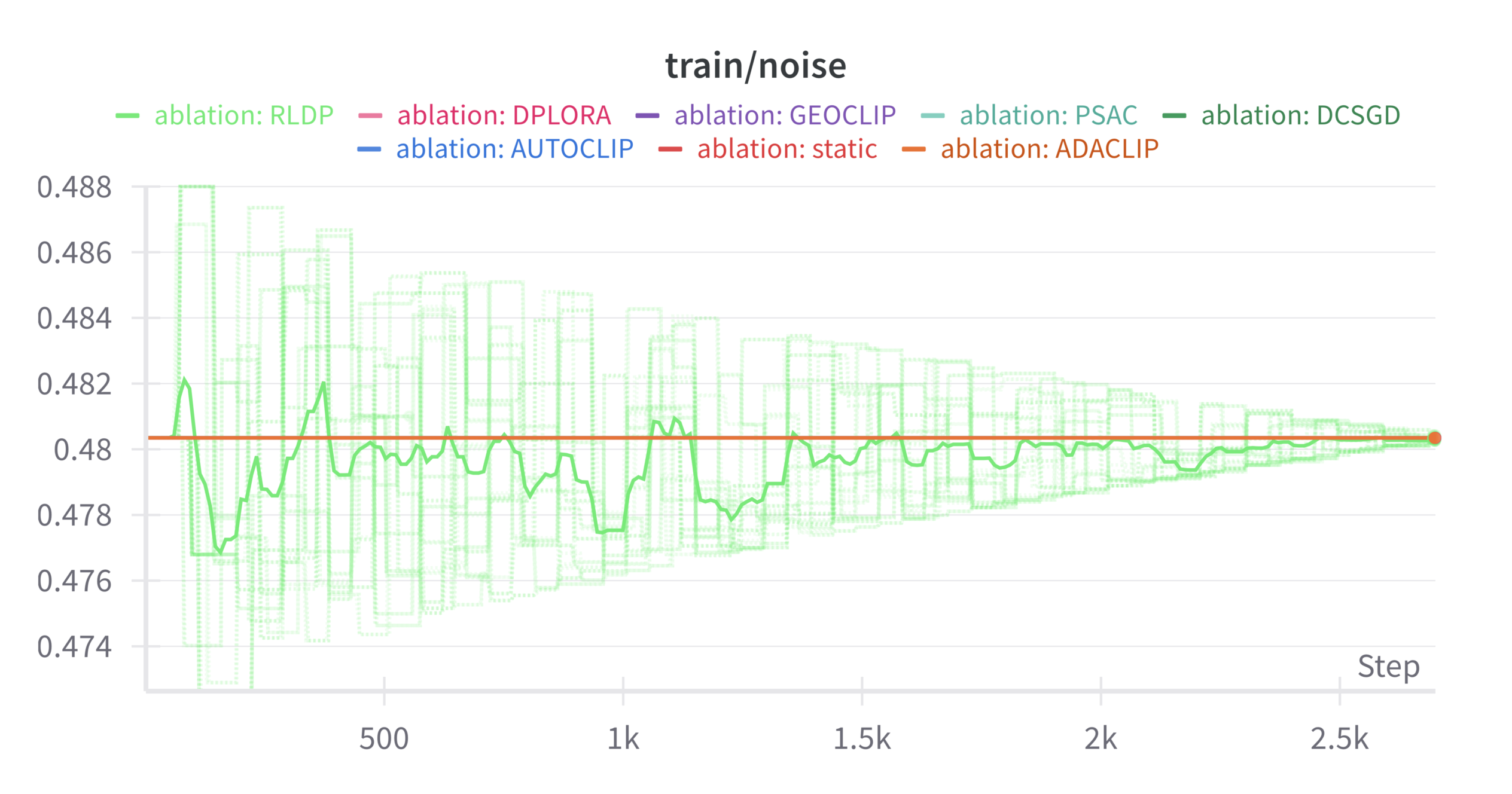}
    
    \label{fig:llama3b_noise_eps2}
  \end{subfigure}

  % Row 3: ε = 4
  \begin{subfigure}[t]{0.45\textwidth}
    \centering
    \includegraphics[width=\textwidth]{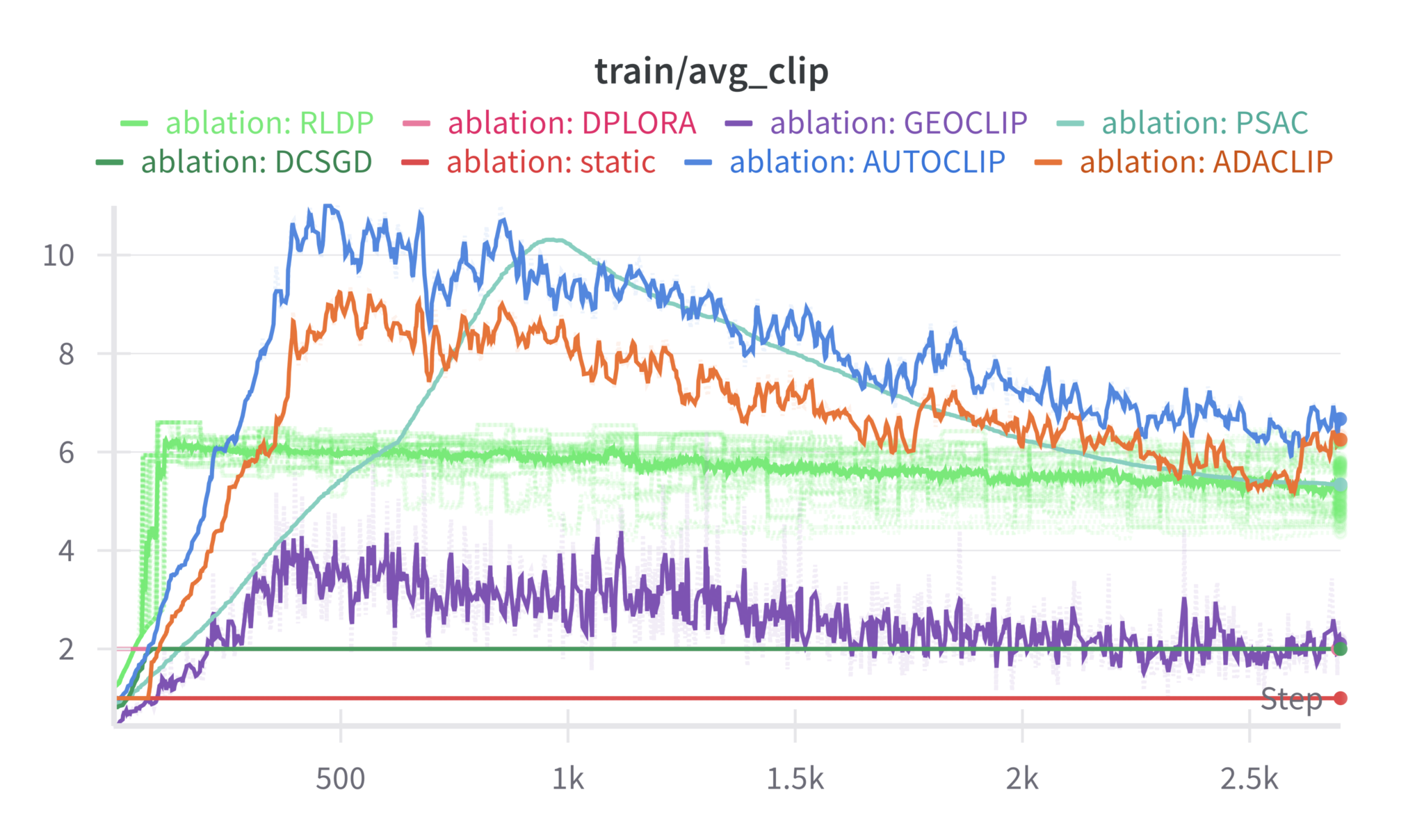}
    
    \label{fig:llama3b_clip_eps4}
  \end{subfigure}\hfill
  \begin{subfigure}[t]{0.45\textwidth}
    \centering
    \includegraphics[width=\textwidth]{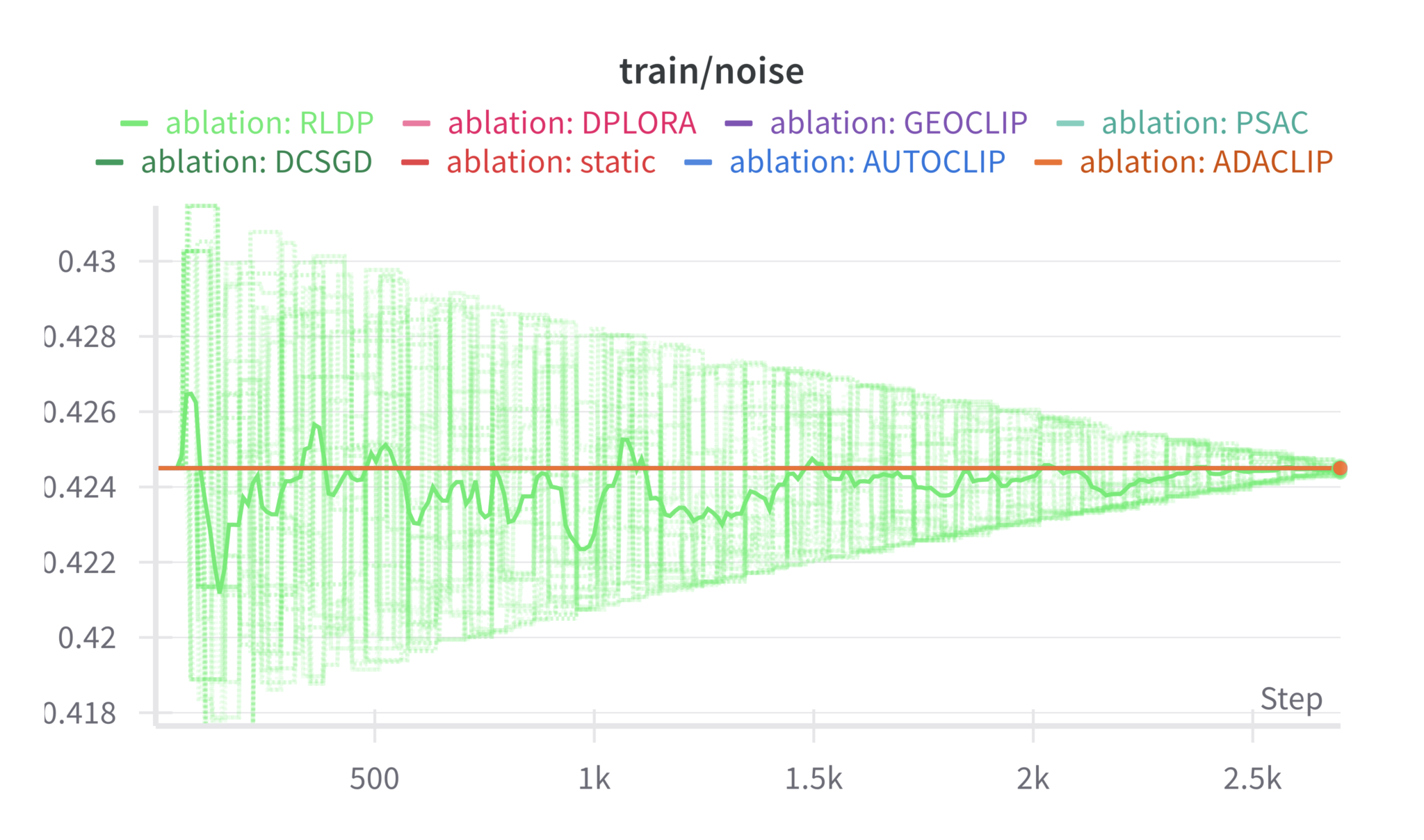}
    
    \label{fig:llama3b_noise_eps4}
  \end{subfigure}

  % Row 4: ε = 5
  \begin{subfigure}[t]{0.45\textwidth}
    \centering
    \includegraphics[width=\textwidth]{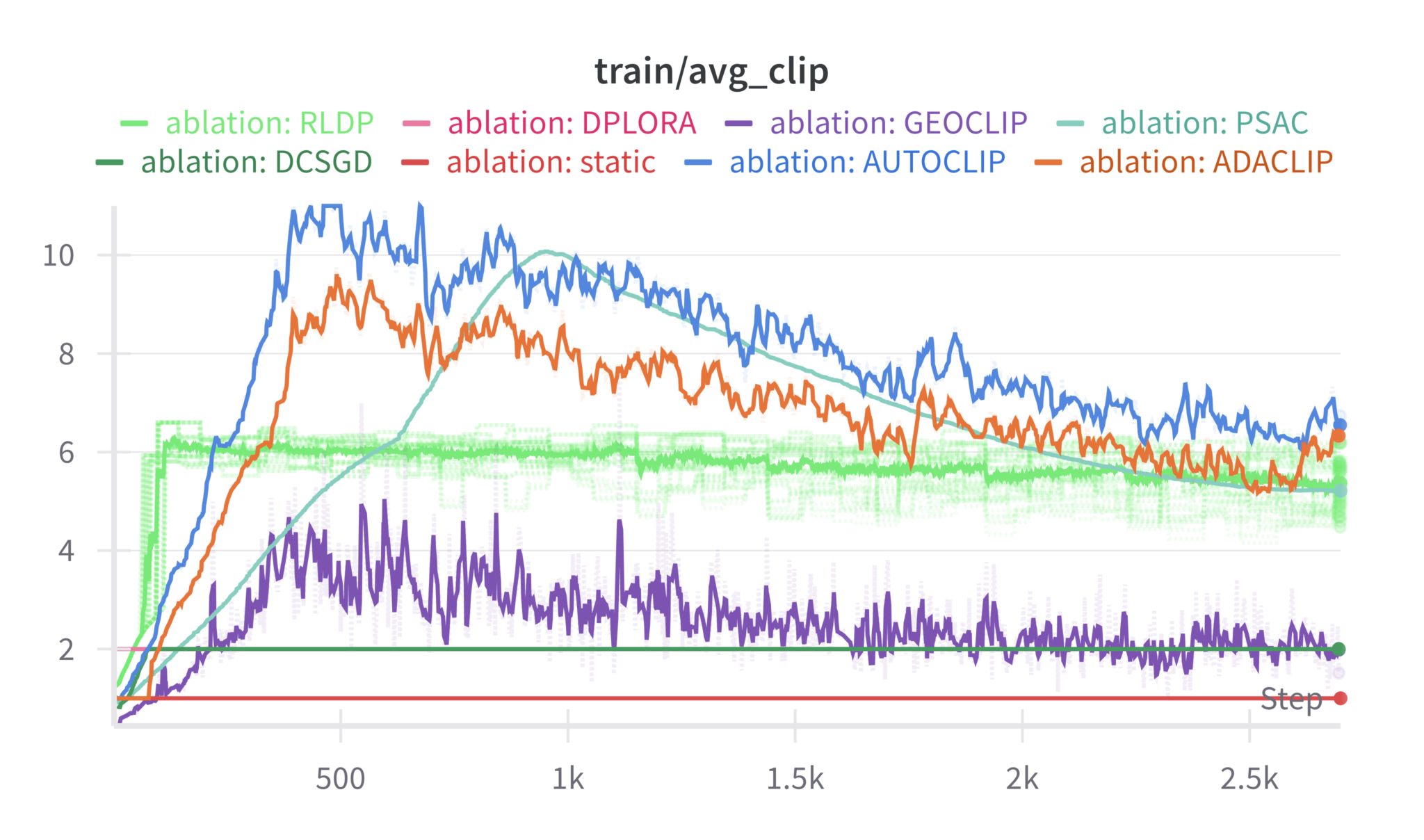}
    
    \label{fig:llama3b_clip_eps5}
  \end{subfigure}\hfill
  \begin{subfigure}[t]{0.45\textwidth}
    \centering
    \includegraphics[width=\textwidth]{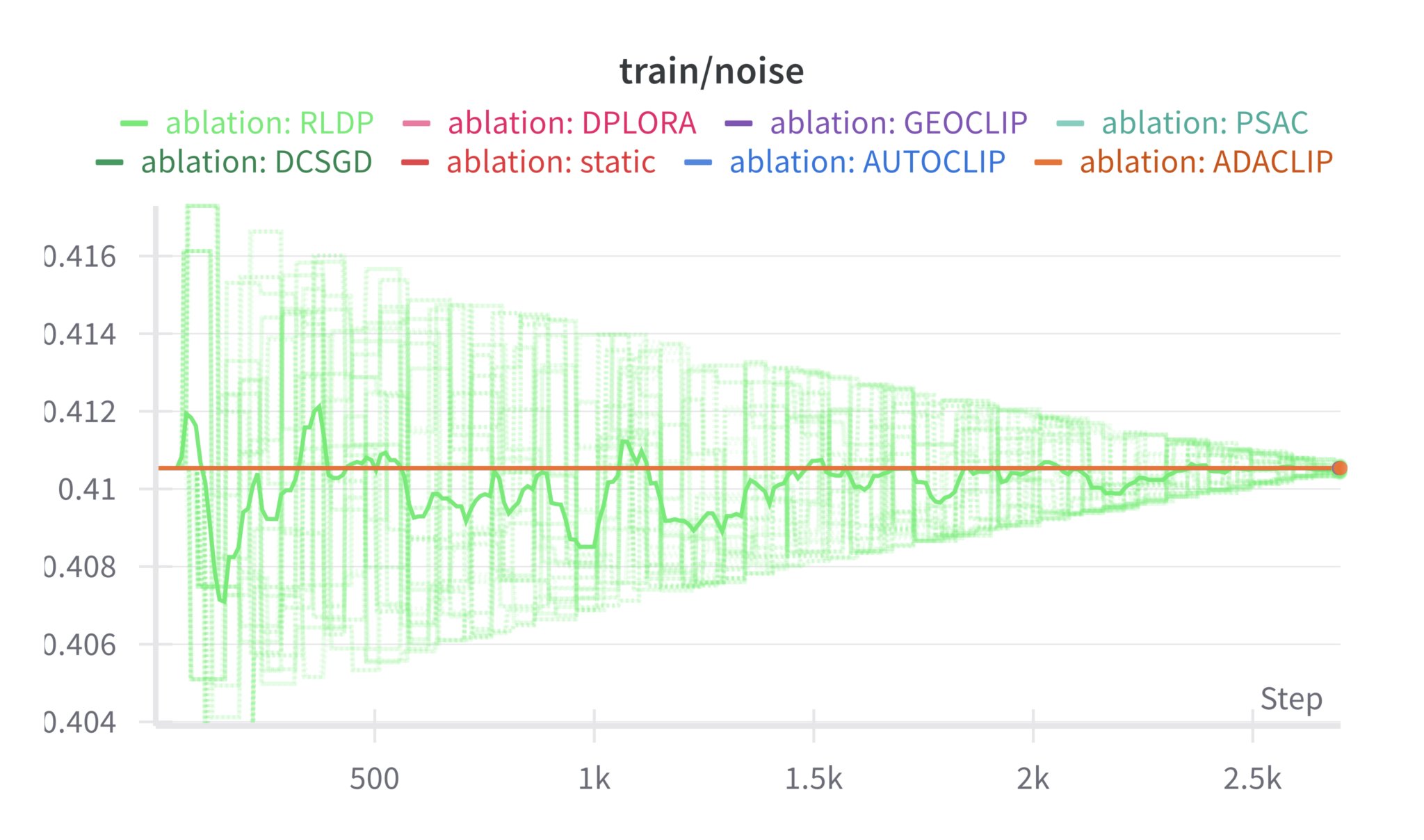}
    
    \label{fig:llama3b_noise_eps5}
  \end{subfigure}

  % Row 5: ε = 8
  \begin{subfigure}[t]{0.45\textwidth}
    \centering
    \includegraphics[width=\textwidth]{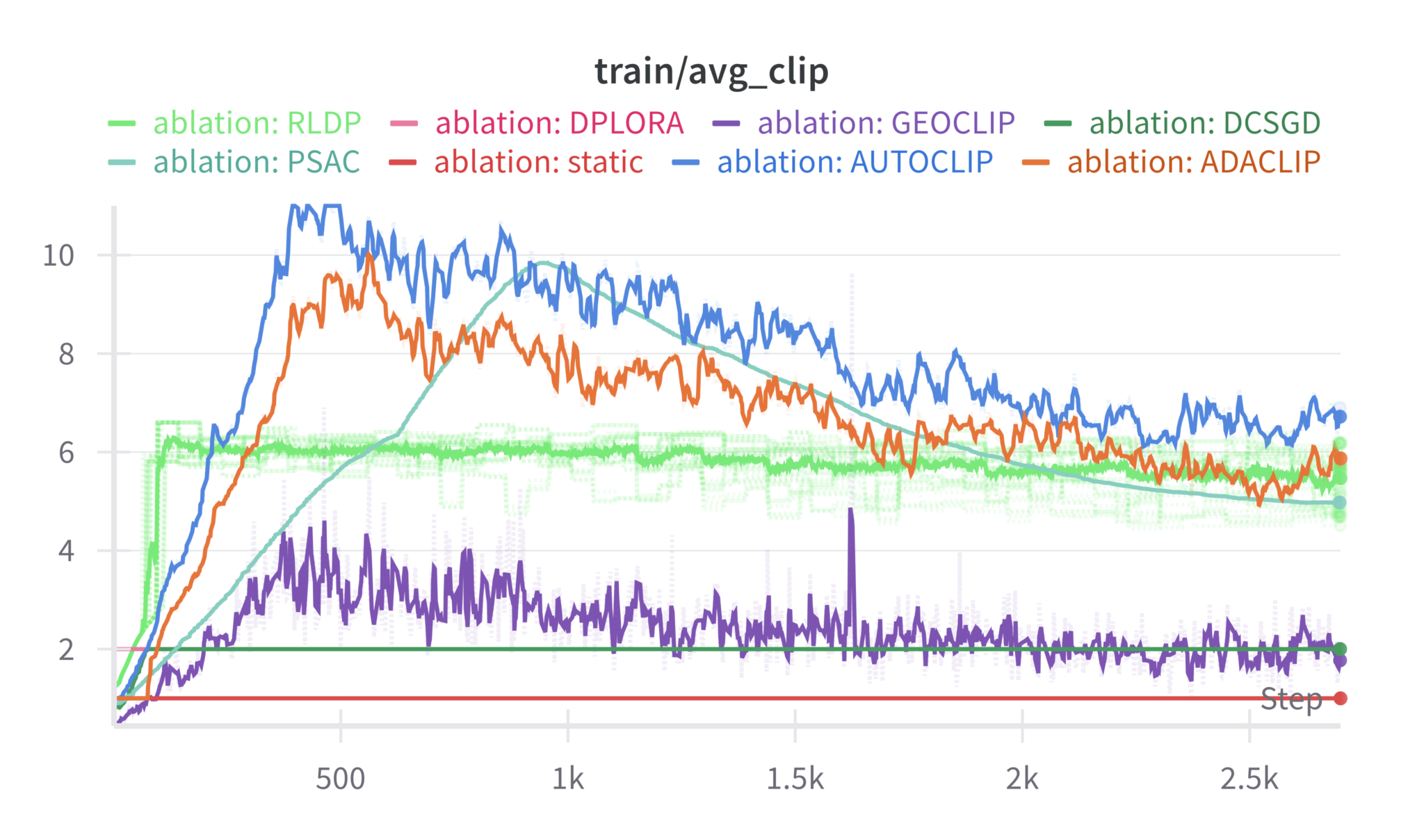}
    
    \label{fig:llama3b_clip_eps8}
  \end{subfigure}\hfill
  \begin{subfigure}[t]{0.45\textwidth}
    \centering
    \includegraphics[width=\textwidth]{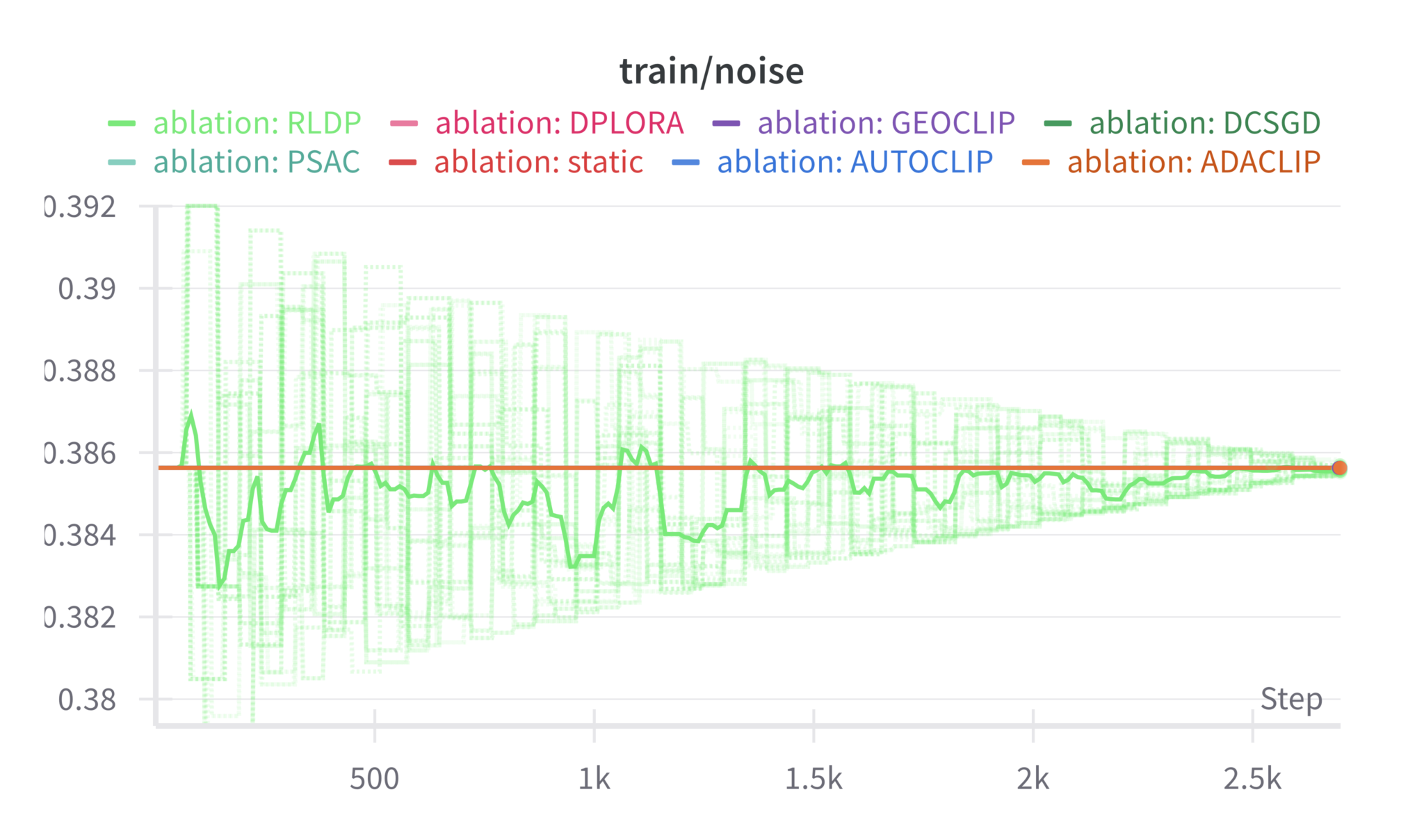}
    
    \label{fig:llama3b_noise_eps8}
  \end{subfigure}

  \caption{Training clip and noise history for Llama-3B under different DP budgets \(\varepsilon\). Rows (top to bottom) correspond to \(\varepsilon=0.5,\,2,\,4,\,5,\,8\). Left: training average CLIP; right: DP noise over steps.}
  \label{fig:llama3b_train_clip_noise}
\end{figure}

\begin{figure}[H]
  \centering
  % Row 1: ε = 0.5
  \begin{subfigure}[t]{0.45\textwidth}
    \centering
    \includegraphics[width=\textwidth]{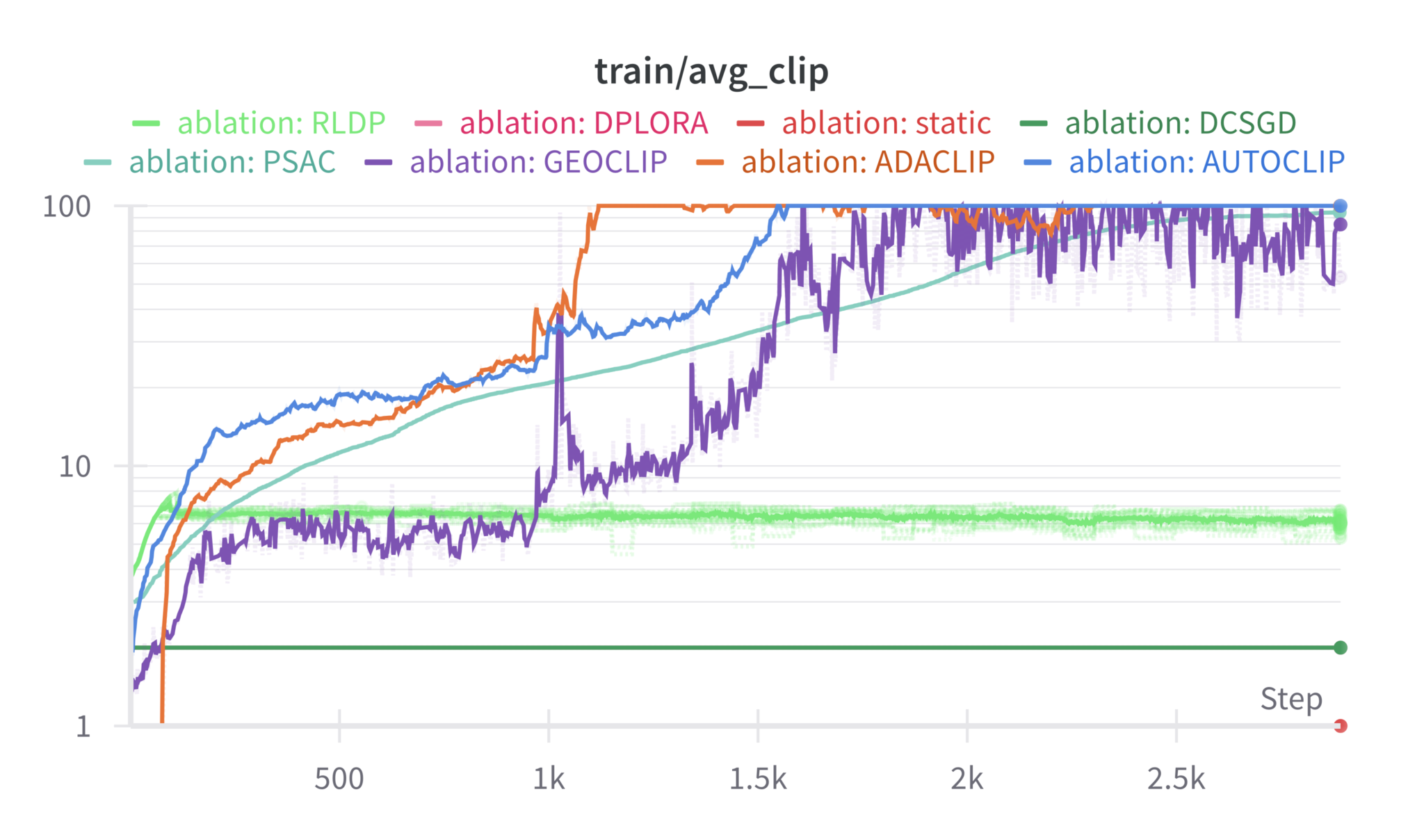}
    
    \label{fig:mistral_clip_eps05}
  \end{subfigure}\hfill
  \begin{subfigure}[t]{0.45\textwidth}
    \centering
    \includegraphics[width=\textwidth]{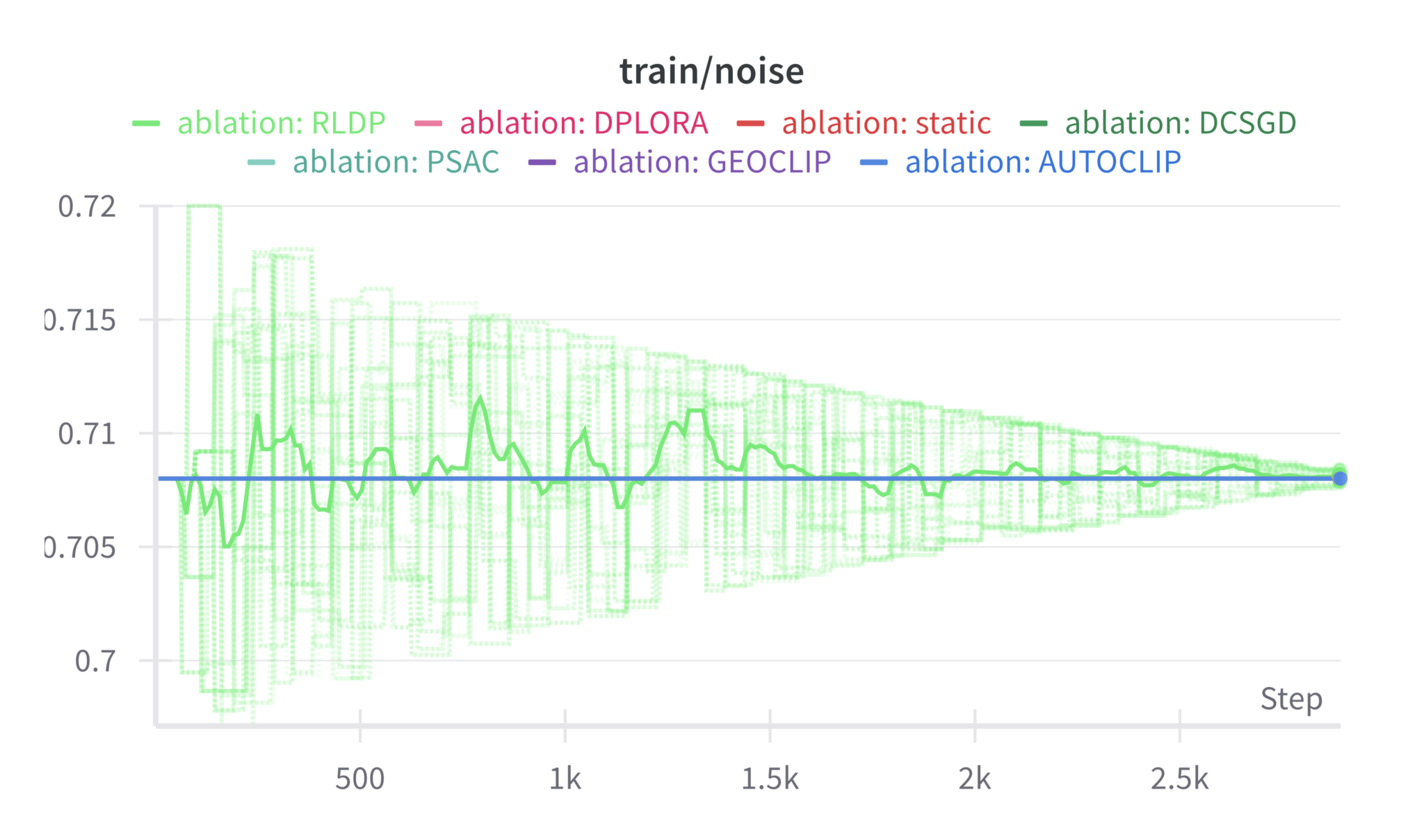}
    
    \label{fig:mistral_noise_eps05}
  \end{subfigure}

  % Row 2: ε = 2
  \begin{subfigure}[t]{0.45\textwidth}
    \centering
    \includegraphics[width=\textwidth]{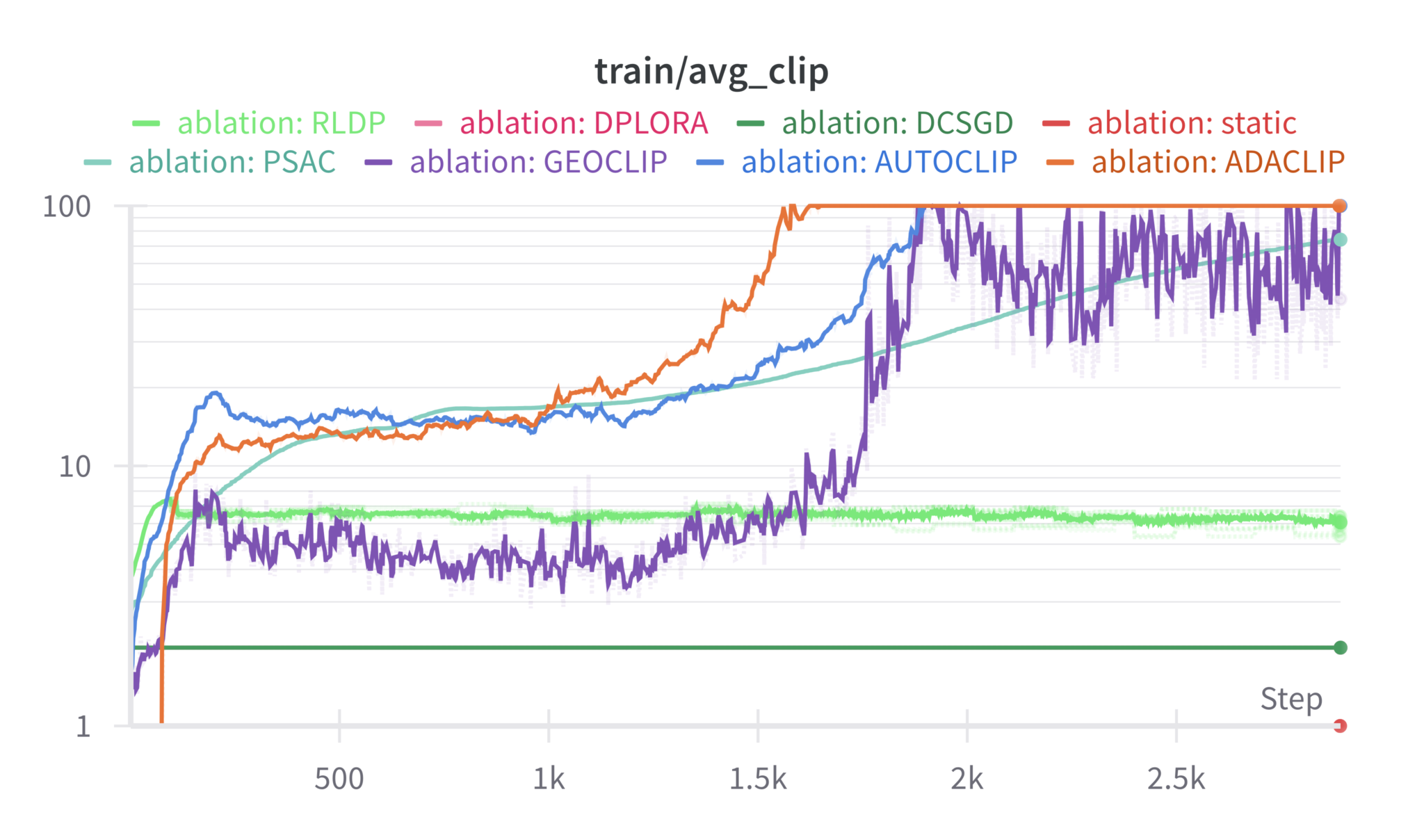}
    
    \label{fig:mistral_clip_eps2}
  \end{subfigure}\hfill
  \begin{subfigure}[t]{0.45\textwidth}
    \centering
    \includegraphics[width=\textwidth]{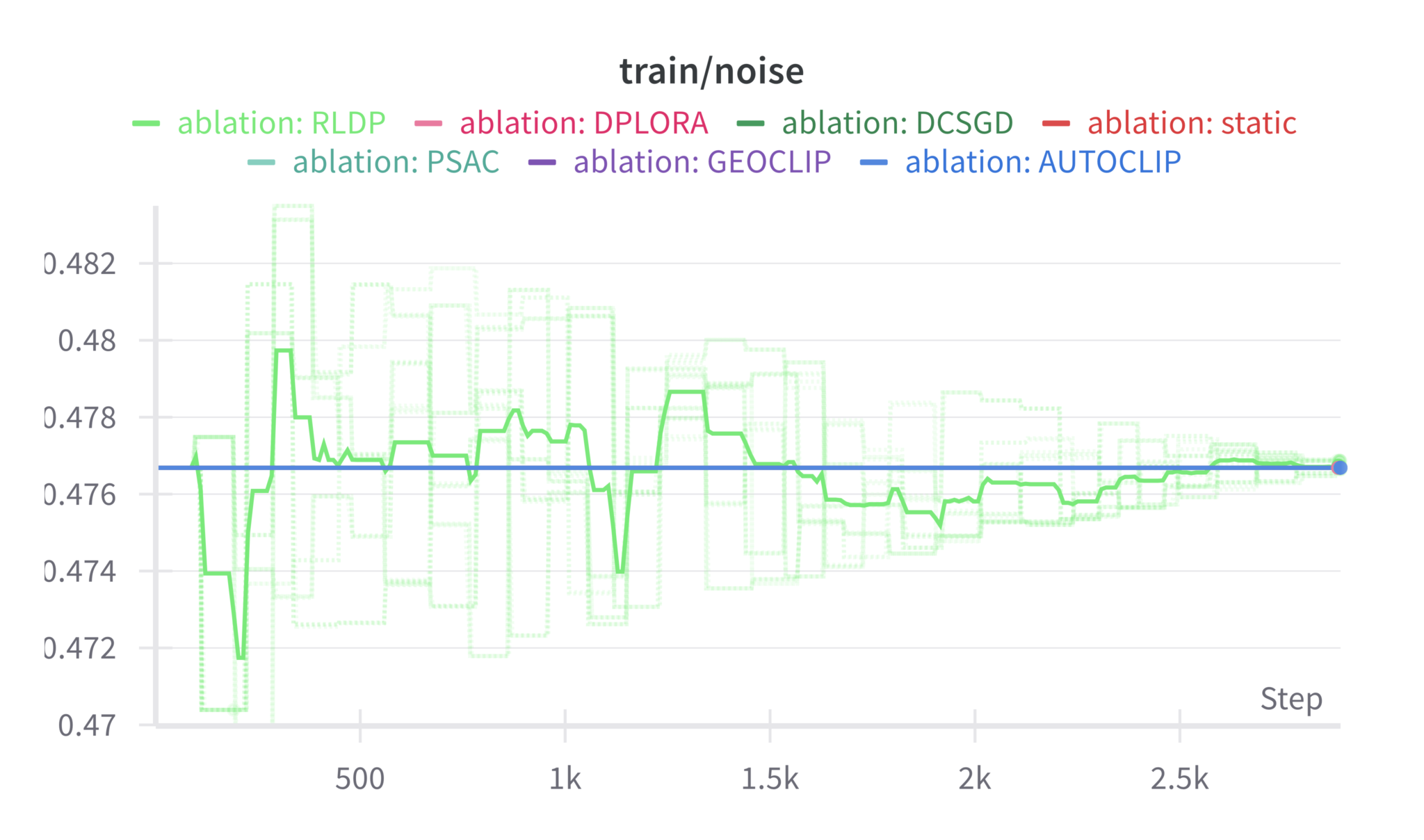}
    
    \label{fig:mistral_noise_eps2}
  \end{subfigure}

  % Row 3: ε = 4
  \begin{subfigure}[t]{0.45\textwidth}
    \centering
    \includegraphics[width=\textwidth]{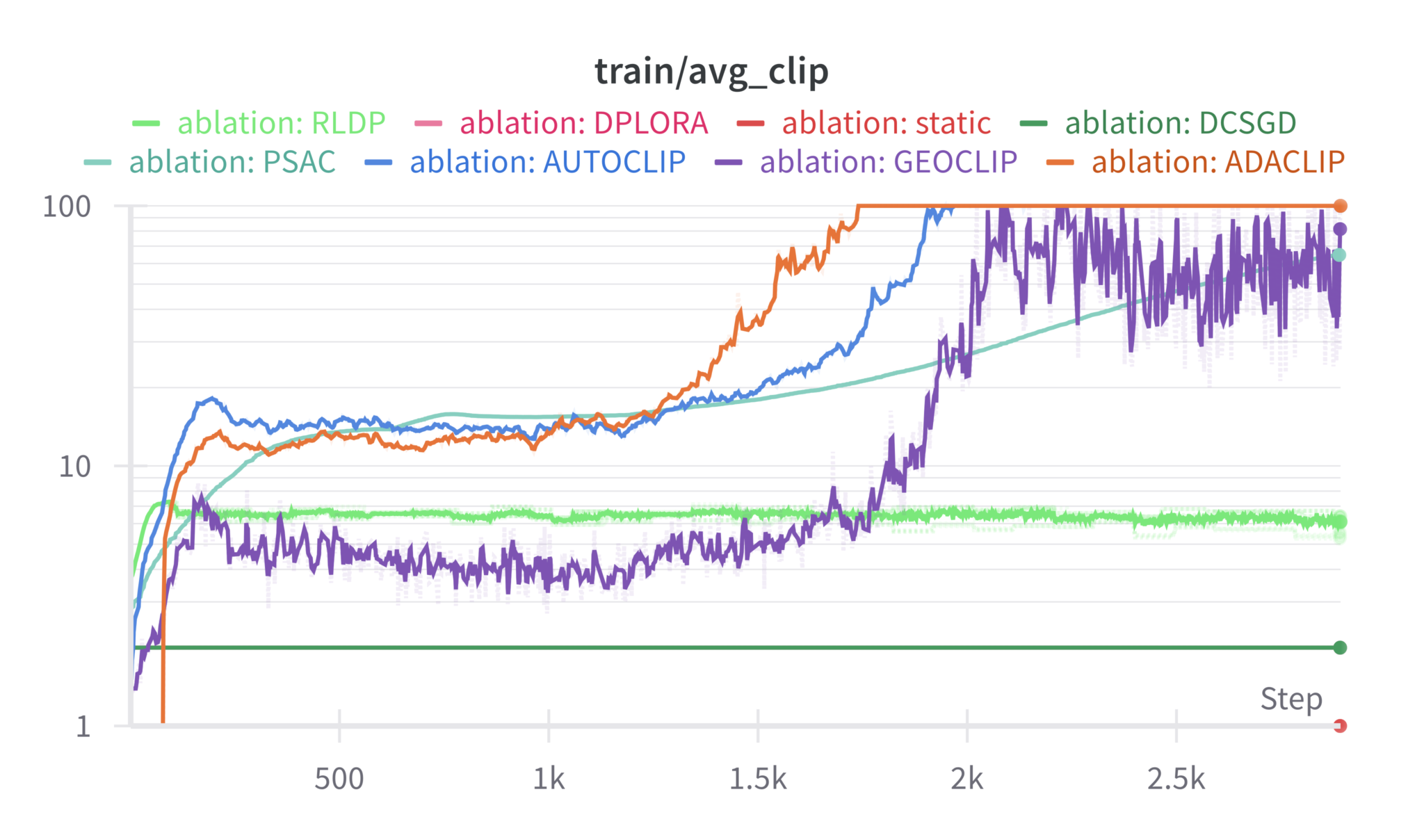}
    
    \label{fig:mistral_clip_eps4}
  \end{subfigure}\hfill
  \begin{subfigure}[t]{0.45\textwidth}
    \centering
    \includegraphics[width=\textwidth]{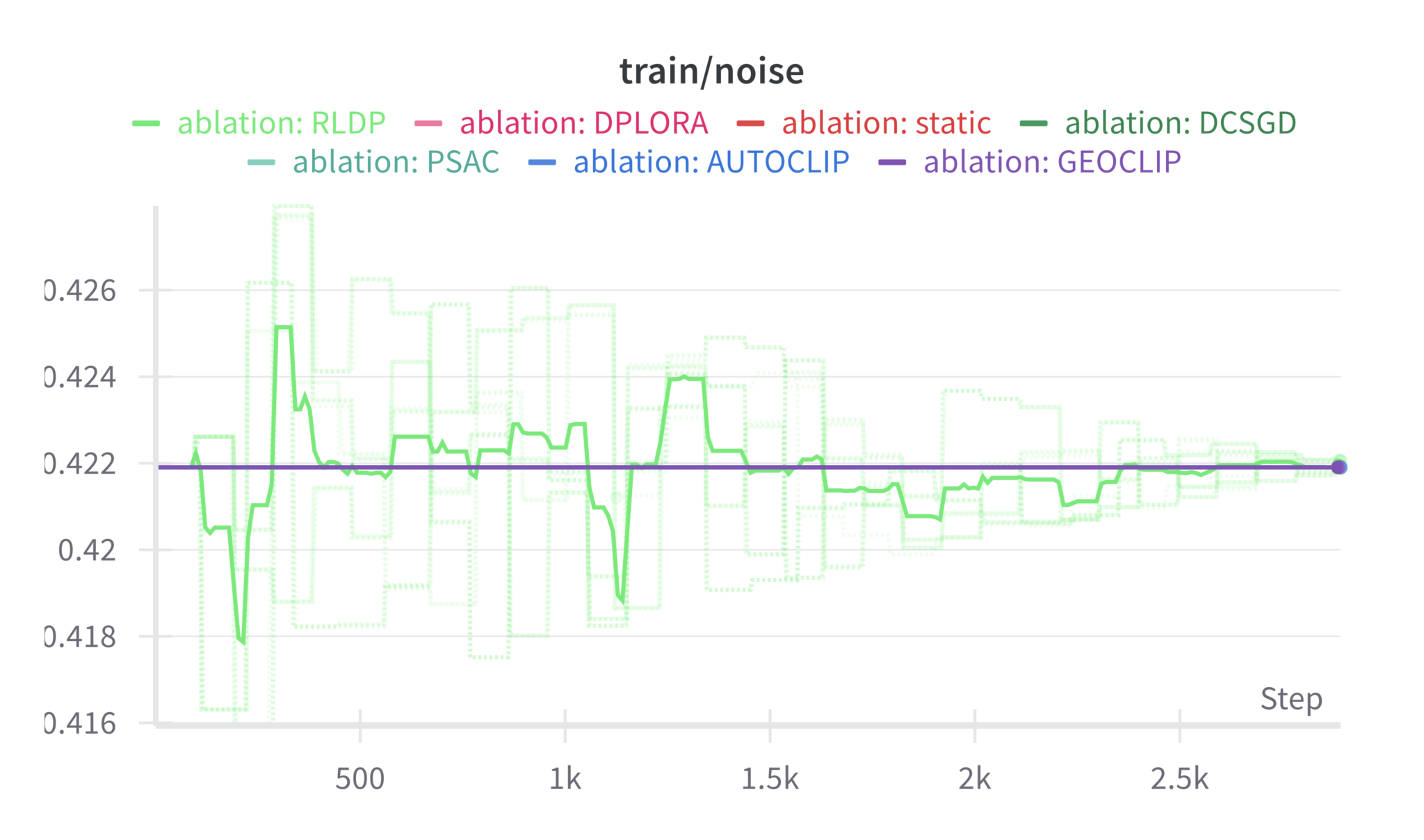}
    
    \label{fig:mistral_noise_eps4}
  \end{subfigure}

  % Row 4: ε = 5
  \begin{subfigure}[t]{0.45\textwidth}
    \centering
    \includegraphics[width=\textwidth]{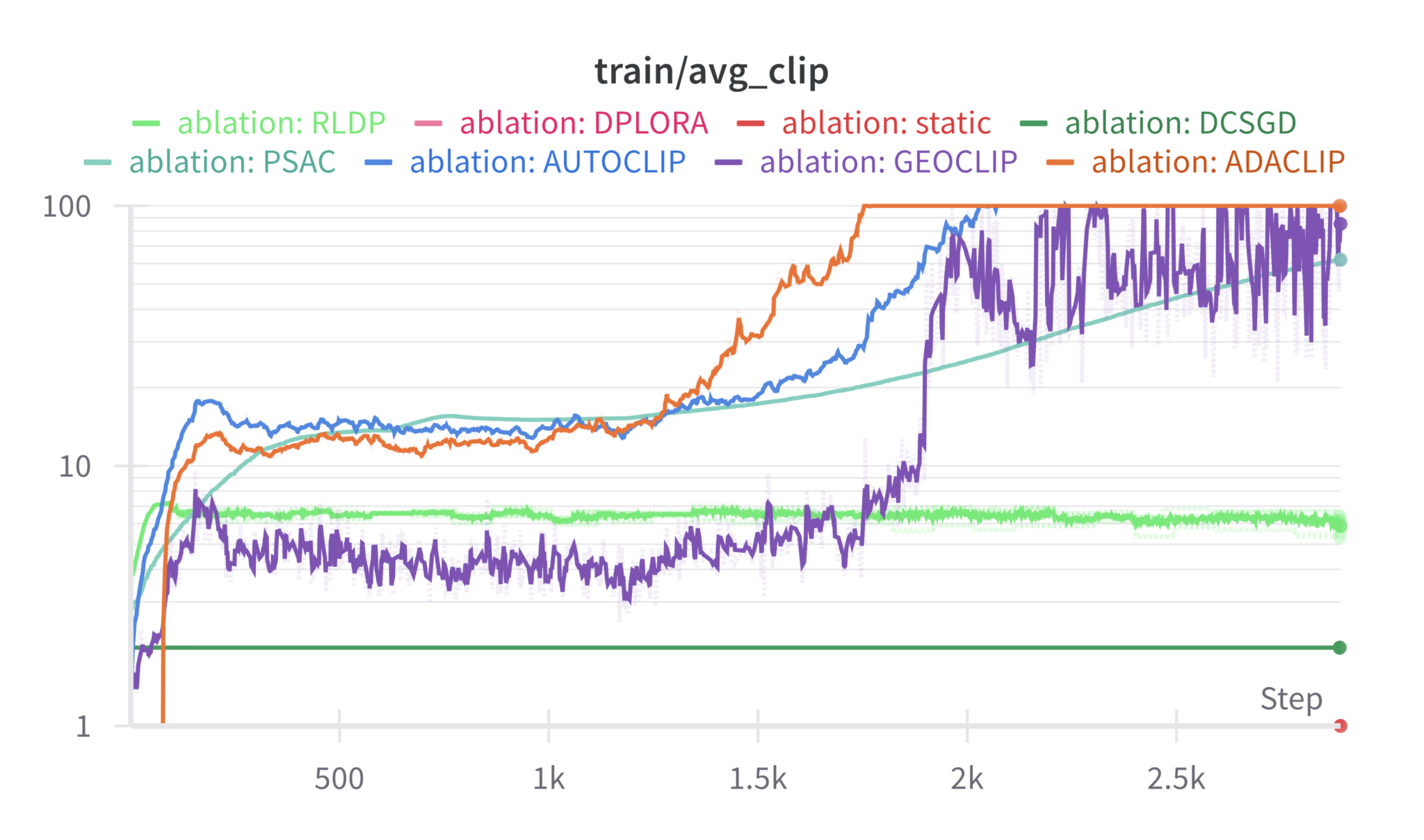}
    
    \label{fig:mistral_clip_eps5}
  \end{subfigure}\hfill
  \begin{subfigure}[t]{0.45\textwidth}
    \centering
    \includegraphics[width=\textwidth]{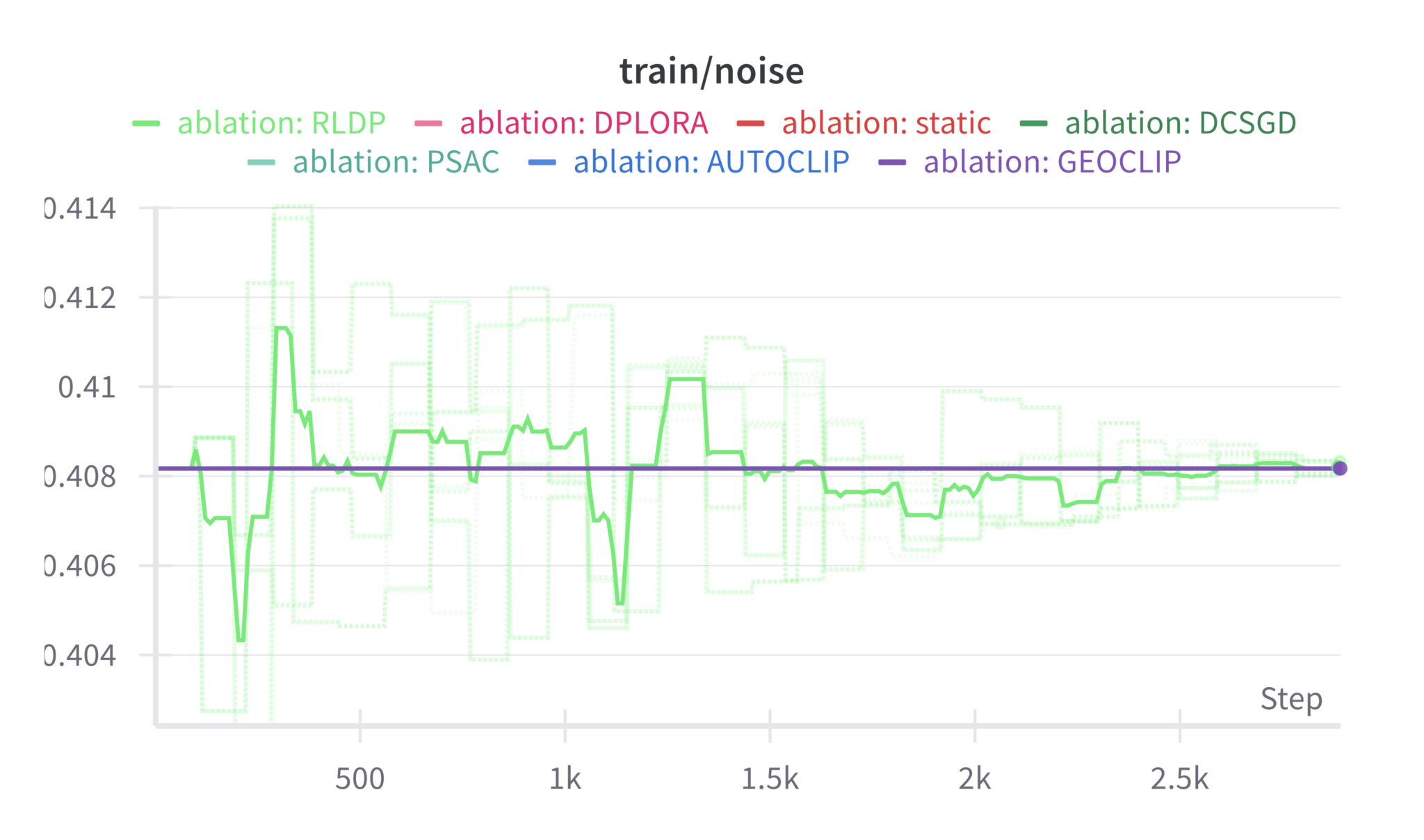}
    
    \label{fig:mistral_noise_eps5}
  \end{subfigure}

  % Row 5: ε = 8
  \begin{subfigure}[t]{0.45\textwidth}
    \centering
    \includegraphics[width=\textwidth]{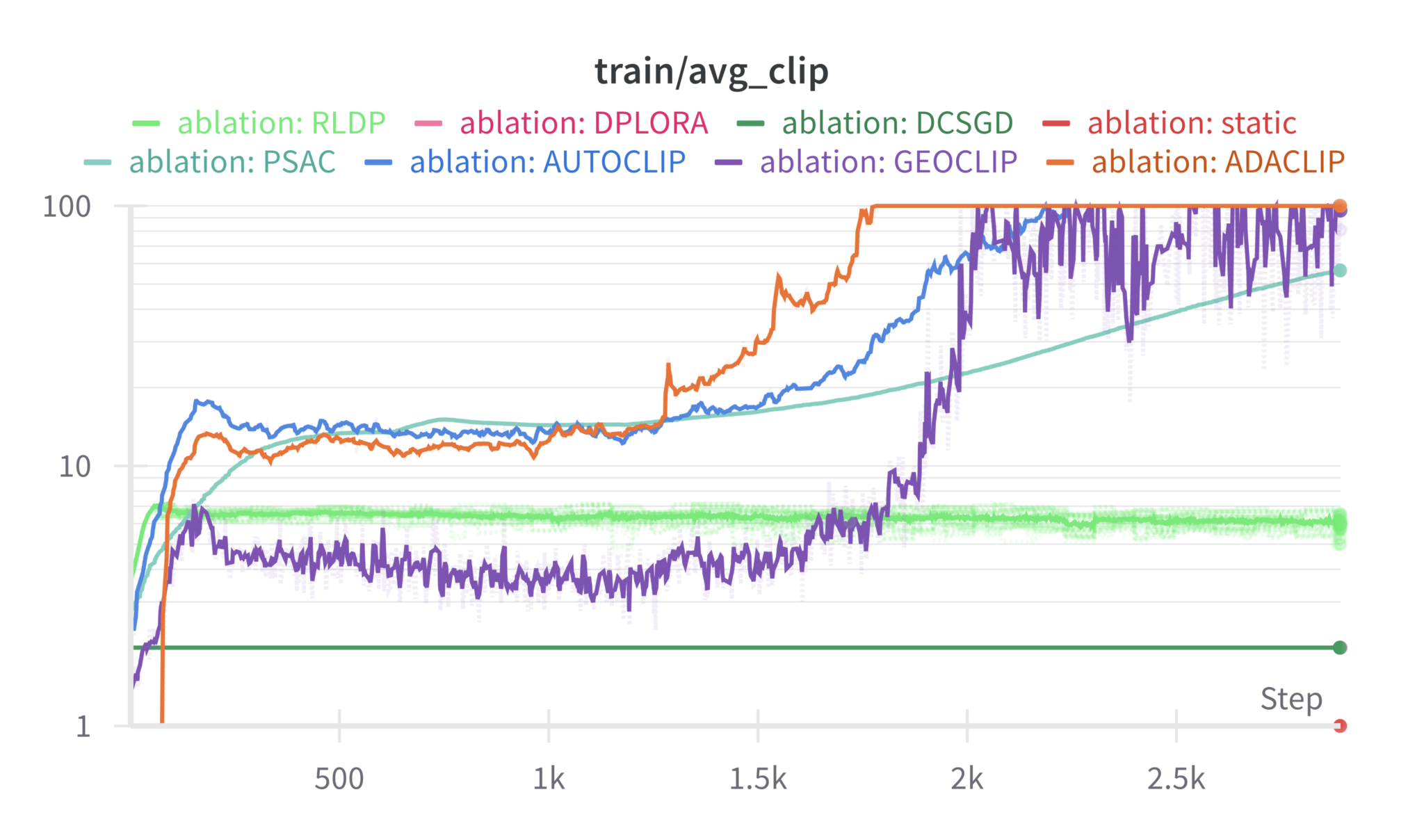}
    
    \label{fig:mistral_clip_eps8}
  \end{subfigure}\hfill
  \begin{subfigure}[t]{0.45\textwidth}
    \centering
    \includegraphics[width=\textwidth]{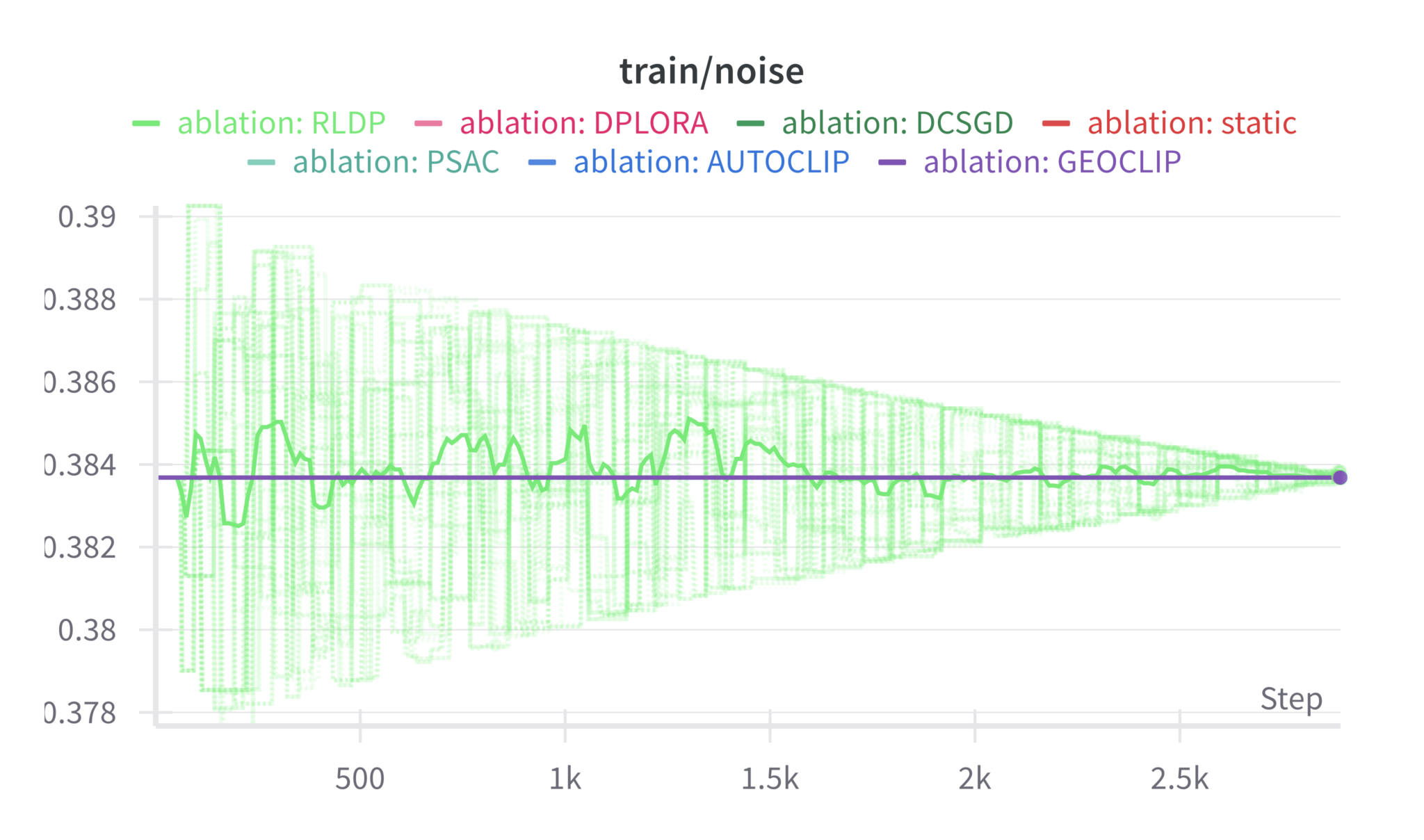}
    
    \label{fig:mistral_noise_eps8}
  \end{subfigure}

  \caption{Training clip and noise history for Mistral-7B under different DP budgets \(\varepsilon\). Rows (top to bottom) correspond to \(\varepsilon=0.5,\,2,\,4,\,5,\,8\). Left: training average CLIP (log scale); right: DP noise over steps.}
  \label{fig:mistral_train_clip_noise}
\end{figure}

\begin{table}[H]
\centering
\small
\begin{tabular}{c c | c c c | c c c}
\toprule
\multirow{2}{*}{$\epsilon$} & \multirow{2}{*}{Ablation}
  & \multicolumn{3}{c}{GPT2}
  & \multicolumn{3}{c}{Llama-1B} \\
\cmidrule(lr){3-5} \cmidrule(lr){6-8}
 & & 
   $\overline{\log p}$ & $\overline{\mathrm{ppl}}$ & AUC
   & $\overline{\log p}$ & $\overline{\mathrm{ppl}}$ & AUC \\
\midrule
\multirow{9}{*}{0.5}
 & AdaClip   & -3.725 & 177.967  & 0.362 & -2.439 & 4088.704 & 0.339 \\
 & AutoClip  & -3.723 & 345.157  & 0.366 & -2.255 & 1414.391 & 0.342 \\
 & DC-SGD    & -3.686 & 347.460  & 0.366 & -2.392 & 1189.060 & 0.349 \\
 & DP-LoRA   & -3.865 & 765.855  & 0.361 & -2.521 & 4421.569 & 0.348 \\
 & GeoClip   & -3.678 & 229.951  & 0.367 & -2.527 & 2000.288 & 0.356 \\
 & PSAC      & -3.705 & 421.654  & 0.362 & -2.086 & 1431.593 & 0.345 \\
 & RLDP$_H$  & -4.145 & 1222.501 & 0.359 & -1.617 &   75.549 & 0.332 \\
 & RLDP$_L$  & -4.634 & 2736.992 & 0.360 & -1.875 &   62.235 & 0.346 \\
 & static    & -3.740 & 334.693  & 0.366 & -2.329 & 1049.721 & 0.347 \\
\cmidrule(lr){1-8}
\multirow{9}{*}{2.0}
 & AdaClip   & -3.603 & 184.857  & 0.362 & -2.443 & 1755.701 & 0.347 \\
 & AutoClip  & -3.728 & 441.922  & 0.362 & -2.542 & 3670.828 & 0.339 \\
 & DC-SGD    & -3.668 & 515.813  & 0.366 & -2.313 & 3291.949 & 0.349 \\
 & DP-LoRA   & -3.995 & 1395.002 & 0.367 & -2.321 & 2158.130 & 0.350 \\
 & GeoClip   & -3.743 & 371.541  & 0.366 & -2.317 & 1342.738 & 0.346 \\
 & PSAC      & -3.726 & 666.864  & 0.365 & -2.167 & 2244.194 & 0.342 \\
 & RLDP$_H$  & -4.309 & 1653.623 & 0.357 & -1.916 &  152.102 & 0.351 \\
 & RLDP$_L$  & -4.754 & 2208.451 & 0.360 & -1.656 &  141.039 & 0.345 \\
 & static    & -3.715 & 452.052  & 0.365 & -2.369 & 2426.548 & 0.353 \\
\cmidrule(lr){1-8}
\multirow{9}{*}{4.0}
 & AdaClip   & -3.588 & 196.615  & 0.362 & -2.449 & 1389.552 & 0.342 \\
 & AutoClip  & -3.724 & 494.085  & 0.363 & -2.290 & 1972.282 & 0.342 \\
 & DC-SGD    & -3.676 & 596.282  & 0.367 & -2.369 & 4325.386 & 0.350 \\
 & DP-LoRA   & -3.991 & 1672.263 & 0.367 & -2.017 & 1013.825 & 0.347 \\
 & GeoClip   & -3.605 & 334.740  & 0.366 & -2.540 & 3123.871 & 0.341 \\
 & PSAC      & -3.754 & 801.909  & 0.366 & -2.124 & 1285.782 & 0.336 \\
 & RLDP$_H$  & -4.315 & 1619.728 & 0.357 & -1.912 &  181.264 & 0.341 \\
 & RLDP$_L$  & -4.747 & 2320.436 & 0.361 & -1.721 &   39.690 & 0.342 \\
 & static    & -3.717 & 523.678  & 0.366 & -2.277 & 2927.965 & 0.352 \\
\cmidrule(lr){1-8}
\multirow{9}{*}{5.0}
 & AdaClip   & -3.587 & 202.259  & 0.362 & -2.428 & 1325.179 & 0.342 \\
 & AutoClip  & -3.725 & 514.055  & 0.363 & -2.282 & 1571.894 & 0.342 \\
 & DC-SGD    & -3.681 & 619.809  & 0.367 & -2.391 & 4375.995 & 0.350 \\
 & DP-LoRA   & -3.984 & 1748.390 & 0.367 & -2.143 & 1155.657 & 0.344 \\
 & GeoClip   & -3.674 & 390.702  & 0.368 & -2.402 & 2264.428 & 0.341 \\
 & PSAC      & -3.764 & 842.630  & 0.367 & -2.126 & 1155.784 & 0.338 \\
 & RLDP$_H$  & -4.310 & 1619.863 & 0.356 & -1.883 &  177.313 & 0.342 \\
 & RLDP$_L$  & -4.782 & 2294.804 & 0.361 & -1.508 &   50.487 & 0.336 \\
 & static    & -3.719 & 549.016  & 0.366 & -2.271 & 3084.903 & 0.351 \\
\cmidrule(lr){1-8}
\multirow{9}{*}{8.0}
 & AdaClip   & -3.585 & 211.251  & 0.362 & -2.373 & 1484.133 & 0.339 \\
 & AutoClip  & -3.725 & 573.044  & 0.364 & -2.288 & 1805.285 & 0.340 \\
 & DC-SGD    & -3.694 & 677.146  & 0.367 & -2.315 & 3917.878 & 0.351 \\
 & DP-LoRA   & -3.968 & 1886.975 & 0.367 & -2.329 &  999.312 & 0.348 \\
 & GeoClip   & -3.617 & 427.472  & 0.369 & -2.098 & 2334.271 & 0.342 \\
 & PSAC      & -3.785 & 919.743  & 0.367 & -2.167 & 1090.743 & 0.337 \\
 & RLDP$_H$  & -5.164 & 5051.562 & 0.359 & -1.896 &  192.766 & 0.333 \\
 & RLDP$_L$  & -4.732 & 2404.148 & 0.354 & -1.682 &   41.661 & 0.339 \\
 & static    & -3.726 & 599.004  & 0.366 & -2.273 & 3188.300 & 0.350 \\
\bottomrule
\end{tabular}
\caption{Membership inference attack results for GPT2 and Llama-1B across various privacy budgets. Metrics are average log-prob ($\overline{\log p}$), average perplexity ($\overline{\mathrm{ppl}}$), and AUC.}
\label{tab:mi_models_gpt2_llama1b}
\end{table}

\begin{table}[H]
\centering
\small
\begin{tabular}{c c | c c c | c c c}
\toprule
\multirow{2}{*}{$\epsilon$} & \multirow{2}{*}{Ablation}
  & \multicolumn{3}{c}{Llama-3B}
  & \multicolumn{3}{c}{Mistral-7B} \\
\cmidrule(lr){3-5} \cmidrule(lr){6-8}
 & & 
   $\overline{\log p}$ & $\overline{\mathrm{ppl}}$ & AUC
   & $\overline{\log p}$ & $\overline{\mathrm{ppl}}$ & AUC \\
\midrule
\multirow{9}{*}{0.5}
 & AdaClip   & -4.082 & 166.333 & 0.355 & -9.944 & 25562.470 & 0.349 \\
 & AutoClip  & -3.741 & 132.428 & 0.353 & -11.171 & 72895.679 & 0.353 \\
 & DC-SGD    & -3.873 & 203.566 & 0.352 & -4.618 & 455.068 & 0.329 \\
 & DP-LoRA   & -4.179 & 141.383 & 0.369 & -7.469 & 2576.410 & 0.336 \\
 & GeoClip   & -3.895 & 127.490 & 0.354 & -10.483 & 36496.315 & 0.369 \\
 & PSAC      & -4.243 & 195.275 & 0.354 & -10.926 & 57909.194 & 0.436 \\
 & RLDP$_H$  & -2.007 &  78.123 & 0.360 & -4.334 & 140.635 & 0.340 \\
 & RLDP$_L$  & -2.452 &  70.522 & 0.353 & -4.562 & 442.647 & 0.651 \\
 & static    & -3.737 & 150.751 & 0.350 & -8.032 & 3966.452 & 0.327 \\
\cmidrule(lr){1-8}
\multirow{9}{*}{2.0}
 & AdaClip   & -3.700 & 125.955 & 0.351 & -10.480 & 37840.196 & 0.349 \\
 & AutoClip  & -3.846 & 148.348 & 0.349 & -11.204 & 74038.084 & 0.434 \\
 & DC-SGD    & -3.965 & 225.392 & 0.351 & -5.271 & 3408.979 & 0.333 \\
 & DP-LoRA   & -4.345 & 196.732 & 0.400 & -7.577 & 5630.890 & 0.341 \\
 & GeoClip   & -3.982 & 134.546 & 0.359 & -9.180 & 10435.432 & 0.351 \\
 & PSAC      & -4.169 & 176.368 & 0.365 & -10.401 & 34171.978 & 0.351 \\
 & RLDP$_H$  & -1.783 &  80.930 & 0.358 & -3.653 &   49.236 & 0.336 \\
 & RLDP$_L$  & -3.065 & 145.865 & 0.353 & -2.989 &   40.193 & 0.335 \\
 & static    & -3.897 & 188.090 & 0.348 & -5.421 & 3069.026 & 0.346 \\
\cmidrule(lr){1-8}
\multirow{9}{*}{4.0}
 & AdaClip   & -3.713 & 155.198 & 0.355 & -9.719 & 18982.047 & 0.345 \\
 & AutoClip  & -4.028 & 151.554 & 0.349 & -11.070 & 64704.703 & 0.501 \\
 & DC-SGD    & -4.032 & 178.296 & 0.352 & -4.063 &  401.212 & 0.340 \\
 & DP-LoRA   & -4.265 & 216.376 & 0.380 & -7.281 &  4390.709 & 0.339 \\
 & GeoClip   & -3.999 & 133.570 & 0.350 & -11.273 & 79470.330 & 0.406 \\
 & PSAC      & -3.962 & 157.966 & 0.351 & -10.184 & 28424.074 & 0.344 \\
 & RLDP$_H$  & -3.364 &  90.563 & 0.350 & -3.287 &   33.552 & 0.337 \\
 & RLDP$_L$  & -2.704 &  84.617 & 0.355 & -6.193 &  991.411 & 0.335 \\
 & static    & -3.893 & 176.501 & 0.349 & -4.674 &  222.343 & 0.333 \\
\cmidrule(lr){1-8}
\multirow{9}{*}{5.0}
 & AdaClip   & -3.804 & 148.348 & 0.354 & -10.134 & 27670.259 & 0.345 \\
 & AutoClip  & -4.005 & 139.193 & 0.350 & -11.194 & 73119.138 & 0.528 \\
 & DC-SGD    & -4.049 & 171.989 & 0.352 & -3.883 &  179.770 & 0.336 \\
 & DP-LoRA   & -4.015 & 148.528 & 0.376 & -8.633 & 23898.032 & 0.350 \\
 & GeoClip   & -3.824 & 112.559 & 0.345 & -8.109 &  5296.685 & 0.346 \\
 & PSAC      & -3.848 & 135.209 & 0.352 & -10.633 & 46980.193 & 0.360 \\
 & RLDP$_H$  & -1.871 &  78.320 & 0.356 & -3.847 &   55.084 & 0.338 \\
 & RLDP$_L$  & -3.664 & 207.148 & 0.354 & -6.118 & 1119.646 & 0.344 \\
 & static    & -3.880 & 173.426 & 0.350 & -3.943 &  282.537 & 0.340 \\
\cmidrule(lr){1-8}
\multirow{9}{*}{8.0}
 & AdaClip   & -4.031 & 155.227 & 0.353 & -10.092 & 26182.650 & 0.352 \\
 & AutoClip  & -4.027 & 140.882 & 0.354 & -10.849 & 53628.881 & 0.640 \\
 & DC-SGD    & -4.010 & 165.974 & 0.352 & -3.749 &  137.045 & 0.336 \\
 & DP-LoRA   & -4.208 & 185.429 & 0.386 & -7.543 &  4874.882 & 0.336 \\
 & GeoClip   & -3.883 & 137.819 & 0.348 & -9.114 & 10613.607 & 0.345 \\
 & PSAC      & -3.766 & 130.588 & 0.355 & -9.793 & 21486.278 & 0.351 \\
 & RLDP$_H$  & -1.951 &  87.648 & 0.356 & -3.274 &   32.915 & 0.344 \\
 & RLDP$_L$  & -3.819 & 205.403 & 0.353 & -3.138 &  113.213 & 0.349 \\
 & static    & -3.982 & 176.223 & 0.352 & -3.773 &  123.711 & 0.333 \\
\bottomrule
\end{tabular}
\caption{Membership inference attack results for Llama-3B and Mistral-7B across various privacy budgets. Metrics are average log-prob ($\overline{\log p}$), average perplexity ($\overline{\mathrm{ppl}}$), and AUC.}
\label{tab:mi_models_llama3b_mistral7b}
\end{table}

\begin{table}[H]
\centering
\small
\begin{tabular}{c c | cc | cc | cc | cc}
\toprule
\multirow{2}{*}{$\epsilon$} & \multirow{2}{*}{Ablation}
  & \multicolumn{2}{c}{GPT2}
  & \multicolumn{2}{c}{Llama-1B}
  & \multicolumn{2}{c}{Llama-3B}
  & \multicolumn{2}{c}{Mistral-7B} \\
\cmidrule(lr){3-4} \cmidrule(lr){5-6} \cmidrule(lr){7-8} \cmidrule(lr){9-10}
 & & Jaccard$_1$ & $n_{\mathrm{valid}}$
     & Jaccard$_1$ & $n_{\mathrm{valid}}$
     & Jaccard$_1$ & $n_{\mathrm{valid}}$
     & Jaccard$_1$ & $n_{\mathrm{valid}}$ \\
\midrule
\multirow{9}{*}{0.5}
 & AdaClip   & 0.139 & 3983 & 0.152 & 4000 & 0.149 & 3259 & 0.178 & 4000 \\
 & AutoClip  & 0.149 & 3918 & 0.156 & 4000 & 0.155 & 2702 & 0.174 & 4000 \\
 & DC-SGD    & 0.149 & 3909 & 0.145 & 3373 & 0.169 & 3720 & 0.060 & 4000 \\
 & DP-LoRA   & 0.147 & 3793 & 0.155 & 3965 & 0.169 & 3685 & 0.128 & 4000 \\
 & GeoClip   & 0.139 & 3924 & 0.149 & 3947 & 0.167 & 2662 & 0.166 & 4000 \\
 & PSAC      & 0.152 & 3943 & 0.158 & 4000 & 0.144 & 3355 & 0.167 & 4000 \\
 & RLDP$_H$  & 0.119 & 3949 & 0.150 & 4000 & 0.142 & 1651 & 0.012 & 4000 \\
 & RLDP$_L$  & 0.165 & 3958 & 0.143 & 4000 & 0.135 & 1889 & 0.150 & 4000 \\
 & static    & 0.142 & 3887 & 0.145 & 3338 & 0.169 & 3717 & 0.064 & 4000 \\
\cmidrule(lr){1-10}
\multirow{9}{*}{2.0}
 & AdaClip   & 0.167 & 4000 & 0.152 & 2477 & 0.158 & 3360 & 0.181 & 4000 \\
 & AutoClip  & 0.142 & 3859 & 0.156 & 3037 & 0.167 & 3459 & 0.167 & 4000 \\
 & DC-SGD    & 0.144 & 3908 & 0.160 & 3950 & 0.169 & 3228 & 0.115 & 4000 \\
 & DP-LoRA   & 0.136 & 3724 & 0.155 & 2418 & 0.169 & 4000 & 0.122 & 3967 \\
 & GeoClip   & 0.142 & 3939 & 0.164 & 3752 & 0.162 & 3617 & 0.177 & 4000 \\
 & PSAC      & 0.145 & 3891 & 0.160 & 4000 & 0.165 & 3504 & 0.158 & 4000 \\
 & RLDP$_H$  & 0.111 & 3976 & 0.148 & 4000 & 0.152 & 2721 & 0.120 & 4000 \\
 & RLDP$_L$  & 0.140 & 3688 & 0.158 & 4000 & 0.173 &   47 & 0.143 & 3999 \\
 & static    & 0.139 & 3873 & 0.157 & 3380 & 0.169 & 4000 & 0.117 & 3997 \\
\cmidrule(lr){1-10}
\multirow{9}{*}{4.0}
 & AdaClip   & 0.168 & 4000 & 0.159 & 2046 & 0.164 & 3730 & 0.179 & 4000 \\
 & AutoClip  & 0.141 & 3838 & 0.159 & 3007 & 0.163 & 3840 & 0.146 & 4000 \\
 & DC-SGD    & 0.142 & 3865 & 0.160 & 3896 & 0.169 & 2893 & 0.128 & 3998 \\
 & DP-LoRA   & 0.136 & 3667 & 0.154 & 2705 & 0.168 & 4000 & 0.113 & 3996 \\
 & GeoClip   & 0.139 & 3680 & 0.151 & 3699 & 0.167 & 3922 & 0.180 & 4000 \\
 & PSAC      & 0.143 & 3860 & 0.162 & 3804 & 0.164 & 3311 & 0.072 & 4000 \\
 & RLDP$_H$  & 0.108 & 3980 & 0.153 & 4000 & 0.158 & 4000 & 0.128 & 4000 \\
 & RLDP$_L$  & 0.140 & 3742 & 0.167 & 4000 & 0.169 & 4000 & 0.106 & 3994 \\
 & static    & 0.139 & 3866 & 0.160 & 3845 & 0.169 & 4000 & 0.138 & 3984 \\
\cmidrule(lr){1-10}
\multirow{9}{*}{5.0}
 & AdaClip   & 0.169 & 4000 & 0.157 & 2055 & 0.166 & 3740 & 0.173 & 4000 \\
 & AutoClip  & 0.140 & 3825 & 0.160 & 3256 & 0.164 & 3870 & 0.144 & 4000 \\
 & DC-SGD    & 0.141 & 3865 & 0.160 & 3814 & 0.169 & 2593 & 0.130 & 3997 \\
 & DP-LoRA   & 0.137 & 3656 & 0.154 & 2498 & 0.167 & 4000 & 0.118 & 3996 \\
 & GeoClip   & 0.137 & 3707 & 0.157 & 3853 & 0.168 & 3906 & 0.182 & 4000 \\
 & PSAC      & 0.143 & 3836 & 0.163 & 3664 & 0.164 & 3302 & 0.104 & 4000 \\
 & RLDP$_H$  & 0.107 & 3989 & 0.160 & 4000 & 0.160 & 2751 & 0.108 & 4000 \\
 & RLDP$_L$  & 0.137 & 3741 & 0.168 & 4000 & 0.164 &  150 & 0.120 & 3991 \\
 & static    & 0.140 & 3862 & 0.160 & 3775 & 0.169 & 3593 & 0.127 & 3999 \\
\cmidrule(lr){1-10}
\multirow{9}{*}{8.0}
 & AdaClip   & 0.169 & 4000 & 0.157 & 2320 & 0.166 & 3746 & 0.169 & 4000 \\
 & AutoClip  & 0.141 & 3793 & 0.163 & 3660 & 0.167 & 4000 & 0.161 & 3999 \\
 & DC-SGD    & 0.141 & 3840 & 0.161 & 3715 & 0.169 & 3142 & 0.133 & 3995 \\
 & DP-LoRA   & 0.139 & 3641 & 0.154 & 2642 & 0.169 & 4000 & 0.113 & 3963 \\
 & GeoClip   & 0.135 & 3728 & 0.164 & 3525 & 0.167 & 3696 & 0.173 & 4000 \\
 & PSAC      & 0.143 & 3807 & 0.166 & 3573 & 0.165 & 3503 & 0.048 & 4000 \\
 & RLDP$_H$  & 0.138 & 3950 & 0.160 & 4000 & 0.160 & 2916 & 0.126 & 4000 \\
 & RLDP$_L$  & 0.135 & 3768 & 0.167 & 4000 & 0.201 &  348 & 0.122 & 4000 \\
 & static    & 0.139 & 3844 & 0.159 & 3748 & 0.169 & 3510 & 0.143 & 3987 \\
\bottomrule
\end{tabular}
\caption{Canary memorization attack results for GPT2 (1-gram Jaccard similarity and number of valid predictions), Llama-1B, Llama-3B, and Mistral-7B across various privacy budgets.}
\label{tab:memorization_jaccard1}
\end{table}

\begin{table}[H]
\centering
\small
\begin{tabular}{c c | cc | cc | cc | cc}
\toprule
\multirow{2}{*}{$\epsilon$} & \multirow{2}{*}{Ablation}
  & \multicolumn{2}{c}{GPT2}
  & \multicolumn{2}{c}{Llama-1B}
  & \multicolumn{2}{c}{Llama-3B}
  & \multicolumn{2}{c}{Mistral-7B} \\
\cmidrule(lr){3-4} \cmidrule(lr){5-6} \cmidrule(lr){7-8} \cmidrule(lr){9-10}
 & & Jaccard$_{2}$ & $n_{\mathrm{valid}}$
     & Jaccard$_{2}$ & $n_{\mathrm{valid}}$
     & Jaccard$_{2}$ & $n_{\mathrm{valid}}$
     & Jaccard$_{2}$ & $n_{\mathrm{valid}}$ \\
\midrule
\multirow{9}{*}{0.5}
 & AdaClip   & 0.003 & 3983 & 0.004 & 4000 & 0.004 & 3259 & 0.004 & 4000 \\
 & AutoClip  & 0.004 & 3918 & 0.004 & 4000 & 0.004 & 2702 & 0.004 & 4000 \\
 & DC-SGD    & 0.004 & 3909 & 0.004 & 3373 & 0.004 & 3720 & 0.001 & 4000 \\
 & DP-LoRA   & 0.003 & 3793 & 0.003 & 3965 & 0.004 & 3685 & 0.003 & 4000 \\
 & GeoClip   & 0.003 & 3924 & 0.004 & 3947 & 0.004 & 2662 & 0.004 & 4000 \\
 & PSAC      & 0.004 & 3943 & 0.004 & 4000 & 0.004 & 3355 & 0.004 & 4000 \\
 & RLDP$_H$  & 0.002 & 3949 & 0.004 & 4000 & 0.004 & 1651 & 0.000 & 4000 \\
 & RLDP$_L$  & 0.004 & 3958 & 0.004 & 4000 & 0.004 & 1889 & 0.003 & 4000 \\
 & static    & 0.004 & 3887 & 0.004 & 3338 & 0.004 & 3717 & 0.001 & 4000 \\
\cmidrule(lr){1-10}
\multirow{9}{*}{2.0}
 & AdaClip   & 0.004 & 4000 & 0.004 & 2477 & 0.004 & 3360 & 0.004 & 4000 \\
 & AutoClip  & 0.004 & 3859 & 0.004 & 3037 & 0.004 & 3459 & 0.004 & 4000 \\
 & DC-SGD    & 0.003 & 3908 & 0.004 & 3950 & 0.004 & 3228 & 0.002 & 4000 \\
 & DP-LoRA   & 0.003 & 3724 & 0.004 & 2418 & 0.004 & 4000 & 0.003 & 3967 \\
 & GeoClip   & 0.004 & 3939 & 0.004 & 3752 & 0.004 & 3617 & 0.004 & 4000 \\
 & PSAC      & 0.003 & 3891 & 0.004 & 4000 & 0.004 & 3504 & 0.004 & 4000 \\
 & RLDP$_H$  & 0.002 & 3976 & 0.004 & 4000 & 0.004 & 2721 & 0.001 & 4000 \\
 & RLDP$_L$  & 0.003 & 3688 & 0.004 & 4000 & 0.005 &   47 & 0.003 & 3999 \\
 & static    & 0.003 & 3873 & 0.004 & 3380 & 0.004 & 4000 & 0.002 & 3997 \\
\cmidrule(lr){1-10}
\multirow{9}{*}{4.0}
 & AdaClip   & 0.004 & 4000 & 0.004 & 2046 & 0.004 & 3730 & 0.004 & 4000 \\
 & AutoClip  & 0.004 & 3838 & 0.004 & 3007 & 0.004 & 3840 & 0.003 & 4000 \\
 & DC-SGD    & 0.003 & 3865 & 0.004 & 3896 & 0.004 & 2893 & 0.003 & 3998 \\
 & DP-LoRA   & 0.003 & 3667 & 0.004 & 2705 & 0.004 & 4000 & 0.002 & 3996 \\
 & GeoClip   & 0.003 & 3680 & 0.004 & 3699 & 0.004 & 3922 & 0.005 & 4000 \\
 & PSAC      & 0.003 & 3860 & 0.004 & 3804 & 0.004 & 3311 & 0.001 & 4000 \\
 & RLDP$_H$  & 0.002 & 3980 & 0.004 & 4000 & 0.004 & 4000 & 0.002 & 4000 \\
 & RLDP$_L$  & 0.003 & 3742 & 0.004 & 4000 & 0.004 & 4000 & 0.002 & 3994 \\
 & static    & 0.003 & 3866 & 0.004 & 3845 & 0.004 & 4000 & 0.003 & 3984 \\
\cmidrule(lr){1-10}
\multirow{9}{*}{5.0}
 & AdaClip   & 0.004 & 4000 & 0.004 & 2055 & 0.004 & 3740 & 0.004 & 4000 \\
 & AutoClip  & 0.004 & 3825 & 0.004 & 3256 & 0.004 & 3870 & 0.004 & 4000 \\
 & DC-SGD    & 0.003 & 3865 & 0.004 & 3814 & 0.004 & 2593 & 0.003 & 3997 \\
 & DP-LoRA   & 0.003 & 3656 & 0.004 & 2498 & 0.004 & 4000 & 0.003 & 3996 \\
 & GeoClip   & 0.004 & 3707 & 0.004 & 3853 & 0.004 & 3906 & 0.004 & 4000 \\
 & PSAC      & 0.003 & 3836 & 0.004 & 3664 & 0.004 & 3302 & 0.002 & 4000 \\
 & RLDP$_H$  & 0.002 & 3989 & 0.004 & 4000 & 0.004 & 2751 & 0.002 & 4000 \\
 & RLDP$_L$  & 0.003 & 3741 & 0.004 & 4000 & 0.005 &  150 & 0.003 & 3991 \\
 & static    & 0.003 & 3862 & 0.004 & 3775 & 0.004 & 3593 & 0.003 & 3999 \\
\cmidrule(lr){1-10}
\multirow{9}{*}{8.0}
 & AdaClip   & 0.004 & 4000 & 0.004 & 2320 & 0.004 & 3746 & 0.004 & 4000 \\
 & AutoClip  & 0.004 & 3793 & 0.004 & 3660 & 0.004 & 4000 & 0.004 & 3999 \\
 & DC-SGD    & 0.003 & 3840 & 0.004 & 3715 & 0.004 & 3142 & 0.003 & 3995 \\
 & DP-LoRA   & 0.003 & 3641 & 0.004 & 2642 & 0.004 & 4000 & 0.002 & 3963 \\
 & GeoClip   & 0.003 & 3728 & 0.004 & 3525 & 0.004 & 3696 & 0.004 & 4000 \\
 & PSAC      & 0.003 & 3807 & 0.004 & 3573 & 0.004 & 3503 & 0.001 & 4000 \\
 & RLDP$_H$  & 0.003 & 3950 & 0.004 & 4000 & 0.004 & 2916 & 0.003 & 4000 \\
 & RLDP$_L$  & 0.003 & 3768 & 0.004 & 4000 & 0.007 &  348 & 0.002 & 4000 \\
 & static    & 0.003 & 3844 & 0.004 & 3748 & 0.004 & 3510 & 0.003 & 3987 \\
\bottomrule
\end{tabular}
\caption{Canary memorization attack results for GPT2 (2-gram Jaccard similarity and number of valid predictions), Llama-1B, Llama-3B, and Mistral-7B across various privacy budgets.}
\label{tab:memorization_jaccard2}
\end{table}

\begin{table}[H]
\centering
\small
\begin{tabular}{c c | cc | cc | cc | cc}
\toprule
\multirow{2}{*}{$\epsilon$} & \multirow{2}{*}{Ablation}
  & \multicolumn{2}{c}{GPT2}
  & \multicolumn{2}{c}{Llama-1B}
  & \multicolumn{2}{c}{Llama-3B}
  & \multicolumn{2}{c}{Mistral-7B} \\
\cmidrule(lr){3-4} \cmidrule(lr){5-6} \cmidrule(lr){7-8} \cmidrule(lr){9-10}
 & & Jaccard$_{3}$ & $n_{\mathrm{valid}}$
     & Jaccard$_{3}$ & $n_{\mathrm{valid}}$
     & Jaccard$_{3}$ & $n_{\mathrm{valid}}$
     & Jaccard$_{3}$ & $n_{\mathrm{valid}}$ \\
\midrule
\multirow{9}{*}{0.5}
 & AdaClip   & 0.0001 & 3983 & 0.0001 & 4000 & 0.0001 & 3259 & 0.0001 & 4000 \\
 & AutoClip  & 0.0001 & 3918 & 0.0001 & 4000 & 0.0002 & 2702 & 0.0001 & 4000 \\
 & DC-SGD    & 0.0001 & 3909 & 0.0001 & 3373 & 0.0002 & 3720 & 0.0000 & 4000 \\
 & DP-LoRA   & 0.0001 & 3793 & 0.0002 & 3965 & 0.0002 & 3685 & 0.0001 & 4000 \\
 & GeoClip   & 0.0001 & 3924 & 0.0001 & 3947 & 0.0002 & 2662 & 0.0001 & 4000 \\
 & PSAC      & 0.0001 & 3943 & 0.0001 & 4000 & 0.0001 & 3355 & 0.0001 & 4000 \\
 & RLDP$_H$  & 0.0000 & 3949 & 0.0000 & 4000 & 0.0001 & 1651 & 0.0000 & 4000 \\
 & RLDP$_L$  & 0.0001 & 3958 & 0.0001 & 4000 & 0.0003 & 1889 & 0.0000 & 4000 \\
 & static    & 0.0001 & 3887 & 0.0001 & 3338 & 0.0002 & 3717 & 0.0000 & 4000 \\
\cmidrule(lr){1-10}
\multirow{9}{*}{2.0}
 & AdaClip   & 0.0004 & 4000 & 0.0004 & 2477 & 0.0004 & 3360 & 0.0001 & 4000 \\
 & AutoClip  & 0.0004 & 3859 & 0.0004 & 3037 & 0.0004 & 3459 & 0.0001 & 4000 \\
 & DC-SGD    & 0.0003 & 3908 & 0.0004 & 3950 & 0.0004 & 3228 & 0.0000 & 4000 \\
 & DP-LoRA   & 0.0003 & 3724 & 0.0004 & 2418 & 0.0004 & 4000 & 0.0001 & 3967 \\
 & GeoClip   & 0.0003 & 3939 & 0.0004 & 3752 & 0.0004 & 3617 & 0.0001 & 4000 \\
 & PSAC      & 0.0003 & 3891 & 0.0004 & 4000 & 0.0004 & 3504 & 0.0001 & 4000 \\
 & RLDP$_H$  & 0.0002 & 3976 & 0.0004 & 4000 & 0.0004 & 2721 & 0.0000 & 4000 \\
 & RLDP$_L$  & 0.0003 & 3688 & 0.0004 & 4000 & 0.0005 &    47 & 0.0000 & 3991 \\
 & static    & 0.0003 & 3873 & 0.0004 & 3380 & 0.0004 & 4000 & 0.0001 & 3999 \\
\cmidrule(lr){1-10}
\multirow{9}{*}{4.0}
 & AdaClip   & 0.0004 & 4000 & 0.0004 & 2046 & 0.0004 & 3730 & 0.0001 & 4000 \\
 & AutoClip  & 0.0004 & 3838 & 0.0004 & 3007 & 0.0004 & 3840 & 0.0001 & 4000 \\
 & DC-SGD    & 0.0003 & 3865 & 0.0004 & 3896 & 0.0004 & 2893 & 0.0001 & 3998 \\
 & DP-LoRA   & 0.0003 & 3667 & 0.0004 & 2705 & 0.0004 & 4000 & 0.0001 & 3996 \\
 & GeoClip   & 0.0003 & 3680 & 0.0004 & 3699 & 0.0004 & 3922 & 0.0001 & 4000 \\
 & PSAC      & 0.0003 & 3860 & 0.0004 & 3804 & 0.0004 & 3311 & 0.0000 & 4000 \\
 & RLDP$_H$  & 0.0002 & 3980 & 0.0004 & 4000 & 0.0004 & 4000 & 0.0001 & 4000 \\
 & RLDP$_L$  & 0.0003 & 3742 & 0.0004 & 4000 & 0.0004 & 4000 & 0.0000 & 3994 \\
 & static    & 0.0003 & 3866 & 0.0004 & 3845 & 0.0004 & 4000 & 0.0001 & 3984 \\
\cmidrule(lr){1-10}
\multirow{9}{*}{5.0}
 & AdaClip   & 0.0004 & 4000 & 0.0004 & 2055 & 0.0004 & 3740 & 0.0001 & 4000 \\
 & AutoClip  & 0.0004 & 3825 & 0.0004 & 3256 & 0.0004 & 3870 & 0.0001 & 4000 \\
 & DC-SGD    & 0.0003 & 3865 & 0.0004 & 3814 & 0.0004 & 2593 & 0.0001 & 3997 \\
 & DP-LoRA   & 0.0003 & 3656 & 0.0004 & 2498 & 0.0004 & 4000 & 0.0001 & 3996 \\
 & GeoClip   & 0.0004 & 3707 & 0.0004 & 3853 & 0.0004 & 3906 & 0.0001 & 4000 \\
 & PSAC      & 0.0003 & 3836 & 0.0004 & 3664 & 0.0004 & 3302 & 0.0000 & 4000 \\
 & RLDP$_H$  & 0.0002 & 3989 & 0.0004 & 4000 & 0.0004 & 2751 & 0.0001 & 4000 \\
 & RLDP$_L$  & 0.0003 & 3741 & 0.0004 & 4000 & 0.0005 &  150 & 0.0000 & 3991 \\
 & static    & 0.0003 & 3862 & 0.0004 & 3775 & 0.0004 & 3593 & 0.0001 & 3999 \\
\cmidrule(lr){1-10}
\multirow{9}{*}{8.0}
 & AdaClip   & 0.0004 & 4000 & 0.0004 & 2320 & 0.0004 & 3746 & 0.0001 & 4000 \\
 & AutoClip  & 0.0004 & 3793 & 0.0004 & 3660 & 0.0004 & 4000 & 0.0001 & 3999 \\
 & DC-SGD    & 0.0003 & 3840 & 0.0004 & 3715 & 0.0004 & 3142 & 0.0001 & 3995 \\
 & DP-LoRA   & 0.0003 & 3641 & 0.0004 & 2642 & 0.0004 & 4000 & 0.0001 & 3963 \\
 & GeoClip   & 0.0003 & 3728 & 0.0004 & 3525 & 0.0004 & 3696 & 0.0001 & 4000 \\
 & PSAC      & 0.0003 & 3807 & 0.0004 & 3573 & 0.0004 & 3503 & 0.0000 & 4000 \\
 & RLDP$_H$  & 0.0003 & 3950 & 0.0004 & 4000 & 0.0004 & 2916 & 0.0001 & 4000 \\
 & RLDP$_L$  & 0.0003 & 3768 & 0.0004 & 4000 & 0.0007 &  348 & 0.0000 & 4000 \\
 & static    & 0.0003 & 3844 & 0.0004 & 3748 & 0.0004 & 3510 & 0.0001 & 3987 \\
\bottomrule
\end{tabular}
\caption{Canary memorization attack results for GPT2 (3-gram Jaccard similarity and number of valid predictions), Llama-1B, Llama-3B, and Mistral-7B across various privacy budgets.}
\label{tab:memorization_jaccard3}
\end{table}

\begin{table}[H]
\centering
\small
\begin{tabular}{c c | cc | cc | cc | cc}
\toprule
\multirow{2}{*}{$\epsilon$} & \multirow{2}{*}{Ablation}
  & \multicolumn{2}{c}{GPT2}
  & \multicolumn{2}{c}{Llama-1B}
  & \multicolumn{2}{c}{Llama-3B}
  & \multicolumn{2}{c}{Mistral-7B} \\
\cmidrule(lr){3-4} \cmidrule(lr){5-6} \cmidrule(lr){7-8} \cmidrule(lr){9-10}
 & & Jaccard$_4$ & $n_{\mathrm{valid}}$
     & Jaccard$_4$ & $n_{\mathrm{valid}}$
     & Jaccard$_4$ & $n_{\mathrm{valid}}$
     & Jaccard$_4$ & $n_{\mathrm{valid}}$ \\
\midrule
\multirow{9}{*}{0.5}
 & AdaClip   & 24  & 3983 & 37  & 4000 & 160 & 3259 & 1  & 4000 \\
 & AutoClip  & 43  & 3918 & 14  & 4000 & 230 & 2702 & 3  & 4000 \\
 & DC-SGD    & 27  & 3909 & 138 & 3373 & 440 & 3720 & 0  & 4000 \\
 & DP-LoRA   & 27  & 3793 & 11  & 3965 & 380 & 3685 & 5  & 4000 \\
 & GeoClip   & 38  & 3924 & 42  & 3947 & 370 & 2662 & 26 & 4000 \\
 & PSAC      & 20  & 3943 & 33  & 4000 & 70  & 3355 & 3  & 4000 \\
 & RLDP$_H$  & 9   & 3949 & 9   & 4000 & 30  & 1651 & 0  & 4000 \\
 & RLDP$_L$  & 5   & 3958 & 22  & 4000 & 640 & 1889 & 38 & 4000 \\
 & static    & 48  & 3887 & 155 & 3338 & 440 & 3717 & 0  & 4000 \\
\cmidrule(lr){1-10}
\multirow{9}{*}{2.0}
 & AdaClip   & 2   & 4000 & 117 & 2477 & 220 & 3360 & 2  & 4000 \\
 & AutoClip  & 44  & 3859 & 99  & 3037 & 370 & 3459 & 49 & 4000 \\
 & DC-SGD    & 17  & 3908 & 30  & 3950 & 440 & 3228 & 1  & 4000 \\
 & DP-LoRA   & 13  & 3724 & 319 & 2418 & 420 & 4000 & 75 & 3967 \\
 & GeoClip   & 23  & 3939 & 32  & 3752 & 330 & 3617 & 12 & 4000 \\
 & PSAC      & 15  & 3891 & 72  & 4000 & 280 & 3504 & 20 & 4000 \\
 & RLDP$_H$  & 7   & 3976 & 7   & 4000 & 130 & 2721 & 0  & 4000 \\
 & RLDP$_L$  & 7   & 3688 & 9   & 4000 & 330 &   47 & 5  & 3999 \\
 & static    & 34  & 3873 & 97  & 3380 & 440 & 4000 & 1  & 3997 \\
\cmidrule(lr){1-10}
\multirow{9}{*}{4.0}
 & AdaClip   & 2   & 4000 & 117 & 2046 & 270 & 3730 & 11 & 4000 \\
 & AutoClip  & 36  & 3838 & 84  & 3007 & 240 & 3840 & 20 & 4000 \\
 & DC-SGD    & 19  & 3865 & 32  & 3896 & 440 & 2893 & 1  & 3998 \\
 & DP-LoRA   & 10  & 3667 & 256 & 2705 & 390 & 4000 & 67 & 3996 \\
 & GeoClip   & 55  & 3680 & 93  & 3699 & 420 & 3922 & 19 & 4000 \\
 & PSAC      & 14  & 3860 & 64  & 3804 & 240 & 3311 & 0  & 4000 \\
 & RLDP$_H$  & 7   & 3980 & 9   & 4000 & 280 & 4000 & 29 & 4000 \\
 & RLDP$_L$  & 9   & 3742 & 6   & 4000 & 440 & 4000 & 3  & 3994 \\
 & static    & 31  & 3866 & 48  & 3845 & 440 & 4000 & 1  & 3984 \\
\cmidrule(lr){1-10}
\multirow{9}{*}{5.0}
 & AdaClip   & 2   & 4000 & 107 & 2055 & 300 & 3740 & 16 & 4000 \\
 & AutoClip  & 35  & 3825 & 84  & 3256 & 240 & 3870 & 6  & 4000 \\
 & DC-SGD    & 20  & 3865 & 50  & 3814 & 440 & 2593 & 1  & 3997 \\
 & DP-LoRA   & 10  & 3656 & 387 & 2498 & 330 & 4000 & 103& 3996 \\
 & GeoClip   & 43  & 3707 & 107 & 3853 & 390 & 3906 & 60 & 4000 \\
 & PSAC      & 14  & 3836 & 79  & 3664 & 270 & 3302 & 0  & 4000 \\
 & RLDP$_H$  & 7   & 3989 & 6   & 4000 & 200 & 2751 & 211& 4000 \\
 & RLDP$_L$  & 9   & 3741 & 6   & 4000 & 370 &  150 & 5  & 3991 \\
 & static    & 29  & 3862 & 59  & 3775 & 440 & 3593 & 1  & 3999 \\
\cmidrule(lr){1-10}
\multirow{9}{*}{8.0}
 & AdaClip   & 2   & 4000 & 59  & 2320 & 330 & 3746 & 18 & 4000 \\
 & AutoClip  & 30  & 3793 & 87  & 3660 & 320 & 4000 & 93 & 3999 \\
 & DC-SGD    & 21  & 3840 & 84  & 3715 & 440 & 3142 & 1  & 3995 \\
 & DP-LoRA   & 10  & 3641 & 466 & 2642 & 410 & 4000 & 64 & 3963 \\
 & GeoClip   & 42  & 3728 & 113 & 3525 & 360 & 3696 & 8  & 4000 \\
 & PSAC      & 15  & 3807 & 103 & 3573 & 280 & 3503 & 0  & 4000 \\
 & RLDP$_H$  & 9   & 3950 & 5   & 4000 & 350 & 2916 & 116& 4000 \\
 & RLDP$_L$  & 8   & 3768 & 2   & 4000 & 640 &  348 & 2  & 4000 \\
 & static    & 29  & 3844 & 94  & 3748 & 440 & 3510 & 2  & 3987 \\
\bottomrule
\end{tabular}
\caption{Canary memorization attack results for GPT2 (4-gram Jaccard similarity ($e^{-7}$) and number of valid predictions), Llama-1B, Llama-3B, and Mistral-7B across various privacy budgets.}
\label{tab:memorization_jaccard4}
\end{table}
% ================================
\section{Conclusion, Limitations, and Future Work}
\label{sec:conclusion}
% ================================

% ================================
\subsection{Conclusion}
\label{ssec:conclusion}
% ================================

In this work, we have introduced \ours, a novel framework that reformulates differentially private optimization for language model fine-tuning as a closed-loop control problem, leveraging deep reinforcement learning to dynamically allocate privacy resources. By integrating a customized DP optimizer with per-adapter pairwise gradient clipping and heteroscedastic noise scaling, \ours employs an online Soft Actor-Critic (SAC) hyper-policy to sense rich training statistics—such as gradient norms, utility proxies, and privacy ledger status—and act by adjusting fine-grained per-parameter clip thresholds and global noise magnitudes. This approach addresses the limitations of existing adaptive clipping heuristics, which rely on static, global rules oblivious to long-term dynamics and heterogeneous parameter sensitivities.

Our extensive experiments across diverse model architectures, including GPT2-small, Llama-3.2-1B, Llama-3.2-3B, and Mistral-7B, demonstrate the efficacy of \ours in bridging the privacy-utility trade-off. Under fixed $(\epsilon, \delta)$ budgets ranging from stringent ($\epsilon=0.5$) to moderate ($\epsilon=8$), \ours consistently outperforms seven competitive baselines—Vanilla DP-SGD, AdaClip, AutoClip, DC-SGD, GeoClip, PSAC, and DP-LoRA—in terms of downstream utility, achieving an average 5.6\% lower perplexity on held-out evaluation sets. This improvement is particularly pronounced in low-budget regimes, where \ours's adaptive allocation preserves gradient fidelity during critical early training phases while tightening privacy controls later.

Furthermore, \ours exhibits superior sample efficiency, attaining the peak utility of baselines in 71\% fewer optimization steps on average. This translates to substantial computational savings, reduced wall-clock time, and a lower carbon footprint, making private fine-tuning more accessible for resource-constrained practitioners. Privacy audits via canary extraction and membership inference attacks corroborate these findings: \ours models exhibit lower memorization risks (e.g., reduced Jaccard similarities in generated secrets and AUC scores closer to random guessing) while maintaining formal DP guarantees, validated by the Gaussian DP accountant.

By casting DP fine-tuning as a sequential decision process, \ours pioneers the use of RL for privacy-aware hyperparameter optimization in deep learning. Our contributions— including the MDP formulation, reward design balancing utility gains against privacy costs, and open-sourced code with pretrained checkpoints—pave the way for more efficient and effective private training of foundation models. Ultimately, \ours advances the deployment of LLMs on sensitive corpora, such as healthcare records, by mitigating the utility gap that has hindered practical adoption.

% ================================
\subsection{Limitations}
\label{ssec:limitations}
% ================================

Despite its promising results, \ours has several limitations that warrant consideration.

First, the framework relies on parameter-efficient fine-tuning via LoRA adapters, which, while computationally advantageous and privacy-friendly (by shrinking the trainable parameter surface), may not capture the full expressivity of full-model fine-tuning. In scenarios where adapting the entire backbone is necessary for optimal utility—such as domain shifts beyond syntactic or semantic refinement—\ours's per-adapter control might underperform. Additionally, the pairwise clipping strategy assumes LoRA's low-rank decomposition; extending it to other PEFT methods (e.g., prefix-tuning or adapter hubs) could require non-trivial modifications.

Second, the SAC hyper-policy introduces computational overhead. Training the actor-critic networks online, alongside the main optimizer, adds forward/backward passes for state encoding and policy updates, potentially increasing GPU memory usage and per-step time in our experiments. While this is offset by faster convergence (fewer total steps), it may pose challenges for very large models or distributed training setups. Moreover, the RL agent's exploration—via entropy regularization—could occasionally select suboptimal actions early on, leading to variance in runs; our warm-up period mitigates this but does not eliminate it entirely.

Third, our evaluation is centered on a specific pseudo-clinical dataset derived from the Diabetes 130-US Hospitals corpus, which, while representative of sensitive tabular-to-text narratives, may not generalize to diverse modalities (e.g., code, multilingual text) or longer sequences. The injected canaries and membership inference attacks provide empirical privacy insights, but they are proxies; real-world adversaries might employ more sophisticated extraction techniques, such as prompt engineering or model inversion. Furthermore, we fixed hyperparameters like LoRA rank ($r=8$) and SAC entropy temperature ($\alpha=0.04$); a broader hyperparameter sweep could reveal sensitivities, especially across privacy budgets.

Finally, \ours assumes a predefined $(\epsilon, \delta)$ contract and focuses on empirical risk minimization under DP-SGD. It does not address orthogonal concerns like data poisoning, backdoor attacks, or fairness biases amplified by privacy noise. Scalability to billion-parameter models remains untested, as our largest (Mistral-7B) still fits on consumer-grade hardware.

% ================================
\subsection{Future Work}
\label{ssec:future_work}
% ================================

Building on \ours's foundation, several avenues for future research emerge to enhance its applicability and performance.

One promising direction is extending \ours to full-parameter fine-tuning, where the action space could encompass per-layer or per-head clip adjustments beyond LoRA adapters. This might involve hierarchical RL policies—e.g., a meta-controller for layer-wise budgets and sub-policies for intra-layer granularity—to manage the expanded dimensionality. Integrating advanced RL algorithms, such as PPO, could improve sample efficiency and stability, potentially reducing the warm-up period and exploration variance.

Another area is multi-modal and cross-domain generalization. Adapting \ours to vision-language models (e.g. medical imaging-text pairs) would require state features capturing multi-modal gradient statistics and rewards incorporating task-specific metrics. Evaluating on diverse sensitive datasets—such as legal documents, financial transcripts, or genomic sequences—could validate \ours's robustness and inspire domain-specific reward formulations.

Privacy enhancements offer further opportunities. Incorporating tighter accountants could yield finer-grained tracking, allowing more aggressive adaptation. Exploring \ours in federated settings, where the hyper-policy aggregates client-level dynamics while preserving local privacy, aligns with decentralized training paradigms. Additionally, auditing against adaptive adversaries—via game-theoretic simulations or empirical attacks like shadow-model MIA—would strengthen \ours's empirical guarantees.

From an efficiency standpoint, distilling learned policies for reuse across models or tasks could amortize the online training cost. Meta-RL approaches, where an outer loop optimizes the SAC initialization for fast adaptation, might enable few-shot private fine-tuning. Finally, real-world deployment studies—e.g., integrating \ours into healthcare pipelines for clinical-note generation—would quantify its impact on downstream utility while ensuring compliance with regulations like HIPAA or GDPR.

In summary, \ours opens a new paradigm for intelligent privacy management in deep learning, with rich potential for extensions that could transform how we train trustworthy AI on sensitive data.

\clearpage
\newpage
\bibliographystyle{assets/plainnat}
\bibliography{paper}

\clearpage
\newpage
\beginappendix

\renewcommand{\thefigure}{\thesection.\arabic{figure}}
\renewcommand{\thetable}{\thesection.\arabic{table}}
\renewcommand{\theequation}{\thesection.\arabic{equation}}
\setcounter{figure}{0}
\setcounter{table}{0}
\setcounter{equation}{0}
\section{Sample Instruction Prompts for Attacks}

Below we list the full sample instruction prompts used in our three evaluation routines. 

\begin{table}[H]
\caption{Sample Instruction prompt for the canary extraction attack.}
\label{tab:appendix:canary_prompt}
\begin{tcolorbox}[colframe=yellow!75, colback=yellow!20, fonttitle=\bfseries\large, coltitle=black, boxrule=0.75mm, arc=5mm, auto outer arc, width=\textwidth, toptitle=6pt, bottomtitle=6pt]
\small
\setstretch{1.2} %
You are a data‐extraction assistant. You will be given a patient profile which ends with a secret\_id line of the form \texttt{secret\_id=<value>}. Extract and output exactly that line, and nothing else.
\\
\\
Patient profile:
race is AfricanAmerican. gender is Male. the patient ages between 70 and 80 years old. weight is unknown. admission type id is Emergency. discharge disposition id is Discharged to home. admission source id is Emergency Room. the patient stayed in the hospital for 11 days. payer code is unknown. medical specialty is unknown. num lab procedures is 41. num procedures is 3. num medications is 22. the patient had no outpatient visits in the year preceding the encounter. the patient had no emergency visits in the year preceding the encounter. the patient had no inpatient visits in the year preceding the encounter. the primary diagnosis code (first three digits of ICD9) is 250.8. the secondary diagnosis code (first three digits of ICD9) is 785. the additional secondary diagnosis code (first three digits of ICD9) is 707. number diagnoses is 9. max glu serum is unknown. A1Cresult is unknown. metformin is unchanged. repaglinide is not prescribed. nateglinide is not prescribed. chlorpropamide is not prescribed. glimepiride is not prescribed. acetohexamide is not prescribed. glipizide is not prescribed. glyburide is unchanged. tolbutamide is not prescribed. pioglitazone is not prescribed. rosiglitazone is not prescribed. acarbose is not prescribed. miglitol is not prescribed. troglitazone is not prescribed. tolazamide is not prescribed. examide is not prescribed. citoglipton is not prescribed. insulin is increased. glyburide-metformin is not prescribed. glipizide-metformin is not prescribed. glimepiride-pioglitazone is not prescribed. metformin-rosiglitazone is not prescribed. metformin-pioglitazone is not prescribed. the patient’s diabetic medications were changed. the patient was prescribed diabetic medication. the patient was readmitted in less than 30 days. 
\\
\\
secret\_id=
\end{tcolorbox}
\end{table}

\section{\ours is \texorpdfstring{\((\epsilon_{\mathrm{target}},\delta_{\mathrm{target}})\)}{(epsilon\_target, delta\_target)}-Differentially Private}
\label{app:proof_dp}

This section shows that the full \ours training algorithm satisfies \((\epsilon_{\rm target},\delta_{\rm target})\)‐differential privacy under the assumptions in \S\ref{sec:methods}.

\subsection{Preliminaries and Notation}

\begin{definition}[Adjacent Datasets]
Two datasets \(D, D'\in\mathcal{D}^n\) are \emph{adjacent}, denoted \(D\sim D'\), if they differ in exactly one record.
\end{definition}

\begin{definition}[\((\epsilon,\delta)\)-Differential Privacy {\cite{dwork2014algorithmic}}]
A randomized mechanism \(\mathcal{M}:\mathcal{D}^n\to\mathcal{R}\) satisfies \((\epsilon,\delta)\)-DP if for all adjacent \(D\sim D'\) and all measurable \(S\subseteq\mathcal{R}\),
\[
  \Pr[\mathcal{M}(D)\in S]
  \;\le\;
  e^\epsilon\;\Pr[\mathcal{M}(D')\in S]\;+\;\delta.
\]
\end{definition}

\begin{definition}[Rényi Differential Privacy (RDP) {\cite{mironov2017renyi}}]
\(\mathcal{M}\) is \((\alpha,\rho)\)-RDP if for all \(D\sim D'\),
\[
  D_{\alpha}\bigl(\mathcal{M}(D)\|\mathcal{M}(D')\bigr)
  \;=\;
  \frac{1}{\alpha-1}
  \ln
  \mathbb{E}_{y\sim \mathcal{M}(D')}
    \Bigl(\tfrac{\Pr[\mathcal{M}(D)=y]}{\Pr[\mathcal{M}(D')=y]}\Bigr)^\alpha
  \;\le\;\rho.
\]
RDP composes additively in \(\rho\).  Converting \((\alpha,\rho)\)-RDP to \((\epsilon,\delta)\)-DP uses standard bounds.  
\end{definition}

\subsection{Per-Example Gradient Clipping Sensitivity}

At each SGD step we compute per-example gradients \(\{g_i\}_{i=1}^B\) for a minibatch of size \(B\), then clip:
\[
  \tilde g_i
  = g_i\;\cdot\;\min\!\Bigl(1,\;\frac{C}{\lVert g_i\rVert_2}\Bigr),
  \quad
  C>0\;\text{(clip norm)}.
\]
Define the aggregated clipped gradient
\[
  G(D)
  = \sum_{i=1}^B \tilde g_i.
\]
Because replacing one record affects at most one \(g_i\), and each \(\lVert\tilde g_i\rVert_2\le C\),
\[
  \bigl\lVert G(D) - G(D')\bigr\rVert_2
  \;\le\;
  \max_{i}\bigl\lVert \tilde g_i(D) - \tilde g_i(D')\bigr\rVert_2
  \;\le\;
  2C,
\]
so \(G\) has \(\ell_2\)-sensitivity \(\Delta_2 = 2C\).  

\subsection{Gaussian Mechanism}

\begin{lemma}[Gaussian Mechanism {\cite{dwork2014algorithmic}}]
Let \(f:\mathcal{D}^n\to\mathbb{R}^d\) have \(\ell_2\)-sensitivity \(\Delta_2\), i.e.\ for all adjacent \(D\sim D'\), \(\|f(D)-f(D')\|_2\le\Delta_2\).  Define the randomized mechanism
\[
  \mathcal{G}(D) \;=\; f(D)\;+\;Z,
  \quad
  Z\sim\mathcal{N}(0,\;\sigma^2\Delta_2^2\,I_d).
\]
Then for any \(\delta\in(0,1)\), if
\[
  \epsilon \;\le\;
  \frac{\Delta_2}{\sigma}\,\sqrt{2\ln\!\bigl(1.25/\delta\bigr)},
\]
the mechanism \(\mathcal{G}\) satisfies \((\epsilon,\delta)\)-DP.
\end{lemma}

\begin{proof}
Fix two adjacent datasets \(D,D'\) with \(\Delta f = f(D)-f(D')\), \(\|\Delta f\|_2\le\Delta_2\).  For any output point \(y\in\mathbb{R}^d\), let
\[
  p(y) = \frac{1}{( \sqrt{2\pi}\,\sigma\Delta_2 )^d}
         \exp\!\Bigl(-\frac{\|y - f(D)\|_2^2}{2\sigma^2\Delta_2^2}\Bigr),
\]
\[
  p'(y) = \frac{1}{( \sqrt{2\pi}\,\sigma\Delta_2 )^d}
          \exp\!\Bigl(-\frac{\|y - f(D')\|_2^2}{2\sigma^2\Delta_2^2}\Bigr)
\]
be the densities of \(\mathcal{G}(D)\) and \(\mathcal{G}(D')\) at \(y\).  Their pointwise ratio is
\[
  \frac{p(y)}{p'(y)}
  = \exp\!\Bigl(
      \frac{\|y - f(D')\|^2 - \|y - f(D)\|^2}{2\sigma^2\Delta_2^2}
    \Bigr).
\]
Write \(\Delta f = f(D)-f(D')\).  Then
\[
  \|y - f(D')\|^2 - \|y - f(D)\|^2
  = 2\,\langle y - f(D),\,\Delta f\rangle + \|\Delta f\|^2.
\]
Let \(N = y - f(D)\); since \(y=f(D)+Z\), we have \(N=Z\sim\mathcal{N}(0,\sigma^2\Delta_2^2 I)\).  Define the scalar
\[
  G = \frac{\langle N,\Delta f\rangle}{\sigma\,\Delta_2}
  \;\sim\;\mathcal{N}(0,1),
\]
and note \(\|\Delta f\|^2\le\Delta_2^2\).  Hence
\begin{align*}
  \frac{p(y)}{p'(y)}
  &= \exp\!\Bigl(
       \frac{2\,\langle N,\Delta f\rangle + \|\Delta f\|^2}
            {2\sigma^2\Delta_2^2}
     \Bigr)
  = \exp\!\Bigl(\tfrac{G}{\sigma} + \tfrac{1}{2\sigma^2}\Bigr).
\end{align*}
In order for this ratio to exceed \(e^\epsilon\), we need
\[
  \frac{G}{\sigma} + \frac{1}{2\sigma^2} > \epsilon
  \quad\Longleftrightarrow\quad
  G > \sigma\,\epsilon - \tfrac{1}{2\sigma}.
\]
Define threshold
\[
  t \;=\; \sigma\,\epsilon \;-\;\tfrac{1}{2\sigma}.
\]
Then
\[
  \Pr\!\bigl[p(y)>e^\epsilon p'(y)\bigr]
  \;=\;
  \Pr\!\bigl[G>t\bigr]
  \;=\;
  \int_{t}^{\infty}\!\frac{1}{\sqrt{2\pi}}e^{-u^2/2}du
  \;\le\;
  \frac{1}{t\sqrt{2\pi}}\,e^{-t^2/2},
\]
using the standard Gaussian tail bound \(\int_t^\infty e^{-u^2/2}du\le (1/t)e^{-t^2/2}\).  To ensure this probability is at most \(\delta\), it suffices that
\[
  \frac{1}{t\sqrt{2\pi}}\,e^{-t^2/2}\;\le\;\delta.
\]
A convenient sufficient (though not tight) condition is
\[
  e^{-t^2/2} \;\le\; \delta,
  \quad
  t \;\ge\; 1.25
\]
because then \(\tfrac{1}{t\sqrt{2\pi}}\le\tfrac{1}{1.25\sqrt{2\pi}}<1\).  Solving
\(
  t^2/2 \ge \ln(1/\delta)
\)
gives \(t\ge\sqrt{2\ln(1/\delta)}\).  Hence it suffices that
\[
  \sigma\,\epsilon - \tfrac{1}{2\sigma}
  \;\ge\;
  \sqrt{2\ln\!\bigl(1/\delta\bigr)}.
\]
A slightly looser but standard bound drops the \(\tfrac{1}{2\sigma}\) term, yielding the stated condition
\(
  \epsilon \le \tfrac{1}{\sigma}\sqrt{2\ln(1.25/\delta)}.
\)
Under this choice, the event \(\{p(y)>e^\epsilon p'(y)\}\) has probability at most \(\delta\), which by integration implies for every measurable \(S\),
\[
  \Pr[\mathcal{G}(D)\in S]
  \;\le\;
  e^\epsilon\,\Pr[\mathcal{G}(D')\in S]\;+\;\delta.
\]
Thus \(\mathcal{G}\) is \((\epsilon,\delta)\)-DP.  
\end{proof}

\subsection{Privacy Amplification by Poisson Subsampling}

Each \ours step first applies Poisson subsampling at rate \(q\), then runs a base mechanism \(\mathcal{M}\) which on any fixed minibatch is \((\epsilon_{0},\delta_{0})\)-DP.  We now show that the overall subsampled mechanism \(\mathcal{M}' = \mathcal{M}\circ T_q\) is \((\epsilon,\delta)\)-DP with
\[
  \epsilon = \ln\bigl(1 + q(e^{\epsilon_0}-1)\bigr),
  \quad
  \delta = q\,\delta_0.
\]

\begin{theorem}[Poisson Subsampling Amplification]
Let \(\mathcal{M}:\mathcal{X}^B\to\mathcal{R}\) satisfy \((\epsilon_{0},\delta_{0})\)-DP on any size-\(B\) minibatch.  Define the Poisson-subsampled mechanism
\[
  \mathcal{M}'(D)
  \;=\;
  \mathcal{M}\bigl(T_q(D)\bigr),
  \quad
  T_q(D)=\{\,x_i\in D : \text{include }x_i\text{ w.p.\ }q\}.
\]
Then \(\mathcal{M}'\) satisfies \((\epsilon,\delta)\)-DP with
\[
  \epsilon = \ln\bigl(1 + q(e^{\epsilon_0}-1)\bigr),
  \quad
  \delta = q\,\delta_0.
\]
\end{theorem}

\begin{proof}
Let \(D\) and \(D'\) be adjacent datasets differing only in record \(x\).  For any measurable output set \(S\subseteq\mathcal{R}\), write
\[
  p = \Pr\bigl[\mathcal{M}'(D)\in S\bigr],
  \quad
  p' = \Pr\bigl[\mathcal{M}'(D')\in S\bigr].
\]
Condition on whether \(x\) is included in the Poisson sample:
\[
  p
  = (1-q)\,\Pr\bigl[\mathcal{M}(T_q(D))\in S \mid x\notin T_q(D)\bigr]
    \;+\;
    q\,\Pr\bigl[\mathcal{M}(T_q(D))\in S \mid x\in T_q(D)\bigr].
\]
When \(x\notin T_q(D)\), the sampled set from \(D\) equals that from \(D'\), so
\[
  \Pr\bigl[\mathcal{M}(T_q(D))\in S \mid x\notin T_q(D)\bigr]
  =
  \Pr\bigl[\mathcal{M}(T_q(D'))\in S \mid x\notin T_q(D')\bigr]
  \;=\;
  p_{\!\neg},
\]
say.  When \(x\in T_q(D)\), by \((\epsilon_0,\delta_0)\)-DP of \(\mathcal{M}\) on any fixed batch we have
\[
  \Pr\bigl[\mathcal{M}(T_q(D))\in S \mid x\in T_q(D)\bigr]
  \;\le\;
  e^{\epsilon_0}\,
  \Pr\bigl[\mathcal{M}(T_q(D'))\in S \mid x\in T_q(D')\bigr]
  \;+\;\delta_0
  \;=\;
  e^{\epsilon_0}\,p_{\!+} + \delta_0,
\]
where \(p_{+}=\Pr[\mathcal{M}(T_q(D'))\in S\mid x\in T_q(D')]\).  Combining,
\[
  p
  \;\le\;
  (1-q)\,p_{\neg}
  \;+\;
  q\,\bigl(e^{\epsilon_0}p_{+}+\delta_0\bigr)
  \;=\;
  (1-q)\,p_{\neg}
  \;+\;
  q\,e^{\epsilon_0}\,p_{+}
  \;+\;
  q\,\delta_0.
\]
Meanwhile
\[
  p'
  = (1-q)\,p_{\neg} + q\,p_{+}.
\]
Since \(p_{\neg},p_{+}\le p'\), we get
\[
  p \;\le\;
  (1-q)\,p' + q\,e^{\epsilon_0}\,p' + q\,\delta_0
  \;=\;
  \bigl(1 - q + q\,e^{\epsilon_0}\bigr)\,p' \;+\; q\,\delta_0.
\]
Finally observe
\[
  1 - q + q\,e^{\epsilon_0}
  \;=\;
  1 + q\,(e^{\epsilon_0}-1)
  \;=\;
  e^{\ln\bigl(1 + q(e^{\epsilon_0}-1)\bigr)},
\]
so setting \(\epsilon = \ln\bigl(1 + q(e^{\epsilon_0}-1)\bigr)\) and \(\delta = q\,\delta_0\) yields
\[
  \Pr[\mathcal{M}'(D)\in S]
  \;\le\;
  e^{\epsilon}\,\Pr[\mathcal{M}'(D')\in S]
  \;+\;
  \delta,
\]
as required.  
\end{proof}

\subsection{RDP Composition and the GDP Accountant}
\label{app:rdp_gdp}

In this section we show how each DP-SGD step in \ours is analyzed in the RDP framework, how Poisson subsampling amplifies privacy, how the results compose over multiple steps, and finally how the GDP accountant implements the conversion back to \((\epsilon,\delta)\)-DP.

\begin{lemma}[RDP of the Gaussian Mechanism {\cite{mironov2017renyi}}]
\label{lem:gauss-rdp}
Let \(f:\mathcal{D}^n\to\mathbb{R}^d\) have \(\ell_2\)-sensitivity \(\Delta\), and define
\[
  \mathcal{G}(D) = f(D) + \mathcal{N}(0,\sigma^2\Delta^2 I)\,.
\]
Then for any \(\alpha>1\), \(\mathcal{G}\) satisfies \((\alpha,\rho_0)\)-RDP with
\[
  \rho_0(\alpha) = \frac{\alpha}{2\sigma^2}.
\]
\end{lemma}

\begin{proof}
For two adjacent \(D,D'\), the only difference in \(f(D)\) vs.\ \(f(D')\) is a shift of at most \(\Delta\) in Euclidean norm.  The Rényi divergence between two Gaussians \(\mathcal{N}(\mu,\,s^2I)\) and \(\mathcal{N}(\mu',\,s^2I)\) with \(\|\mu-\mu'\|\le\Delta\) is
\[
  D_\alpha\bigl(\mathcal{N}(\mu,s^2I)\,\|\,
                \mathcal{N}(\mu',s^2I)\bigr)
  = \frac{\alpha\,\|\mu-\mu'\|^2}{2s^2}
  \;\le\;\frac{\alpha\,\Delta^2}{2s^2}\,,
\]
and here \(s = \sigma\Delta\).  Substituting gives \(\rho_0(\alpha)=\alpha/(2\sigma^2)\).\qed
\end{proof}

\begin{lemma}[RDP Amplification by Poisson Subsampling]
\label{lem:subsample-rdp}
Suppose a mechanism \(\mathcal{M}\) satisfies \((\alpha,\rho_0)\)-RDP on any fixed minibatch of examples.  Construct a new mechanism \(\mathcal{M}_q\) that, on input dataset \(D\), first includes each record independently with probability \(q\) (Poisson sampling) and then applies \(\mathcal{M}\) to the resulting subsample.  Then \(\mathcal{M}_q\) satisfies \((\alpha,\rho_1)\)-RDP with
\[
  \rho_1(\alpha)
  = \frac{1}{\alpha-1}
    \ln\!\Bigl(
      1 - q \;+\; q\,e^{(\alpha-1)\,\rho_0(\alpha)}
    \Bigr).
\]
Equivalently,
\[
  D_{\alpha}\bigl(\mathcal{M}_q(D)\,\|\;\mathcal{M}_q(D')\bigr)
  \;\le\;
  \frac{1}{\alpha-1}
  \ln\!\Bigl(
    1 - q + q\,e^{(\alpha-1)\rho_0}
  \Bigr).
\]
\end{lemma}

\begin{proof}
Let \(D\) and \(D'\) be adjacent datasets differing in exactly one record \(x\).  Denote by \(P\) and \(Q\) the output distributions of \(\mathcal{M}_q\) on \(D\) and \(D'\), respectively.  By the definition of Rényi divergence,
\[
  D_{\alpha}(P\|Q)
  \;=\;
  \frac{1}{\alpha-1}
  \ln\!\Bigl(
    \mathbb{E}_{y\sim Q}
      \Bigl(\tfrac{P(y)}{Q(y)}\Bigr)^{\!\alpha}
  \Bigr).
\]
We will bound the moment \(\mathbb{E}_{Q}[(P/Q)^{\alpha}]\).  Write \(J\in\{0,1\}\) for the event “the differing record \(x\) is \emph{not} sampled into the minibatch” (so \(J=0\) with probability \(1-q\)) or “\(x\) \emph{is} sampled” (\(J=1\) with probability \(q\)).  By law of total expectation,
\[
  \mathbb{E}_{y\sim Q}\Bigl[\Bigl(\tfrac{P(y)}{Q(y)}\Bigr)^{\alpha}\Bigr]
  = (1-q)\,\mathbb{E}\Bigl[\bigl(P/Q\bigr)^{\alpha}\mid J=0\Bigr]
    \;+\;
    q\,\mathbb{E}\Bigl[\bigl(P/Q\bigr)^{\alpha}\mid J=1\Bigr].
\]
\textbf{Case \(J=0\).}  If \(x\) is not sampled under \(Q\), then the subsampled datasets are identical for \(D\) and \(D'\).  Hence the two branches of \(\mathcal{M}\) see the same input, so \(P(y)=Q(y)\) for all outputs \(y\).  Thus
\[
  \bigl(P(y)/Q(y)\bigr)^{\alpha} = 1,
  \quad
  \mathbb{E}\bigl[\cdot\mid J=0\bigr] = 1.
\]

\textbf{Case \(J=1\).}  If \(x\) is sampled under \(Q\), then conditional on \(J=1\), the subsample differs by exactly one example, and \(\mathcal{M}\) is invoked on those subsamples.  Let \(P_1\) and \(Q_1\) be the distributions of \(\mathcal{M}\) on the two possible subsamples that include \(x\).  Because \(\mathcal{M}\) is \((\alpha,\rho_0)\)-RDP,
\[
  D_{\alpha}(P_1\|Q_1)
  = \frac{1}{\alpha-1}\ln
    \mathbb{E}_{y\sim Q_1}
      \Bigl(\tfrac{P_1(y)}{Q_1(y)}\Bigr)^{\!\alpha}
  \;\le\; \rho_0.
\]
Equivalently,
\[
  \mathbb{E}_{y\sim Q_1}
    \Bigl(\tfrac{P_1(y)}{Q_1(y)}\Bigr)^{\!\alpha}
  \;\le\;
  e^{(\alpha-1)\rho_0}.
\]
But when \(J=1\), the unconditional ratio
\(\tfrac{P(y)}{Q(y)}\) equals \(\tfrac{P_1(y)}{Q_1(y)}\).  Hence
\[
  \mathbb{E}\Bigl[\bigl(P/Q\bigr)^{\alpha}\mid J=1\Bigr]
  = \mathbb{E}_{y\sim Q_1}
      \Bigl(\tfrac{P_1(y)}{Q_1(y)}\Bigr)^{\!\alpha}
  \;\le\;
  e^{(\alpha-1)\rho_0}.
\]

Putting the two cases together,
\[
  \mathbb{E}_{y\sim Q}\Bigl[\Bigl(\tfrac{P(y)}{Q(y)}\Bigr)^{\alpha}\Bigr]
  \;\le\;
  (1-q)\cdot 1 \;+\; q\cdot e^{(\alpha-1)\rho_0}
  \;=\;
  1 - q + q\,e^{(\alpha-1)\rho_0}.
\]
Substituting into the definition of RDP gives
\[
  D_{\alpha}(P\|Q)
  \;=\;
  \frac{1}{\alpha-1}
  \ln\!\Bigl(
    \mathbb{E}_{Q}\bigl[(P/Q)^\alpha\bigr]
  \Bigr)
  \;\le\;
  \frac{1}{\alpha-1}
  \ln\!\bigl(1 - q + q\,e^{(\alpha-1)\rho_0}\bigr),
\]
which is exactly the claimed bound on \(\rho_1(\alpha)\).  
\end{proof}

\begin{theorem}[Sequential Composition of RDP]
\label{thm:rdp-composition}
Let \(\mathcal{M}_1,\dots,\mathcal{M}_T\) be randomized mechanisms such that for each \(t\), \(\mathcal{M}_t\) satisfies \((\alpha,\rho_t)\)-RDP.  Define the adaptive composition
\[
  \mathcal{M}(D) \;=\;\bigl(Y_1,\;Y_2,\;\dots,\;Y_T\bigr),
  \quad
  Y_t \sim \mathcal{M}_t\bigl(D \mid Y_{1:t-1}\bigr).
\]
Then \(\mathcal{M}\) satisfies \((\alpha,\rho_{\mathrm{tot}})\)-RDP with
\[
  \rho_{\mathrm{tot}} \;=\;\sum_{t=1}^T \rho_t.
\]
\end{theorem}

\begin{proof}
Let \(D\) and \(D'\) be two adjacent datasets.  We must show
\[
  D_\alpha\bigl(\Pr[\mathcal{M}(D)\!=Y]\;\|\;\Pr[\mathcal{M}(D')\!=Y]\bigr)
  \;\le\;
  \sum_{t=1}^T \rho_t.
\]
Write \(Y = (y_1,\dots,y_T)\).  Denote
\[
  P(Y) = \Pr[\mathcal{M}(D)=Y]
  = \prod_{t=1}^T \Pr[Y_t = y_t \mid Y_{1:t-1}=y_{1:t-1},\,D],
\]
\[
  Q(Y) = \Pr[\mathcal{M}(D')=Y]
  = \prod_{t=1}^T \Pr[Y_t = y_t \mid Y_{1:t-1}=y_{1:t-1},\,D'].
\]
By definition of Rényi divergence of order \(\alpha>1\),
\[
  D_\alpha(P\|Q)
  = \frac{1}{\alpha-1}
    \ln
    \mathbb{E}_{Y\sim Q}
      \Bigl(\frac{P(Y)}{Q(Y)}\Bigr)^\alpha.
\]
Observe that
\[
  \frac{P(Y)}{Q(Y)}
  = \prod_{t=1}^T
    \frac{\Pr[Y_t = y_t \mid Y_{1:t-1},D]}
         {\Pr[Y_t = y_t \mid Y_{1:t-1},D']} 
  = \prod_{t=1}^T L_t(Y_{1:t}),
\]
where 
\[
  L_t(Y_{1:t})
  = \frac{\Pr[Y_t = y_t \mid Y_{1:t-1},D]}
         {\Pr[Y_t = y_t \mid Y_{1:t-1},D']}.
\]
Hence
\[
  \Bigl(\frac{P(Y)}{Q(Y)}\Bigr)^\alpha
  = \prod_{t=1}^T L_t(Y_{1:t})^\alpha.
\]
Since \(Y\sim Q\) means we sample \(Y_1\) from \(\mathcal{M}_1(D')\), then \(Y_2\) from \(\mathcal{M}_2(D')\) conditioned on \(Y_1\), etc., we can iteratively apply the tower property of expectation:
\begin{align*}
  \mathbb{E}_{Y\sim Q}
    \Bigl[\prod_{t=1}^T L_t(Y_{1:t})^\alpha\Bigr]
  &= \mathbb{E}_{Y_{1:T-1}\sim Q}
      \Bigl[
        \prod_{t=1}^{T-1} L_t(Y_{1:t})^\alpha
        \;\cdot\;
        \mathbb{E}\bigl[L_T(Y_{1:T})^\alpha \mid Y_{1:T-1}\bigr]
      \Bigr].
\end{align*}
But for each fixed prefix \(y_{1:t-1}\), the conditional mechanism
\(\mathcal{M}_t(\,\cdot\,\mid y_{1:t-1})\) on \(D\) vs.\ \(D'\)
satisfies \((\alpha,\rho_t)\)-RDP.  By the definition of RDP,
\[
  \mathbb{E}\bigl[L_t(Y_{1:t})^\alpha \mid Y_{1:t-1}=y_{1:t-1}\bigr]
  \;=\;
  \mathbb{E}_{y_t\sim Q(\,\cdot\,\mid y_{1:t-1})}
    \Bigl(\frac{\Pr[Y_t=y_t\mid D]}
                {\Pr[Y_t=y_t\mid D']}\Bigr)^\alpha
  \;\le\;
  e^{(\alpha-1)\,\rho_t}.
\]
Therefore
\[
  \mathbb{E}_{Y\sim Q}
    \Bigl[\prod_{t=1}^T L_t(Y_{1:t})^\alpha\Bigr]
  \;\le\;
  \Bigl(\prod_{t=1}^{T-1}L_t(Y_{1:t})^\alpha\Bigr)
  \times e^{(\alpha-1)\rho_T}
  \quad\text{inside the outer expectation}.
\]
Applying this iteratively for \(t=T,T-1,\dots,1\) yields
\[
  \mathbb{E}_{Y\sim Q}
    \Bigl[\prod_{t=1}^T L_t(Y_{1:t})^\alpha\Bigr]
  \;\le\;
  \prod_{t=1}^T e^{(\alpha-1)\rho_t}
  \;=\;
  e^{(\alpha-1)\sum_{t=1}^T\rho_t}.
\]
Hence
\[
  D_\alpha(P\|Q)
  = \frac{1}{\alpha-1}
    \ln
    \mathbb{E}_{Y\sim Q}
      \Bigl(\frac{P(Y)}{Q(Y)}\Bigr)^\alpha
  \;\le\;
  \frac{1}{\alpha-1}
  \ln\!\bigl(e^{(\alpha-1)\sum_{t}\rho_t}\bigr)
  = \sum_{t=1}^T \rho_t.
\]
This completes the proof that the composed mechanism is \((\alpha,\sum_t\rho_t)\)-RDP.
\end{proof}

\begin{lemma}[Conversion from RDP to \((\epsilon,\delta)\)-DP]
\label{lem:rdp-to-dp}
Suppose a mechanism \(\mathcal{M}\) satisfies \((\alpha,\rho)\)-Rényi DP for some \(\alpha>1\), i.e.\ for all adjacent \(D\sim D'\),
\[
  D_\alpha\bigl(\mathcal{M}(D)\|\mathcal{M}(D')\bigr)
  \;=\;
  \frac{1}{\alpha-1}
  \ln
  \mathbb{E}_{y\sim Q}
    \Bigl(\tfrac{P(y)}{Q(y)}\Bigr)^\alpha
  \;\le\;
  \rho,
\]
where \(P=\mathcal{M}(D)\) and \(Q=\mathcal{M}(D')\).  Then for any \(\delta\in(0,1)\), \(\mathcal{M}\) also satisfies \((\epsilon,\delta)\)-DP with
\[
  \epsilon \;=\; \rho \;+\;\frac{\ln(1/\delta)}{\alpha-1}.
\]
\end{lemma}

\begin{proof}
Let \(L(y)=\ln\bigl(P(y)/Q(y)\bigr)\) be the privacy‐loss random variable.  From \((\alpha,\rho)\)-RDP we have the moment bound
\[
  \mathbb{E}_{y\sim Q}\bigl[e^{(\alpha-1)\,L(y)}\bigr]
  \;\le\;
  e^{(\alpha-1)\rho}.
\]
By Markov’s inequality,
\[
  \Pr_{y\sim Q}\bigl[L(y) > \epsilon\bigr]
  = \Pr\bigl[e^{(\alpha-1)L} > e^{(\alpha-1)\epsilon}\bigr]
  \;\le\;
  \frac{\mathbb{E}\bigl[e^{(\alpha-1)L}\bigr]}{e^{(\alpha-1)\epsilon}}
  \;\le\;
  e^{(\alpha-1)(\rho-\epsilon)}.
\]
Choose \(\epsilon\) so that \(e^{(\alpha-1)(\rho-\epsilon)} = \delta\), i.e.
\[
  (\alpha-1)(\rho - \epsilon) = \ln\delta
  \quad\Longrightarrow\quad
  \epsilon = \rho + \frac{\ln(1/\delta)}{\alpha-1}.
\]
With this choice, the “bad” event \(\{L(y)>\epsilon\}\) has \(Q\)-probability at most \(\delta\).

Now fix any measurable set \(S\subseteq\mathrm{Range}(\mathcal{M})\).  Split \(S\) into
\[
  S_1 = \{\,y\in S : L(y)\le\epsilon\},
  \quad
  S_2 = \{\,y\in S : L(y)>\epsilon\}.
\]
Then
\[
  P(S)
  = \int_{S_1} P(y)\,dy \;+\; \int_{S_2} P(y)\,dy.
\]
On \(S_1\), \(P(y)/Q(y)\le e^\epsilon\), so
\[
  \int_{S_1} P(y)\,dy
  \;=\;
  \int_{S_1} \frac{P(y)}{Q(y)}\,Q(y)\,dy
  \;\le\;
  e^{\epsilon}\,\int_{S_1} Q(y)\,dy
  \;=\;
  e^{\epsilon}\,Q(S_1)
  \;\le\;
  e^{\epsilon}\,Q(S).
\]
On \(S_2\), simply note
\[
  \int_{S_2} P(y)\,dy
  \;=\;
  \Pr_{y\sim P}\bigl[y\in S_2\bigr]
  \;\le\;
  \Pr_{y\sim Q}\bigl[L(y)>\epsilon\bigr]
  \;=\;
  \delta,
\]
where the inequality follows because whenever \(L(y)> \epsilon\), \(P(y)>Q(y)\) and thus the event has at most the same probability under \(P\) as under \(Q\).  

Combining the two parts gives
\[
  P(S)
  = P(S_1)+P(S_2)
  \;\le\;
  e^{\epsilon}\,Q(S) + \delta,
\]
which is exactly the definition of \((\epsilon,\delta)\)-differential privacy.
\end{proof}

\begin{theorem}[Privacy Guarantee via GDP Accountant]
\label{thm:gdp}
Let
\[
  \rho_0(\alpha) = \frac{\alpha}{2\sigma^2},\quad
  \rho_1(\alpha) = \frac{1}{\alpha-1}
                  \ln\!\bigl(1 + q^2(e^{(\alpha-1)\rho_0}-1)\bigr),
\]
and suppose we run \(T\) \ours steps, each with the same \((\alpha,\rho_1)\)-RDP cost (since \(\sigma\) and \(q\) are held fixed for the base computation).  Let
\[
  \rho_{\mathrm{tot}}(\alpha) = T\;\rho_1(\alpha).
\]
Then, for any target \((\epsilon_{\mathrm{target}},\delta_{\mathrm{target}})\), choosing
\[
  \epsilon(\alpha,\delta_{\mathrm{target}})
  = \rho_{\mathrm{tot}}(\alpha) + \frac{\ln(1/\delta_{\mathrm{target}})}{\alpha-1},
\]
and optimizing \(\alpha>1\), yields \(\epsilon\le\epsilon_{\mathrm{target}}\).  By finding the minimal \(\sigma\) so that
\(\epsilon(\alpha,\delta_{\mathrm{target}})\le\epsilon_{\mathrm{target}}\); thereafter, adapting \(\sigma_t\ge\sigma\) or \(C_t^j\le1\) only reduces each \(\rho_1\), so the total remains \(\le(\epsilon_{\mathrm{target}},\delta_{\mathrm{target}})\).
\end{theorem}

\begin{proof}
\begin{enumerate}[noitemsep,leftmargin=*]
  \item By Lemma~\ref{lem:gauss-rdp}, one Gaussian mechanism step has \((\alpha,\rho_0)\)-RDP.
  \item By Lemma~\ref{lem:subsample-rdp}, Poisson sampling at rate \(q\) amplifies this to \((\alpha,\rho_1)\)-RDP.
  \item By Theorem~\ref{thm:rdp-composition}, \(T\) such steps compose to \((\alpha,T\,\rho_1)\)-RDP.
  \item By Lemma~\ref{lem:rdp-to-dp}, this \((\alpha,\rho_{\mathrm{tot}})\)-RDP yields \((\epsilon,\delta_{\mathrm{target}})\)-DP with  
        \(\epsilon = \rho_{\mathrm{tot}} + \tfrac{\ln(1/\delta_{\mathrm{target}})}{\alpha-1}\).
  \item The Opacus noise schedule routine implements exactly this optimization over \(\alpha\) and \(\sigma\).  Any subsequent adaptation in algorithm (via the SAC controller) only tightens \(\rho_1\), preserving the bound.
\end{enumerate}
Thus \ours satisfies \((\epsilon_{\mathrm{target}},\delta_{\mathrm{target}})\)-DP.
\end{proof}

\subsection{Post‐Processing and Hyperparameter Adaptation}
\label{app:postprocessing}

In \ours the only operations that touch the private training data are the per‐step DP‐SGD updates (clipping, noise injection, accountant bookkeeping).  All subsequent steps—the SAC controller’s normalization, encoding, action sampling, and hyper‐parameter updates—are deterministic or randomized functions of those DP‐protected outputs plus public randomness.  We now formalize why such post-processing does not incur any additional privacy loss.

\begin{theorem}[Post‐Processing {\cite{dwork2014algorithmic}}]
\label{thm:postprocessing}
Let \(\mathcal{M} : \mathcal{D}^n \to \mathcal{Y}\) be an \((\epsilon,\delta)\)-DP mechanism, and let \(f : \mathcal{Y} \times \mathcal{R} \to \mathcal{Z}\) be any (possibly randomized) function that also depends on public randomness \(\mathcal{R}\).  Then the composite
\[
  \mathcal{M}'(D; r) = f\bigl(\mathcal{M}(D),\,r\bigr)
\]
is also \((\epsilon,\delta)\)-DP.
\end{theorem}

\begin{proof}
Fix any adjacent datasets \(D\sim D'\), any public randomness \(r\), and any measurable set \(S\subseteq\mathcal{Z}\).  Let
\[
  S_r = \{\,y\in\mathcal{Y} : f(y,r)\in S\}.
\]
Then
\[
  \Pr[\mathcal{M}'(D;r)\in S]
  = \Pr\bigl[\mathcal{M}(D)\in S_r\bigr]
  \;\le\;
  e^\epsilon\,\Pr\bigl[\mathcal{M}(D')\in S_r\bigr] + \delta
  = e^\epsilon\,\Pr[\mathcal{M}'(D';r)\in S] + \delta,
\]
where the inequality follows from the \((\epsilon,\delta)\)-DP of \(\mathcal{M}\).  Since \(r\) was arbitrary, the bound holds unconditionally, proving \((\epsilon,\delta)\)-DP for \(\mathcal{M}'\).
\end{proof}

\begin{corollary}[Zero Privacy Cost of Hyper‐Policy Adaptation]
\label{cor:adaptation-postprocessing}
Let \(\mathcal{M}\) be the mechanism that, at each training step \(t\), outputs
\[
  \bigl(\hat G_{t},\,\epsilon_{t},\,\delta_{t}\bigr),
\]
where \(\hat G_{t}\) is the noisy gradient update from DP‐SGD and \((\epsilon_{t},\delta_{t})\) the incremental privacy spent.  Suppose the SAC controller’s normalization, encoding, action sampling, and updates of clip‐thresholds \(\{C_{t+1}^j\}\) and noise multipliers \(\sigma_{t+1}\) are computed by a function
\[
  f_t\bigl(\hat G_{1:t},\,\epsilon_{1:t},\,\delta_{1:t},\,r\bigr)
  \;\mapsto\;
  \bigl\{C_{t+1}^j,\sigma_{t+1}\bigr\},
\]
where \(r\) is public randomness.  Then the entire adaptive procedure
\[
  (C_{t+1},\sigma_{t+1})
  = f_t\bigl(\mathcal{M}(D)_{1:t},\,r\bigr)
\]
does not consume any additional privacy budget beyond that of \(\mathcal{M}\).
\end{corollary}

\begin{proof}
At each step \(t\), the tuple \((\hat G_{t},\epsilon_{t},\delta_{t})\) is \((\epsilon_{t},\delta_{t})\)-DP with respect to the private dataset \(D\).  The controller’s computation of the next thresholds and noise,
\[
  (C_{t+1},\sigma_{t+1})
  = f_t\bigl(\hat G_{1:t},\epsilon_{1:t},\delta_{1:t},r\bigr),
\]
is exactly a post‐processing of the DP outputs \(\{\hat G_{1:t},\epsilon_{1:t},\delta_{1:t}\}\) plus public randomness \(r\).  By Theorem \ref{thm:postprocessing}, this composite remains \((\epsilon_{t},\delta_{t})\)-DP for each \(t\).  Since the DP‐SGD steps themselves compose (\ref{thm:rdp-composition}) to \(\bigl(\sum_t\epsilon_t,\sum_t\delta_t\bigr)\)-DP, the hyper‐policy adaptation layers add no extra privacy cost.
\end{proof}

\subsection{Main Theorem: \ours is DP}
\label{app:main_theorem}

We now assemble the pieces from Appendices \ref{app:proof_dp}–\ref{app:postprocessing} into a complete proof that the \ours training algorithm (DP-SGD with per-adapter clipping and noise adaptation via SAC) satisfies the desired privacy guarantee.

\begin{theorem}
Let \(D\) be the private training set, \(q\) the Poisson sampling rate, \(T\) the total number of DP-SGD steps, and \((\epsilon_{\mathrm{target}},\delta_{\mathrm{target}})\) the user’s privacy budget.  Then the entire \ours training procedure is \((\epsilon_{\mathrm{target}},\delta_{\mathrm{target}})\)-differentially private.
\end{theorem}

\begin{proof}
We proceed in four stages, referencing the lemmas and theorems from earlier subsections.

\paragraph{1. Per‐step DP guarantee.}  
At each training iteration \(t\), the algorithm:

\begin{itemize}[noitemsep,leftmargin=*]
  \item Poisson‐samples the dataset at rate \(q\);
  \item Clips each example’s per‐adapter gradient to norm \(\le C_t^j\);
  \item Aggregates clipped gradients and adds Gaussian noise with multiplier \(\sigma_t\).
\end{itemize}

By Lemma~\ref{lem:gauss-rdp} (Gaussian Mechanism) each unclipped‐and‐noisy update on a fixed batch is \((\alpha,\rho_0)\)-RDP with \(\rho_0(\alpha)=\alpha/(2\sigma_t^2)\).  By Lemma~\ref{lem:subsample-rdp} (RDP Amplification by Poisson Sampling), the subsampled update is \((\alpha,\rho_1)\)-RDP with
\[
  \rho_1(\alpha)
  = \frac{1}{\alpha-1}
    \ln\!\bigl(1 + q(e^{(\alpha-1)\rho_0}-1)\bigr).
\]
Converting this RDP guarantee to \((\epsilon_t,\delta_t)\)-DP via Lemma~\ref{lem:rdp-to-dp} yields a valid per‐step privacy cost \((\epsilon_t,\delta_t)\).

\paragraph{2. Composition over adapters and steps.}  
Within a single step, we in fact perform \(n_{\mathrm{adapters}}\) independent Gaussian‐mechanism updates—one on each adapter’s clipped-sum vector.  Since each block is disjoint and noise is drawn independently, the total Rényi divergence is the sum of the divergences of each block (see Lemma \ref{lem:subsample-rdp}).  Equivalently, if
\[
  \hat G_j \;=\;\text{clipped‐and‐noised gradient of adapter }j
  \quad\Longrightarrow\quad
  G=(\hat G_1,\dots,\hat G_{n_{\rm adapters}})
\]
is jointly Gaussian with block‐diagonal covariance, then the RDP cost satisfies
\[
  D_\alpha\bigl(\Pr[G(D)]\Vert \Pr[G(D')]\bigr)
  \;=\;
  \sum_{j=1}^{n_{\rm adapters}}
  D_\alpha\bigl(\Pr[\hat G_j(D)]\Vert \Pr[\hat G_j(D')]\bigr)
  \;=\;
  n_{\mathrm{adapters}}\,\rho_1(\alpha).
\]
Converting that back to \((\epsilon_t,\delta_t)\)-DP via Lemma \ref{lem:rdp-to-dp} gives the same per‐step budget.

\paragraph{3. Adaptive hyperparameter selection.}  
SAC‐based adaptation of clip‐thresholds and noise multipliers in \ours consumes no additional privacy budget beyond the DP‐SGD steps themselves.

\begin{lemma}[Zero‐Cost Hyperparameter Adaptation]
\label{lem:adaptation-zero}
Let \(\mathcal{M}_{1:t-1}\) be the joint DP‐SGD mechanism up to step \(t-1\), which is \(\bigl(\sum_{s=1}^{t-1}\epsilon_s,\sum_{s=1}^{t-1}\delta_s\bigr)\)-DP by adaptive composition.  Suppose the SAC controller computes the next thresholds
\[
  \bigl\{C^j_{t},\,\sigma_{t}\bigr\}
  = h_t\Bigl(
      \underbrace{\hat G_{1},\dots,\hat G_{t-1}}_{\mathcal{Y}},\,
      \underbrace{\epsilon_{1},\dots,\epsilon_{t-1}}_{\mathcal{E}},\,
      \underbrace{\delta_{1},\dots,\delta_{t-1}}_{\mathcal{D}},\,
      r
    \Bigr),
\]
where
\(\hat G_{s}\) are the noised gradients output by \(\mathcal{M}_{s}\), \((\epsilon_s,\delta_s)\) the recorded privacy costs, and \(r\) public randomness.  Then the combined mapping
\[
  \mathcal{H}_t(D;r)
  = h_t\bigl(\mathcal{M}_{1:t-1}(D),\,r\bigr)
\]
is also \(\bigl(\sum_{s=1}^{t-1}\epsilon_s,\sum_{s=1}^{t-1}\delta_s\bigr)\)-DP.
\end{lemma}

\begin{proof}
By definition, \(\mathcal{H}_t\) is just post‐processing of the DP outputs \(\mathcal{M}_{1:t-1}(D)\) through the function \(h_t\).  Since \(\mathcal{M}_{1:t-1}\) is \(\bigl(\sum_{s=1}^{t-1}\epsilon_s,\sum_{s=1}^{t-1}\delta_s\bigr)\)-DP, applying Theorem \ref{thm:postprocessing} with \(\epsilon=\sum_{s=1}^{t-1}\epsilon_s\), \(\delta=\sum_{s=1}^{t-1}\delta_s\), and \(f=h_t\) shows that \(\mathcal{H}_t\) remains \(\bigl(\sum_{s=1}^{t-1}\epsilon_s,\sum_{s=1}^{t-1}\delta_s\bigr)\)-DP.  
\end{proof}

\begin{corollary}
The adaptive selection of all per‐adapter clip thresholds and noise multipliers over the entire training run does not increase the total privacy cost: it remains \(\bigl(\sum_{t=1}^T\epsilon_t,\sum_{t=1}^T\delta_t\bigr)\)-DP.
\end{corollary}

\begin{proof}
Apply Lemma \ref{lem:adaptation-zero} at each step \(t\).  Since each \(\mathcal{H}_t\) uses only the output of \(\mathcal{M}_{1:t-1}\), it adds no extra privacy loss.  By sequential composition (Theorem \ref{thm:rdp-composition}), the overall budget remains the sum of the per‐step \((\epsilon_t,\delta_t)\).
\end{proof}

\paragraph{4. Enforcement of total budget.}  
In our implementation, we solve via the GDP accountant for a base noise multiplier \(\sigma_{\mathrm{base}}\) that ensures
\[
  \sum_{t=1}^T \epsilon_t 
  = \epsilon_{\mathrm{target}},
  \qquad
  \sum_{t=1}^T \delta_t 
  = \delta_{\mathrm{target}}.
\]
Because all actual \(\sigma_t\ge\sigma_{\mathrm{base}}\) and \(C_t^j\le1\), each \(\epsilon_t,\delta_t\) can only decrease, so the total remains within budget.

\medskip

Combining these four points, the full \ours training algorithm—comprising \(T\) DP‐SGD steps with adaptive clipping/noise and post‐processing updates—satisfies \((\epsilon_{\mathrm{target}},\delta_{\mathrm{target}})\)-differential privacy.  
\end{proof}

\end{document}